\documentclass{article}

\usepackage[preprint]{neurips_2025}

\usepackage[utf8]{inputenc} %
\usepackage[T1]{fontenc}    %
\usepackage{hyperref}       %
\usepackage{url}            %
\usepackage{booktabs}       %
\usepackage{amsfonts}       %
\usepackage{nicefrac}       %
\usepackage{microtype}      %
\usepackage{xcolor}         %
\usepackage[inline]{enumitem}
\usepackage{upgreek}
\usepackage{float}
\usepackage[bf,small]{caption}
\usepackage{subcaption}
\usepackage{wrapfig}
\usepackage{tikz}
\usepackage{pgfplots}
\DeclareUnicodeCharacter{2212}{−}
\usepgfplotslibrary{groupplots,dateplot}
\usetikzlibrary{patterns,shapes.arrows}
\pgfplotsset{compat=newest}

\usepackage[oxfordcolors,fancytheorems]{fancy_theorems}
\input{sam_macros_nofontspec.def}
\newcommand{\mem}{\mathrm{mem}}
\newcommand{\nom}{\mathrm{nom}}
\newcommand{\pmem}{\mathrm{pmem}}

\newcommand{\softmax}{\operatorname{softmax}}
\newcommand{\rset}{\mathbb{R}}
\newcommand{\nset}{\mathbb{N}}
\newcommand{\rmd}{\mathrm{d}}
\newcommand{\vareps}{\varepsilon}

\newcommand{\denoiser}{\bar{\vx}}
\newcommand{\denoisermem}{\denoiser_{\mem}}
\newcommand{\denoisergen}{\denoiser_{\star}}
\newcommand{\denoiserpmem}{\denoiser_{\pmem, M}}
\newcommand{\snr}{\psi}
\newcommand{\snrinv}{\varsigma}

\usepackage{array}   %
\usepackage{autonum}

\newcolumntype{P}[1]{>{\centering\arraybackslash}p{#1}}

\newcommand{\compressstyle}{\displaystyle}

\title{On the Edge of Memorization in Diffusion Models}

\author{%
  Sam Buchanan\thanks{These authors contributed equally to this work.} \\
  TTIC \\
  \And 
  Druv Pai$^{*}$ \\
  UC Berkeley \\
  \And
  Yi Ma \\
  UC Berkeley, HKU \\
  \And 
  Valentin De Bortoli \\
  Google DeepMind
}

\begin{document}

\maketitle

\begin{abstract}
    When do diffusion models reproduce their training data, and when are they able to generate samples beyond it?  A practically relevant theoretical understanding of this interplay between memorization and generalization may significantly impact real-world deployments of diffusion models with respect to issues such as copyright infringement and data privacy. In this work, to disentangle the different factors that influence memorization and generalization in practical diffusion models, we introduce a scientific and mathematical ``laboratory'' for investigating these phenomena in diffusion models trained on fully synthetic or natural image-like structured data. Within this setting, we hypothesize that the memorization or generalization behavior of an underparameterized trained model is determined by the \textit{difference in training loss} between an associated memorizing model and a generalizing model. To probe this hypothesis, we theoretically characterize a \textit{crossover point} wherein the weighted training loss of a fully generalizing model becomes greater than that of an underparameterized memorizing model at a critical value of model
    (under)parameterization. We then demonstrate via carefully-designed experiments %
    that the location of this crossover predicts a phase transition in diffusion models trained via gradient descent, validating our hypothesis. Ultimately, our theory enables us to analytically predict the model size at which memorization becomes predominant. Our work %
    provides an analytically tractable and practically meaningful setting for future theoretical and empirical investigations. Code for our experiments is available at \url{https://github.com/DruvPai/diffusion_mem_gen}.
\end{abstract}

\section{Introduction}

Diffusion models are one of the premier methodologies for deep generative modeling. They exhibit great capabilities across modalities and are state-of-the-art at synthesizing images \citep{saharia2022photorealistic,podell2023sdxl}, videos \citep{ho2022video,blattmann2023align}, or proteins  \citep{watson2023novo}. Despite significant success when used in practice, in the context of large-scale commercial deployment \citep{ramesh2022hierarchical} diffusion models are often plagued with data privacy and copyright infringement issues \citep{ghalebikesabi2023differentially, carlini2023extracting,nasr2023scalable,cui2023diffusionshield,wang2024replication,vyas2023provable,franceschelli2022copyright}, which may have significant long-term consequences. These issues stem from the possibility and (sometimes) propensity for diffusion models to \textit{memorize} their training data. Namely, in some cases, trained diffusion models generate outputs which are verbatim copies of the training samples  \citep{zhang2023emergence,kadkhodaie2023generalization}. Yet, in many other cases, these  models can generate outputs which appear natural but are not present in the training set; this behavior is often described as \emph{creativity} \citep{kamb2024analytic} or \emph{generalization} \citep{zhang2023emergence,kadkhodaie2023generalization,niedoba2024towards}.

In general, memorization and generalization are difficult to disambiguate. There has been significant \textit{empirical} work focusing on building heuristic approaches to detect and study memorization in large-scale diffusion models \citep{zhang2023emergence,gu2023memorization,yoon2023diffusion,carlini2023extracting,somepalli2023diffusion,wen2024detecting,ross2024geometric,wang2024discrepancy,wang2024replication,chen2024towards}. Most such approaches use ad-hoc definitions of memorization which compare the features of the generated samples (w.r.t.~some deep neural network encoder) to features of samples from the training data \citep{pizzi2022self}. This choice effectively reduces the question of \textit{how exactly can we define memorization?} to \textit{how are the features of the chosen encoder related to the input data?} which is also a difficult problem \citep{papyan2020prevalence,yu2023white}. While such methods appear to perform reasonably well in some practical cases, we emphasize that such measures of memorization are ultimately heuristic and there is still scientific disagreement about their efficacy \citep{stein2024exposing}. This is unsuitable for building a scientific understanding of diffusion models that could, among other things, potentially suggest resolutions to the aforementioned legal and societal issues.

Thus, building a theoretical and scientific understanding of memorization in diffusion models is critical. The central problem that such a theory needs to contend with is that the training loss promotes memorization: given a fixed sample set, a sufficiently powerful and perfectly-optimized diffusion model will \textit{always} reproduce the training data exactly, i.e., every single one of its generations will be exactly a training point  \citep{peluchetti2023non,biroli2024dynamical,kamb2024analytic}. Any theory of memorization must therefore explain why well-trained diffusion models do not always memorize. Several previous works have proposed different explanations for this issue in terms of certain aspects of the training procedure, such as the landscape of the stochastic optimization problem which minimizes the training loss \citep{wu2025taking,vastola2025generalization} and the parameterization of the backbone denoiser in the diffusion model \citep{yoon2023diffusion,zhang2023emergence,wang2024diffusion,kamb2024analytic,niedoba2024towards,george2025denoising}. 
However, a precise and predictive theoretical characterization of memorization remains elusive. 

\paragraph{Our contributions.} In this work, we theoretically investigate memorization and generalization in diffusion models. We first introduce a \textit{memorization laboratory}, a natural setting for investigating memorization and generalization in diffusion models trained on synthetic data. We justify this setting by proving that within it, we may distinguish a target distribution from its empirical version for many configurations of problem parameters. To explore the behavior of trained models, we hypothesize that \textit{denoisers trained via gradient descent-like methods memorize or generalize depending on whether the empirical (training) loss is lower for parameter-matched memorizing denoisers or generalizing denoisers}. This hypothesis, to the best of our knowledge, has not been formally explored by prior work, and can only be probed here due to the controlled laboratory setting. To formalize and eventually attempt to falsify this hypothesis, we introduce a \textit{partially memorizing denoiser}, an underparameterized denoiser whose output distribution's samples are always memorized. Within a simple  setting of our laboratory, we characterize the critical level of model (under)parameterization (``crossover point'') at which the training loss of the partially memorizing denoiser first becomes lower than the idealized generalizing denoiser, by deriving and using tight theoretical approximations to these losses. We show that, if our hypothesis is true, then this crossover characterization naturally provides the location of the \textit{phase transition} from generalization to memorization in trained denoisers. Ultimately, we use these theoretically-derived tools to build a \textit{predictive model for the phase transition} from generalization to memorization in terms of a minimal set of problem parameters, and show via experiments that our model achieves extremely low error in practice, validating our hypothesis. We finish our experiments by examining another setting of our laboratory which captures more of the complexities involved in training diffusion models on natural images, showing that despite its additional complexity it is qualitatively very similar to the first case. Beyond the current investigation, the framework we propose provides an analytically tractable yet rich setting for further investigation of memorization and generalization in diffusion models.

All proofs are included in the appendices. Our notation is presented in \Cref{tab:notation_shapes,tab:notation_params}.

\section{A Memorization/Generalization Laboratory}\label{sec:setup}

Diffusion models show a complex tradeoff between
model capacity, training compute, and dataset
size with respect to key behaviors such as memorization
and generalization.
We present a framework (``laboratory'') to disentangle these different factors. 
The class of models we study is sufficiently expressive to admit a rich family of 
behaviors while remaining tractable for theoretical analysis.

\paragraph{Diffusion models.}

Given a target probability distribution $\pi_\star$ and a training dataset $(\vx^i)_{i=1}^N$ from $\pi_\star$, the goal of generative modeling is to define an \emph{output probability distribution} $\hat{\pi}$ which approximates the underlying target $\pi_\star$. Diffusion models \citep{sohl2015deep,song2019generative,ho2020denoising} define such an output distribution $\hat{\pi}$ as the result of a stochastic process.  More precisely, we consider, for $t \in [0, 1]$, a noising process with marginals as follows:
\begin{equation}
    \label{eq:noising_time_t}
    X_t \equid \alpha_t X_0 + \sigma_t Z , \quad Z \sim \mathcal{N}(\Zero,
    \vI) , \quad X_0 \sim \pi_\star ,
\end{equation}
with $\alpha_0 = \sigma_1 = 1$ and $\alpha_1 = \sigma_0 = 0$.
This noising process can be associated with the dynamic
\begin{equation}\label{eq:fwd-sde}
    \mathrm{d} X_t = f_t X_t \mathrm{d} t + g_t \mathrm{d} B_t , \quad
    X_0 \sim \pi_{\star}.
\end{equation}
where $f_t$ and $g_t$ can be obtained in closed form, see
\citet{gao2025diffusionmeetsflow} for instance, and $(B_t)_{t \in
            [0,1]}$ is a $d$-dimensional Brownian motion. We recall that the
backward process associated with \eqref{eq:fwd-sde} is
given by
\begin{equation}\label{eq:bwd-sde}
    \mathrm{d} Y_t = \{-f_{1-t} Y_t + (g_{1-t}^2/2) \nabla \log p_{1-t}(Y_t) \}
    \mathrm{d} t + g_{1-t} \mathrm{d} B_t , \quad Y_0 \sim
    \mathcal{N}(\Zero, \vI),
\end{equation}
where $p_t$ is the density of $X_t$ %
at time $t$. 
Using Tweedie's identity \citep{Robbins1956-fz}, we have 
\begin{equation}\label{eq:tweedie}
    \nabla \log p_t(\vx_t) = (\alpha_t \denoiser(t, \vx_t) - \vx_t) /
    \sigma_t^2 ,
\end{equation}
where $\denoiser(t,\vx_t) = \mathbb{E}[X_0 \ | \ X_t = \vx_t]$.
Thus, using the samples $(\vx^i)_{i=1}^N$, practical diffusion models define a denoiser that
approximates $\denoiser$ by solving the training loss minimization problem
\begin{equation}\label{eq:denoising-objective}
    \compressstyle \min_{\theta} \sL_{N}(\denoiser_{\theta}, \lambda)
\end{equation}
where 
\begin{equation}\label{eq:training_loss_defns}
    \sL_{N}(\denoiser, \lambda) = \Esub*{t}{ \lambda(t)\sL_{N, t}(\denoiser) } , \qquad \sL_{N, t}(\denoiser) = \frac{1}{N} \sum_{i=1}^N \Esub*{X_t^i}{\norm*{\denoiser(t, X_t^i) - \vx^i }^2}
\end{equation}
over a parametric class of denoisers $\denoiser_\theta$, with $\lambda(t)$ a weighting function, $t$ distributed as $\Unif{[0, 1]}$, and  $X_t^i$ distributed according to the noising process \eqref{eq:noising_time_t} with $X_0 = \vx^i$, i.e., $X_{t}^{i} \equid \alpha_{t}\vx^{i} + \sigma_{t}Z$.
We then substitute the learned denoiser $\denoiser_\theta(t, \vx_t)$ into a suitable
discretization of the backward process \eqref{eq:bwd-sde}, via
\eqref{eq:tweedie}, which defines the \emph{output distribution} $\hat{\pi}$ as the
marginal distribution of $Y_1$.

As has been well noted in the literature, the training objective
\eqref{eq:denoising-objective} is at odds with the stated goal of diffusion models.
This is because the non-parametric minimizer of \eqref{eq:denoising-objective} memorizes the training data:
\begin{equation}\label{eq:memorizing-denoiser-empirically-optimal}
    \compressstyle\argmin_{\denoiser}\sL_{N}(\denoiser, \lambda) = \sum_{i=1}^{N} \vx^{i} \softmax(\vw(\spcdot))_i ,
\end{equation}
where $\denoiser(t, \spcdot)$ is square-integrable, for any $\vv \in \bbR^N$ we have $\softmax(\vv)_i = \mathrm{e}^{v_i} / \sum_{j=1}^N \mathrm{e}^{v_j}$, and
\begin{equation}\label{eq:softmax-vector-mem}
    \compressstyle w_i(\vx_t) = -\frac{1}{2\sigma_t^2}\| \alpha_t \vx^{i} - \vx_t \|^2, \quad
    i=1, \dots, N.
\end{equation}
We denote the memorizing denoiser in \eqref{eq:memorizing-denoiser-empirically-optimal} as $\denoisermem(t, \vx_t)$. Towards theoretically studying the memorization/generalization trade-off in diffusion models, we first specify our main assumptions regarding the data and our models which define our memorization/generalization laboratory.

\paragraph{Data and model assumptions.}

We set $\pi_\star$ to be an equally-weighted mixture of
$K$ Gaussians in $\bbR^d$ with means $\vmu_\star^k \in \bbR^d$ and covariances
$\vSigma_\star^k \succeq \Zero$:
\begin{equation}\label{eq:ground-truth-pi-anisotropic}
    \compressstyle \pi_\star = \frac{1}{K} \sum_{k=1}^{K} \sN(\vmu_\star^{k}, \vSigma_{\star}^k \vI).
\end{equation}
Gaussian mixture models are flexible enough to represent a large class of
datasets while being amenable to theoretical investigation, see \citep{wang2024unreasonable,Shah2023-qw,wang2024diffusion} for instance, \Cref{sub:experiments_image} for an example, and  \Cref{sec:gmm-denoiser-calcs} for a further discussion. 

Let $M \in \bbN$ denote the number of equally-weighted
mixture components in a generic Gaussian mixture $\pi_\theta$, with $\theta = (\vmu^1,
\vSigma^1, \dots, \vmu^M, \vSigma^M)$. %
Tweedie's identity \eqref{eq:tweedie} links the 
statistical model $\pi_\theta$ and its associated denoiser $\denoiser_\theta$,
which we recall below in \Cref{lemma:denoiser_mog} and prove in
\Cref{sec:gmm-denoiser-calcs}.
\begin{lemma}{Gaussian mixture model denoiser}{denoiser_mog}
    Assume that $\pi_\theta = (1/M) \sum_{i=1}^{M} \sN(\vmu^{i},
        \vSigma^i)$. Then, we have that
    \begin{equation}\label{eq:model-class-general}
    \compressstyle
        \denoiser_\theta(t,\vx_t) = 
        \frac{1}{\alpha_t}\left(
        \vx_t -
        \sigma^{2}_t\sum_{i =
        1}^{M}(\alpha_t^{2}\bm{\Sigma}^{i}
        + \sigma_t^{2}\bm{I})^{-1}(\vx_t - \alpha_t \bm{\mu}^{i})
        \operatorname{softmax}(\bm{w}(t, \vx_t))_{i}
        \right),
    \end{equation}
    where for all $i \in [M]$
    \begin{equation}
        w_i(t, \vx_t) = 
        -\frac{1}{2}\log\det(\alpha_t^{2}\bm{\Sigma}^{i} + \sigma_t^{2}\bm{I}) -
        \frac{1}{2}(\vx_t-
        \alpha_t\bm{\mu}^{i})^{\top}(\alpha_t^{2}\bm{\Sigma}^{i} +
        \sigma_t^{2}\bm{I})^{-1}(\vx_t - \alpha_t \bm{\mu}^{i}) 
        .
    \end{equation}
\end{lemma}
The class of $M$-parameter Gaussian mixture models constitutes a rich framework
for investigating memorization and generalization. For at one extreme, where
$M=K$, $\vSigma^k = \vSigma_\star^k$ and $\vmu^k = \vmu_\star^k$ for $k \in [K]$
we recover the \emph{generalizing denoiser} associated to the target
distribution \eqref{eq:ground-truth-pi-anisotropic}, which we will denote as
$\denoisergen$; and at the other extreme, with $M=N$, $\vSigma^i = \Zero$, and
$\vmu^i = \vx^i$ for $i \in [N]$ we recover the \emph{memorizing denoiser}
$\denoisermem$ as in \eqref{eq:memorizing-denoiser-empirically-optimal}.
We will therefore consider the training loss \eqref{eq:denoising-objective} with
the model class specified by \Cref{lemma:denoiser_mog}, and study the role
played by the model capacity $M$ as a function of the data complexity $K$, the
dimension $d$, and the number of training samples $N$. %

\paragraph{Memorization and generalization.} We now aim to quantify our notions of memorization and generalization.
Quantifying memorization in generative
models is a complicated issue and many metrics have been proposed to evaluate
it in practice.
We adopt a popular, relatively strict metric for memorization, proposed by
\citet{yoon2023diffusion}. 
\begin{definition}{Memorization}{mem}
    Given a dataset $(\vx^i)_{i=1}^N$, a small absolute constant $c \in (0, 1)$, and the output distribution $\hat{\pi}$ of
    a diffusion model, we say that a sample $\hat{\vx}\sim \hat{\pi}$ is
    \emph{memorized} if $\| \hat{\vx} - \vx^{(1)} \|^2 \leq c \| \hat{\vx} - \vx^{(2)} \|^2$, 
    where $\vx^{(k)}$ is the $k$-th nearest neighbor in $\ell_2$ norm to
    $\hat{\vx}$ in  $( \vx^i )_{i=1}^N$.
\end{definition}

On the other hand, generalization is easily formalized in statistical learning
terms (in contrast to, e.g., assessing creativity of practical image diffusion
models), see \Cref{sec:gmm-denoiser-calcs}. 
In experiments, we estimate the generalization error on a held-out set of
samples from $\pi_\star$, which we can freely generate.

Given these quantitative definitions, the key question that we
investigate in our laboratory is:
\begin{quote}
    \centering
    \textit{%
        For a trained {denoiser $\denoiser_\theta$}, can we quantitatively predict, based on problem parameters \(M\), \(N\), \(d\), and \(K\), whether it memorizes or generalizes?}
\end{quote}
By the phrases ``the denoiser memorizes/generalizes'' we mean that its associated output distribution \(\hat{\pi}\) produces memorized samples or samples (approximately) from the true distribution, as defined above. As we will immediately see, it is also fruitful to understand these behaviors as \(\hat{\pi}\) being ``close'' to the empirical distribution \(\pi_{\star}^{N}\) of the training set \((\vx^{i})_{i = 1}^{N}\) or the target distribution \(\pi_{\star}\) respectively.

\paragraph{Scaling the number of samples.} 
Under the Gaussian mixture model $\pi_\star$ that generates the training data
$(\vx^i)_{i=1}^N$, the parameters $K$ and $d$ (together with the means
$\vmu^k_\star$ and covariances $\vSigma_\star^k$) control the geometric
complexity of the data. Relative to these measures of complexity,
the scaling of the number of samples $N$ plays a fundamental role in the
dichotomy between memorization and generalization behavior we seek to establish.
We will focus on the case where $N = \poly(d)$, which we argue below is a correct scaling in
which to study memorization and generalization.

For assessing the similarity of probability distributions, 
it is standard to use the $2$-Wasserstein distance $\mathrm{W}_2$
\citep{Blau2017-de}. 
We will argue that when $N = \exp[d \log d]$, there is no meaningful
distinction between memorization and generalization in sufficiently high
dimensions. %
For simplicity, assume all covariance matrices $\vSigma_\star^k$ in the definition of $\pi_{\star}$ in \eqref{eq:ground-truth-pi-anisotropic} are full rank.
Then by \citep[Theorem 1, Proposition 7]{Weed2019-lo}, 
we have for all $d \geq 4$ that
$\E{\mathrm{W}_2(\pi_\star, \pi^N_\star)}
\leq C_0 N^{-1/2d} \leq C_0/\sqrt{d}$, for a constant $C_0 \geq 0$. %
Therefore, $\mathrm{W}_2(\pi_\star, \pi^N_\star) \to 0$ and we cannot
distinguish the true distribution $\pi_\star$
from its samples $\pi^N_\star$: memorization and generalization are equivalent.
On the other hand, using \citep[Theorem 1, Proposition 2]{Weed2019-lo}, 
one has for \textit{any} draw of the empirical measure that
$\mathrm{W}_2(\pi_\star, \pi_\star^N) \geq C_1 N^{-2 / d}$
for a constant $C_1 \geq 0$.
In particular, if $N = \poly(d)$ then
the $2$-Wasserstein distance is lower bounded by a constant for any
dimension, implying a meaningful distinction between memorization and
generalization.

\begin{minipage}[t]{0.35\textwidth}
    \centering
    \begin{tabular}{|c|c|}
        \hline
        Notation & Interpretation \\
        \hline
        $N$      & Number of samples \\
        $M$      & Capacity of the model \\
        $K$      & Number of modes \\
        $d$      & Problem dimension \\
        \hline
    \end{tabular}
    \captionof{table}{Data and model scalings.}
    \label{tab:notation_shapes}
\end{minipage}
\hspace{1em}
\begin{minipage}[t]{0.59\textwidth}
    \centering
    \begin{tabular}{|c|P{0.7\linewidth}|}
        \hline
        Notation & Interpretation \\
        \hline
        $(\vmu^i_\star)_{i=1}^K, \sigma_\star^2$ & Data generating process parameters \\
        $(\vmu^i)_{i=1}^M, \sigma^2$ & Learned model parameters \\
        $\denoisergen(t, \vx_t)$ & MMSE denoiser, data generating process \\
        $\denoiser_\theta(t, \vx_t)$ & MMSE denoiser, learned model \\
        $\denoisermem(t, \vx_t)$ & MMSE denoiser, empirical distribution \\
        $\denoiserpmem(t, \vx_t)$ & Partially memorizing denoiser \\
        $\vw$ & Softmax weight vector \\ \hline
    \end{tabular}
    \captionof{table}{Data and model parameters.}
    \label{tab:notation_params}
\end{minipage}

\section{Training Losses and Memorization}
\label{sec:training_losses}

Now that we have set up the framework, we turn to answering the following fundamental question: \textit{when do trained denoisers (i.e., approximate optimizers of \eqref{eq:denoising-objective}) memorize, and when do they generalize}? One first approach to answering this question would be to directly examine the critical points of the diffusion model training optimization problem \eqref{eq:denoising-objective} w.r.t.~the denoiser parameterization \eqref{eq:model-class-general}. However, this is not straightforward; the problem is non-convex and there are many spurious critical points. Therefore, we posit a hypothesis which, roughly speaking, says that \textit{the training losses of surrogate memorizing denoisers with \(M\) parameters and generalizing denoisers is all that matters} for predicting the behavior of trained models with \(M\) parameters. 

To formalize this hypothesis, we will define two {simple surrogate} denoisers whose training losses will help us identify transitions between memorization and generalization as a function of the models'  underparameterization.  
The first of these two summary models is the generalizing denoiser $\denoisergen$. The second is a \textit{partially memorizing} denoiser \(\denoiserpmem\), which memorizes the first \(M\) training samples:\footnote{For our theory, the choice of the subset of samples we use in the partial memorizing denoiser is irrelevant, as long as its size is \(M\) and it is chosen independently of the samples.}
\begin{equation}\label{eq:pmem-denoiser}
    \compressstyle \denoiserpmem(t,\vx_t) =
    \sum_{i=1}^{M} \vx^i \softmax(\vw(t, \vx_t))_i , \quad \text{where} \quad
    \compressstyle w_i(t, \vx_t) = -\frac{1}{2}\| \alpha_t \vx^i - \vx_t \|^2/ \sigma_t^2 .
\end{equation}
Note that here \(\vw \colon \bbR \times \bbR^{d} \to \bbR^{M}\), and in the case where $M=N$, it holds \(\denoiserpmem = \denoisermem\). %
Therefore, we can formulate our hypothesis as follows:
\begin{quote}
    \centering
    \textit{There exists a loss weighting \(\lambda\) such that a trained denoiser \(\denoiser_{\theta}\) with \(M\) parameters memorizes if and only if \(\sL_{N}(\denoiserpmem, \lambda) \leq \sL_{N}(\denoisergen, \lambda)\).}
\end{quote}

To attempt to verify this hypothesis, we will estimate the (excess) training losses of the two denoisers \(\denoiserpmem\) and \(\denoisergen\) and compare them to develop a criterion for when a trained denoiser memorizes. For tractability, we focus on a special case of the general mixture of Gaussians model \eqref{eq:ground-truth-pi-anisotropic} where each mixture component has identical covariance, equal to a constant multiple of the identity matrix (i.e., $\vSigma_\star^k = \sigma_\star^2 \vI$ for each $k \in [K]$ for the ground truth model \eqref{eq:ground-truth-pi-anisotropic}, and $\vSigma^i = \sigma^2 \vI$ for each $i \in [M]$ for general  denoisers $\denoiser_\theta$), as reflected in \Cref{tab:notation_params}. This is a common assumption in the theoretical literature on diffusion models \citep{Shah2023-qw,Gatmiry2024-yg}. The structure of the learned denoiser $\denoiser_\theta$ implied by \Cref{lemma:denoiser_mog} simplifies to:
\begin{equation}\label{eq:model-class}
    \compressstyle \denoiser_\theta(t,\vx_t) = \frac{\alpha_t \sigma^2}{\alpha_t^2
    \sigma^2 + \sigma_t^2} \vx_t + \frac{\sigma_t^2}{\alpha_t^2
    \sigma^2 + \sigma_t^2} \sum_{i=1}^{M} \vmu^{i}
    \softmax(\vw(\vx_t))_i ,
\end{equation}
where
\begin{equation}\label{eq:softmax-vector-model}
    \compressstyle w_i(\vx_t) = -\frac{1}{2}\| \alpha_t \vmu^{i} - \vx_t \|^2/
    (\alpha_t^2 \sigma^2 + \sigma_t^2), \quad i = 1, \dots, M,
\end{equation}
with the analogous structure for $\denoisergen = \denoiser_{(\vmu^1_\star, \dots, \vmu^K_\star, \sigma_\star^2)}$. In contrast to \eqref{eq:model-class}, the form of the memorizing and partially memorizing denoisers \eqref{eq:memorizing-denoiser-empirically-optimal} and \eqref{eq:pmem-denoiser} does not change. Note that when \(M \geq K\), the denoiser \(\denoiser_{\theta}\) with \(M\) parameters can generalize, and when \(M \geq N\), it can memorize.\footnote{As we show rigorously in \Cref{app:softmax-infinity} using properties of the softmax function, the loss \eqref{eq:denoising-objective} of any $M$-parameter denoiser $\denoiser_{(\vmu^1, \dots, \vmu^M, \sigma^2)}$ can be achieved by a sequence of $M+1$ parameter denoisers. Therefore for a fixed value of $M \geq K$, a denoiser with $M$ parameters can be (arbitrarily close to) the generalizing denoiser, i.e., denoisers with arbitrarily many parameters can generalize, and similarly for memorization with \(M \geq N\).}

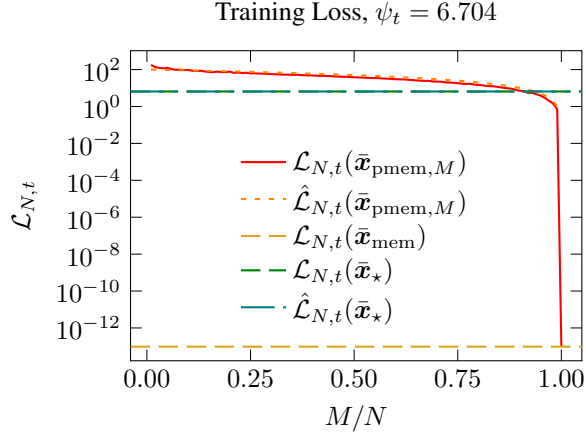
\begin{wrapfigure}{r}{3.5in}
    \vspace{-1.25em}
    \centering
    \begin{tikzpicture}

    \definecolor{darkgray176}{RGB}{176,176,176}
    \definecolor{darkorange}{RGB}{255,140,0}
    \definecolor{goldenrod}{RGB}{218,165,32}
    \definecolor{green}{RGB}{0,128,0}
    \definecolor{teal}{RGB}{0,128,128}
    
    \begin{axis}[
    width=6cm,
    height=4.125cm,
    scale only axis,
    legend cell align={left},
    legend style={
      fill opacity=0.8,
      draw opacity=1,
      text opacity=1,
      at={(0.5,0.09)},
      anchor=south,
      draw=none
    },
    tick pos=both,
    title={Training Loss, \(\displaystyle \psi_t=6.704\)},
    x grid style={darkgray176},
    xlabel={\(\displaystyle M/N\)},
    xmin=-0.0395, xmax=1.0495,
    xtick style={color=black},
    xtick={-0.25,0,0.25,0.5,0.75,1,1.25},
    xticklabels={
      \(\displaystyle {\ensuremath{-}0.25}\),
      \(\displaystyle {0.00}\),
      \(\displaystyle {0.25}\),
      \(\displaystyle {0.50}\),
      \(\displaystyle {0.75}\),
      \(\displaystyle {1.00}\),
      \(\displaystyle {1.25}\)
    },
    y grid style={darkgray176},
    ylabel={\(\displaystyle \sL_{N, t}\)},
    ymin=1.66751366424431e-14, ymax=1051.25447053777,
    ymode=log,
    ytick style={color=black},
    ytick={1e-16,1e-14,1e-12,1e-10,1e-08,1e-06,0.0001,0.01,1,100,10000,1000000},
    yticklabels={
      \(\displaystyle {10^{-16}}\),
      \(\displaystyle {10^{-14}}\),
      \(\displaystyle {10^{-12}}\),
      \(\displaystyle {10^{-10}}\),
      \(\displaystyle {10^{-8}}\),
      \(\displaystyle {10^{-6}}\),
      \(\displaystyle {10^{-4}}\),
      \(\displaystyle {10^{-2}}\),
      \(\displaystyle {10^{0}}\),
      \(\displaystyle {10^{2}}\),
      \(\displaystyle {10^{4}}\),
      \(\displaystyle {10^{6}}\)
    }
    ]
    \addplot [thick, red]
    table {%
    0.01 181.170684814453
    0.02 133.855880737305
    0.03 123.162231445312
    0.04 116.370750427246
    0.05 123.61043548584
    0.06 104.535087585449
    0.07 98.3118743896484
    0.08 98.60205078125
    0.09 93.9722900390625
    0.1 95.2687225341797
    0.11 85.3343124389648
    0.12 83.6117477416992
    0.13 85.9586563110352
    0.14 78.6683044433594
    0.15 74.3079605102539
    0.16 72.0888290405273
    0.17 77.8091659545898
    0.18 73.4287033081055
    0.19 72.0889511108398
    0.2 77.3866577148438
    0.21 66.0567169189453
    0.22 68.5534439086914
    0.23 65.0179672241211
    0.24 66.7331314086914
    0.25 62.5315055847168
    0.26 62.551399230957
    0.27 60.9723968505859
    0.28 59.436595916748
    0.29 58.6361389160156
    0.3 56.4673614501953
    0.31 55.6729431152344
    0.32 54.2639579772949
    0.33 52.8108215332031
    0.34 53.0099906921387
    0.35 51.9190979003906
    0.36 51.7803840637207
    0.37 50.8909034729004
    0.38 49.1069107055664
    0.39 48.7385520935059
    0.4 47.5291976928711
    0.41 46.6154174804688
    0.42 45.3658599853516
    0.43 45.2809143066406
    0.44 44.1082420349121
    0.45 43.0011367797852
    0.46 41.9924278259277
    0.47 41.0245018005371
    0.48 40.7452278137207
    0.49 39.1224517822266
    0.5 38.3772315979004
    0.51 37.641674041748
    0.52 36.9856147766113
    0.53 35.4440650939941
    0.54 35.1983261108398
    0.55 34.8605690002441
    0.56 34.1064758300781
    0.57 32.9464874267578
    0.58 32.2447776794434
    0.59 30.739200592041
    0.6 30.1946067810059
    0.61 29.5526313781738
    0.62 29.4174747467041
    0.63 28.1524066925049
    0.64 26.9600772857666
    0.65 26.0730400085449
    0.66 24.9251556396484
    0.67 24.3658218383789
    0.68 24.357551574707
    0.69 22.5032386779785
    0.7 22.3039016723633
    0.71 22.0489330291748
    0.72 20.2791290283203
    0.73 19.7330188751221
    0.74 19.1006164550781
    0.75 18.4120788574219
    0.76 17.2008361816406
    0.77 16.6683349609375
    0.78 17.1075801849365
    0.79 15.435453414917
    0.8 14.7518835067749
    0.81 14.1077671051025
    0.82 12.8226547241211
    0.83 12.6238822937012
    0.84 12.1201686859131
    0.85 11.183557510376
    0.86 10.3495101928711
    0.87 9.71383857727051
    0.88 9.09174346923828
    0.89 7.93989086151123
    0.9 7.10683584213257
    0.91 6.54107666015625
    0.92 5.47833681106567
    0.93 5.42826700210571
    0.94 4.56988859176636
    0.95 3.86432957649231
    0.96 2.98018598556519
    0.97 2.03940987586975
    0.98 1.45248878002167
    0.99 0.714033365249634
    1 9.67585454575599e-14
    };
    \addlegendentry{$\sL_{N, t}(\denoiserpmem)$}
    \addplot [thick, darkorange, dash pattern=on 2pt off 3.3pt]
    table {%
    0.01 99
    0.02 98
    0.03 97
    0.04 96
    0.05 95
    0.06 94
    0.07 93
    0.08 92
    0.09 91
    0.1 90
    0.11 89
    0.12 88
    0.13 87
    0.14 86
    0.15 85
    0.16 84
    0.17 83
    0.18 82
    0.19 81
    0.2 80
    0.21 79
    0.22 78
    0.23 77
    0.24 76
    0.25 75
    0.26 74
    0.27 73
    0.28 72
    0.29 71
    0.3 70
    0.31 69
    0.32 68
    0.33 67
    0.34 66
    0.35 65
    0.36 64
    0.37 63
    0.38 62
    0.39 61
    0.4 60
    0.41 59
    0.42 58
    0.43 57
    0.44 56
    0.45 55
    0.46 54
    0.47 53
    0.48 52
    0.49 51
    0.5 50
    0.51 49
    0.52 48
    0.53 47
    0.54 46
    0.55 45
    0.56 44
    0.57 43
    0.58 42
    0.59 41
    0.6 40
    0.61 39
    0.62 38
    0.63 37
    0.64 36
    0.65 35
    0.66 34
    0.67 33
    0.68 32
    0.69 31
    0.7 30
    0.71 29
    0.72 28
    0.73 27
    0.74 26
    0.75 25
    0.76 24
    0.77 23
    0.78 22
    0.79 21
    0.8 20
    0.81 19
    0.82 18
    0.83 17
    0.84 16
    0.85 15
    0.86 14
    0.87 13
    0.88 12
    0.89 11
    0.9 10
    0.91 9
    0.92 8
    0.93 6.99999999999999
    0.94 6.00000000000001
    0.95 5
    0.96 4
    0.97 3
    0.98 2
    0.99 1
    1 0
    };
    \addlegendentry{$\hat{\sL}_{N, t}(\denoiserpmem)$}
    \addplot [thick, goldenrod, dash pattern=on 7.4pt off 3.2pt]
    table {%
    -0.0394999999999999 9.68486155530393e-14
    1.0495 9.68486155530393e-14
    };
    \addlegendentry{$\sL_{N, t}(\denoisermem)$}
    \addplot [thick, green, dash pattern=on 7.4pt off 3.2pt]
    table {%
    -0.0394999999999999 6.47629642486574
    1.0495 6.47629642486574
    };
    \addlegendentry{$\sL_{N, t}(\denoisergen)$}
    \addplot [thick, teal, dash pattern=on 12.8pt off 3.2pt on 2pt off 3.2pt]
    table {%
    -0.0394999999999999 6.4900841712952
    1.0495 6.4900841712952
    };
    \addlegendentry{$\hat{\sL}_{N, t}(\denoisergen)$}
    \end{axis}
    
    \end{tikzpicture}
    \vspace{-0.5em}
    \caption{\textbf{We observe a remarkable
    degree of agreement between our loss approximations and the empirical
    losses.} A simulation of the loss of the partially memorizing and
    generalizing denoisers $\denoisergen$ and $\denoiserpmem$ at
    a high-SNR value of $t$, with $d=50$, $K=12$, and $N=200$, and compare them
    to the approximations introduced in 
    \Cref{thm:asymptotics-generalizing,thm:simplification_p_mem_scaling}.
    We use a constant $2$ for the $\Theta(\spcdot)$ expression appearing
    in \Cref{thm:simplification_p_mem_scaling}, which we can prove is also an upper-bound. }
    \vspace{-2em}
    \label{fig:train-loss-approx-verify}
\end{wrapfigure}
\paragraph{Training losses of memorizing and generalizing denoisers.} In what
follows, we are going to state our main theoretical results which characterize
the behavior of two denoisers of interest. More precisely, we compare the
training loss \eqref{eq:denoising-objective} computed for the $K$-parameter
generalizing denoiser $\denoisergen$ with the training loss obtained by an
$M$-parameter partially memorizing denoiser \eqref{eq:pmem-denoiser}, with $M
\geq K$. To do this we compute the high dimensional asymptotics of these
denoisers' excess training losses. Note that these asymptotics provide
approximations (annotated with hats, i.e., \(\hat{\sL}_{N, t}\)) which agree
well with experiments at even moderate dimensions (see
\Cref{fig:train-loss-approx-verify} and \Cref{sec:experiments}).  

We recall the standard re-expression of the objective in
\eqref{eq:denoising-objective} via the orthogonality principle: for any $t$, 
\begin{align}\label{eq:training-loss-ortho}
    \sL_{N, t}(\denoiser)
    &=
    \frac{1}{N} \sum_{i=1}^N \Esub*{X_t^i}{\norm*{
    \denoiser(t, X_t^i) - \denoisermem(t, X_t^i)
    }^2}
    +
    \frac{1}{N} \sum_{i=1}^N \Esub*{X_t^i}{\norm*{
    \denoisermem(t, X_t^i) - \vx^i
    }^2}, \\
    &= \frac{1}{N} \sum_{i=1}^N \Esub*{X_t^i}{\norm*{
    \denoiser(t, X_t^i) - \denoisermem(t, X_t^i)
    }^2} + \sL_{N, t}(\denoisermem),
\end{align}
where $\denoisermem$ is the memorizing denoiser
\eqref{eq:memorizing-denoiser-empirically-optimal} and $X_t^i$ is distributed as
\eqref{eq:noising_time_t} initialized with $\vx^i$, i.e., $X_t^i \equid \alpha_t \vx^i + \sigma_t Z$ where $Z \sim \mathcal{N}(\Zero, \vI)$.

Our first result fully characterizes the expected excess training loss of the
generalizing denoiser $\denoisergen$, given by \eqref{eq:model-class} with
$\theta = (\vmu_\star^1, \dots, \vmu_\star^K,
\sigma_\star^2)$, over a draw of the random i.i.d.\ sample
$(\vx^1, \dots, \vx^N)$ from $\pi_\star$, under an assumption that the cluster
centers $\vmu_\star^k$ are well-separated in $\ell^2$ distance. 
For simplicity, we state our results in the regime where $\max_k
\norm{\vmu_\star^k}^2 = O(d)$ and $\sigma_\star^2 = \Theta(1)$.\footnote{To motivate this regime, note that in intuitive
terms, images sampled from such a $\pi_\star$ are centered at a nominal image
whose pixel intensities are normalized to $[0, 1]$, and the variability due to
the Gaussian noise can change each pixel by a constant amount.}
Denote the signal-to-noise ratio (SNR) of the noising process
\eqref{eq:noising_time_t} by $\psi : (0, 1) \to (0, +\infty)$, where $\snr_t
= \alpha_t^2/\sigma_t^2$, which is assumed to be decreasing, and its inverse by $\snrinv \colon (0, +\infty) \to (0, 1)$.

\begin{theorem}{}{asymptotics-generalizing}
    Assume that $N = \mathrm{poly}(d)$,
    $\min_{k \neq k'}\, \norm{\vmu_{\star}^k - \vmu_{\star}^{k'}}^2
    = \Theta(d)$,
    $\max_{k}\, \norm{\vmu_\star^k}^2 = \Theta(d)$ 
    and $\sigma_\star^2 = \Theta(1)$. 
    Let $\kappa(d) = \snrinv(\Theta(\log(d)^2/d))$. We have that uniformly on $t \in [0, \kappa(d)]$
    \begin{equation}
        \Esub*{(\vx^i)_{i=1}^{N}}{\sL_{N, t}(\denoisergen) - \sL_{N, t}(\denoisermem)}
        = \Theta\left(\frac{d\sigma_\star^2 }{\snr_t \sigma_\star^2 + 1}\right).
    \end{equation}
    In particular, the leading-order coefficient for the right-hand side is \(1\).
\end{theorem}

Although \Cref{thm:asymptotics-generalizing} is stated for the excess
training loss, our proofs further establish high-probability bounds on the behavior of
the excess loss of the same order of magnitude. Moreover, in the scaling regime for
$\sigma_\star^2$ and $(\vmu_\star^k)_{k=1}^K$ treated in
\Cref{thm:asymptotics-generalizing}, it is easily shown that for
practically-relevant choices of $\alpha_t$ and $\sigma_t$, for example the
scheme $\alpha_t = \sqrt{1-t^2}$, $\sigma_t = t$ that we use in our experiments
in \Cref{sec:experiments}, $\kappa(d) \to 1$ as $d\to \infty$.
As a consequence, for sufficiently large $d$, the uniform control of the risk
for $t \in [0, \kappa(d)]$ that is established in
\Cref{thm:asymptotics-generalizing} implies uniform control of the weighted integrated
loss \eqref{eq:denoising-objective}, via \eqref{eq:training-loss-ortho}.

Our second result estimates the excess training loss of the partially
memorizing denoiser $\denoiserpmem$.

\begin{theorem}{}{simplification_p_mem_scaling}
    Assume that $N = \mathrm{poly}(d)$,
    $\min_{k \neq k'}\, \norm{\vmu_{\star}^k - \vmu_{\star}^{k'}}^2
    = \Theta(d)$,
    $\max_{k}\, \norm{\vmu_\star^k}^2 = \Theta(d)$ 
    and $\sigma_\star^2 = \Theta(1)$. 
    Let $\kappa(d) = \snrinv(\Theta(\log(d)^2/d))$. We have that uniformly on $t
    \in [0, \kappa(d)]$
    \begin{equation}
        \Esub*{(\vx^i)_{i=1}^{N}}{\sL_{N, t}(\denoiserpmem) - \sL_{N, t}(\denoisermem)}
        = \Theta\left(\left(1 - \frac{M}{N}\right)
        d \sigma_\star^2\right).
    \end{equation}
    In particular, the leading-order coefficient for the right-hand side is between \(1\) and \(2\).
\end{theorem}
 Note that by following the proof, we can also derive a corresponding high-probability bound.

\paragraph{When does each denoiser have lower training loss?} Comparing \Cref{thm:asymptotics-generalizing} and \Cref{thm:simplification_p_mem_scaling}, in high dimensions (\(d \to \infty\)) we estimate the training loss difference as
\begin{equation}
    \scalebox{0.975}{\(\compressstyle\Esub*{(\vx^i)_{i=1}^{N}}{\sL_{N}(\denoiserpmem, \lambda) - \sL_{N}(\denoisergen, \lambda)} \approx \Esub*{t}{\lambda(t)\left\{C\left(1 - \frac{M}{N}\right) d \sigma_\star^2 - \frac{d\sigma_\star^2}{\snr_t \sigma_\star^2 + 1}\right\}},\)}
\end{equation}
where \(C \in [1, 2]\) is a constant. Notice that this expression is monotonically decreasing as \(M \to N\), so there exists a ``crossover point'' \(M_{\star}\) such that the (approximate) training loss of the partially memorizing denoiser is higher than that of the generalizing denoiser for \(M \leq M_{\star}\) and lower for \(M \geq M_{\star}\). Solving for \(M_{\star}\), we obtain 
\begin{equation}\label{eq:crossover-point}
    \scalebox{0.975}{\(\compressstyle\Esub*{(\vx^i)_{i=1}^{N}}{\sL_{N}(\denoiser_{\pmem, M_{\star}}, \lambda) - \sL_{N}(\denoisergen, \lambda)} \approx 0 \implies M_{\star} \approx N\left\{1 - \frac{\Esub*{t}{\lambda(t)/(\snr_{t}\sigma_{\star}^{2} + 1)}}{C\Esub*{t}{\lambda(t)}}\right\},\)}
\end{equation}
which is a \textit{linear} function of \(N\). This criterion provides a \textit{test} for our hypothesis: if we believe that there exists a loss weighting \(\lambda\) such that the memorization and generalization properties of trained denoisers are controlled by the location of the crossover point w.r.t.~\(\lambda\), then we should see an approximately linear relationship between the location of memorization and the number of samples. We conduct this test in \Cref{sec:experiments}, in the immediate sequel.

\section{Experiments}
\label{sec:experiments}

In this section, we illustrate the versatility and efficacy of our laboratory through experimental analysis. First, we show that inside the setting of our laboratory, \textit{trained} models exhibit a \textit{phase transition from generalization to memorization}. Specifically in the isotropic Gaussian mixture setting, our main result builds a \textit{predictive model} for the onset of memorization, and shows that it resolves to a simple linear fit as in \eqref{eq:crossover-point}, which \textit{validates our hypothesis} from \Cref{sec:training_losses}. Finally, we showcase the flexibility of our laboratory by studying a low-rank Gaussian mixture setting which is constructed to imitate the structure of training a denoiser on natural images, yet still belongs to the setting of our theoretical laboratory and exhibits similar phase transition behavior as in the simple isotropic case.  
We provide further details and more results, including mechanistic analyses, in \Cref{app:experiments}.

\paragraph{High-level experiment details, and memorization criterion.} For training, we use the training loss \eqref{eq:denoising-objective} with the loss weighting \(\lambda(t) = (\alpha_{t}/\sigma_{t})^{2}\) (i.e., ``noise prediction''). For sampling, we use the DDIM sampler \citep{song2020denoising}, more precisely the implementation suggested by \citep{De-Bortoli2025-ng}, with the ``variance preserving'' coefficients \(\alpha_{t} = \sqrt{1 - t^{2}}\) and \(\sigma_{t} = t\). We discretize the time interval \([\epsilon, 1-\epsilon]\) uniformly into \(L+1\) timesteps \((t_{\ell})_{\ell = 0}^{L}\), where \(\epsilon = 10^{-3}\), and use these timesteps for both training and sampling. For completeness we formally describe the end-to-end procedure in \Cref{app:sub_sampling}.  
We consider the memorization criterion introduced in \Cref{def:mem} with constant \(c = 1/9\), i.e., \(\hat{\vx}\) is memorized if \(\norm{\hat{\vx} - \vx^{(1)}}^{2} \leq (1/9)\norm{\hat{\vx} - \vx^{(2)}}^{2}\) where \(\hat{\vx}\) is the generated sample, and \(\vx^{(k)}\) is the \(k^{\mathrm{th}}\) closest point to \(\hat{\vx}\) in the training dataset. Then, we say that a denoiser \(\bar{\vx}\) is \textit{memorizing on average} if its \textit{memorization ratio} is \(\geq 1/2\), i.e., at least 50\% of the samples \(\hat{\vx}\) drawn from the associated output measure \(\hat{\pi}\) are memorized; we say it has \textit{started the phase transition} if its memorization ratio is \(\geq 1/10\) and \textit{ended the phase transition} if its memorization ratio is \(\geq 9/10\).

\subsection{Experiments for an Isotropic Gaussian Mixture Model} 
\label{sub:experiments_gmm}

First, as in \Cref{sec:training_losses}, we consider the target measure \(\pi_{\star}\) to be an isotropic Gaussian mixture model, namely, \(\pi_{\star} = (1/K)\sum_{i = 1}^{K}\mathcal{N}(\vmu_{\star}^{i}, \sigma_{\star}^{2}\vI)\). As emphasized in the rest of the work, we will study what happens when we use a denoiser \(\bar{\vx}_{\theta}\) corresponding to an isotropic Gaussian mixture model with a different number of components \(M\), with the parameterization given by \eqref{eq:model-class}. %

\begin{figure}
    \centering

    \resizebox{\textwidth}{!}{
        \makebox{
            \begin{tikzpicture}

    \definecolor{darkgray176}{RGB}{176,176,176}
    \definecolor{gray}{RGB}{128,128,128}
    
    \begin{axis}[
    width=2.875cm,
    height=3cm,
    scale only axis,
    tick pos=both,
    title={Memorization Ratio},
    x grid style={darkgray176},
    xlabel={\(\displaystyle M/N\)},
    xmin=0.055, xmax=1.045,
    xtick style={color=black},
    xtick={0,0.2,0.4,0.6,0.8,1,1.2},
    xticklabels={
      \(\displaystyle {0.0}\),
      \(\displaystyle {0.2}\),
      \(\displaystyle {0.4}\),
      \(\displaystyle {0.6}\),
      \(\displaystyle {0.8}\),
      \(\displaystyle {1.0}\),
      \(\displaystyle {1.2}\)
    },
    y grid style={darkgray176},
    ylabel={Ratio},
    ymin=-0.05, ymax=1.05,
    ytick style={color=black},
    ytick={-0.2,0,0.2,0.4,0.6,0.8,1,1.2},
    yticklabels={
      \(\displaystyle {\ensuremath{-}0.2}\),
      \(\displaystyle {0.0}\),
      \(\displaystyle {0.2}\),
      \(\displaystyle {0.4}\),
      \(\displaystyle {0.6}\),
      \(\displaystyle {0.8}\),
      \(\displaystyle {1.0}\),
      \(\displaystyle {1.2}\)
    }
    ]
    \addplot [thick, blue, mark=*, mark size=1, mark options={solid}]
    table {%
    0.1 0
    0.2 0
    0.3 0
    0.4 0
    0.5 0
    0.6 0.0199999995529652
    0.64 0.0199999995529652
    0.68 0.0399999991059303
    0.7 0.140000000596046
    0.72 0.0999999940395355
    0.76 0.239999994635582
    0.8 0.560000002384186
    0.84 0.620000004768372
    0.88 0.799999952316284
    0.9 0.759999990463257
    0.92 0.919999957084656
    0.96 1
    1 1
    };
    \fill [gray!50!white, opacity=0.3] (axis cs:0.6,-0.05) rectangle (axis cs:1,1.05);
    \end{axis}
    
    \end{tikzpicture}
            \begin{tikzpicture}

    \definecolor{darkgray176}{RGB}{176,176,176}
    \definecolor{goldenrod}{RGB}{218,165,32}
    \definecolor{gray}{RGB}{128,128,128}
    \definecolor{green}{RGB}{0,128,0}
    
    \begin{axis}[
    width=5cm,
    height=3cm,
    scale only axis,
    legend cell align={left},
    legend style={
      fill opacity=0.8,
      draw opacity=1,
      text opacity=1,
      at={(0.0,0.02)},
      anchor=south west,
      draw=none,
      scale=0.75,
      font=\footnotesize
    },
    tick pos=both,
    title={Training Loss, \(\displaystyle \psi=6.704\)},
    x grid style={darkgray176},
    xlabel={\(\displaystyle M/N\)},
    xmin=-0.0395, xmax=1.0495,
    xtick style={color=black},
    xtick={-0.2,0,0.2,0.4,0.6,0.8,1,1.2},
    xticklabels={
      \(\displaystyle {\ensuremath{-}0.2}\),
      \(\displaystyle {0.0}\),
      \(\displaystyle {0.2}\),
      \(\displaystyle {0.4}\),
      \(\displaystyle {0.6}\),
      \(\displaystyle {0.8}\),
      \(\displaystyle {1.0}\),
      \(\displaystyle {1.2}\)
    },
    y grid style={darkgray176},
    ylabel={\(\displaystyle \mathcal{L}_{N, t}\)},
    ymin=0, ymax=10,
    ytick style={color=black},
    ]
    \addplot [thick, red]
    table {%
    0.01 181.170684814453
    0.02 133.855880737305
    0.03 123.162231445312
    0.04 116.370750427246
    0.05 123.61043548584
    0.06 104.535087585449
    0.07 98.3118743896484
    0.08 98.60205078125
    0.09 93.9722900390625
    0.1 95.2687225341797
    0.11 85.3343124389648
    0.12 83.6117477416992
    0.13 85.9586563110352
    0.14 78.6683044433594
    0.15 74.3079605102539
    0.16 72.0888290405273
    0.17 77.8091659545898
    0.18 73.4287033081055
    0.19 72.0889511108398
    0.2 77.3866577148438
    0.21 66.0567169189453
    0.22 68.5534439086914
    0.23 65.0179672241211
    0.24 66.7331314086914
    0.25 62.5315055847168
    0.26 62.551399230957
    0.27 60.9723968505859
    0.28 59.436595916748
    0.29 58.6361389160156
    0.3 56.4673614501953
    0.31 55.6729431152344
    0.32 54.2639579772949
    0.33 52.8108215332031
    0.34 53.0099906921387
    0.35 51.9190979003906
    0.36 51.7803840637207
    0.37 50.8909034729004
    0.38 49.1069107055664
    0.39 48.7385520935059
    0.4 47.5291976928711
    0.41 46.6154174804688
    0.42 45.3658599853516
    0.43 45.2809143066406
    0.44 44.1082420349121
    0.45 43.0011367797852
    0.46 41.9924278259277
    0.47 41.0245018005371
    0.48 40.7452278137207
    0.49 39.1224517822266
    0.5 38.3772315979004
    0.51 37.641674041748
    0.52 36.9856147766113
    0.53 35.4440650939941
    0.54 35.1983261108398
    0.55 34.8605690002441
    0.56 34.1064758300781
    0.57 32.9464874267578
    0.58 32.2447776794434
    0.59 30.739200592041
    0.6 30.1946067810059
    0.61 29.5526313781738
    0.62 29.4174747467041
    0.63 28.1524066925049
    0.64 26.9600772857666
    0.65 26.0730400085449
    0.66 24.9251556396484
    0.67 24.3658218383789
    0.68 24.357551574707
    0.69 22.5032386779785
    0.7 22.3039016723633
    0.71 22.0489330291748
    0.72 20.2791290283203
    0.73 19.7330188751221
    0.74 19.1006164550781
    0.75 18.4120788574219
    0.76 17.2008361816406
    0.77 16.6683349609375
    0.78 17.1075801849365
    0.79 15.435453414917
    0.8 14.7518835067749
    0.81 14.1077671051025
    0.82 12.8226547241211
    0.83 12.6238822937012
    0.84 12.1201686859131
    0.85 11.183557510376
    0.86 10.3495101928711
    0.87 9.71383857727051
    0.88 9.09174346923828
    0.89 7.93989086151123
    0.9 7.10683584213257
    0.91 6.54107666015625
    0.92 5.47833681106567
    0.93 5.42826700210571
    0.94 4.56988859176636
    0.95 3.86432957649231
    0.96 2.98018598556519
    0.97 2.03940987586975
    0.98 1.45248878002167
    0.99 0.714033365249634
    1 9.67585454575599e-14
    };
    \addlegendentry{$\sL_{N, t}(\denoiserpmem)$}
    \addplot [thick, goldenrod, dash pattern=on 7.4pt off 3.2pt]
    table {%
    -0.0395 9.68486155530393e-14
    1.0495 9.68486155530393e-14
    };
    \addlegendentry{$\sL_{N, t}(\denoisermem)$}
    \addplot [thick, green, dash pattern=on 7.4pt off 3.2pt]
    table {%
    -0.0395 6.47629642486574
    1.0495 6.47629642486574
    };
    \addlegendentry{$\sL_{N, t}(\denoisergen)$}
    \addplot [thick, blue, mark=*, mark size=1, mark options={solid}]
    table {%
    0.1 6.37981796264648
    0.2 6.20295000076294
    0.3 6.03699111938477
    0.4 5.81414318084717
    0.5 5.52396059036255
    0.6 5.1610255241394
    0.64 4.92577171325684
    0.68 4.76335334777832
    0.7 4.61235618591309
    0.72 4.44972991943359
    0.76 4.2041711807251
    0.8 3.79291558265686
    0.84 3.29685664176941
    0.88 2.7765588760376
    0.9 2.4839882850647
    0.92 2.11749267578125
    0.96 1.27712404727936
    1 2.21679408163311e-09
    };
    \addlegendentry{$\sL_{N, t}(\denoiser_{\theta})$}
    \fill [gray!50!white, opacity=0.3] (axis cs:0.6,1.54080320277688e-14) rectangle (axis cs:1,1055.21813775424);
    \end{axis}
    
    \end{tikzpicture}
            \begin{tikzpicture}

    \definecolor{darkgray176}{RGB}{176,176,176}
    \definecolor{goldenrod}{RGB}{218,165,32}
    \definecolor{gray}{RGB}{128,128,128}
    \definecolor{green}{RGB}{0,128,0}
    
    \begin{axis}[
    width=5cm,
    height=3cm,
    scale only axis,
    legend cell align={left},
    legend style={
      fill opacity=0.8,
      draw opacity=1,
      text opacity=1,
      at={(0,0.05)},
      anchor=south west,
      draw=none,
      scale=0.75,
      font=\footnotesize
    },
    tick pos=both,
    title={Test Loss, \(\displaystyle \psi=6.704\)},
    x grid style={darkgray176},
    xlabel={\(\displaystyle M/N\)},
    xmin=-0.0395, xmax=1.0495,
    xtick style={color=black},
    xtick={-0.2,0,0.2,0.4,0.6,0.8,1,1.2},
    xticklabels={
      \(\displaystyle {\ensuremath{-}0.2}\),
      \(\displaystyle {0.0}\),
      \(\displaystyle {0.2}\),
      \(\displaystyle {0.4}\),
      \(\displaystyle {0.6}\),
      \(\displaystyle {0.8}\),
      \(\displaystyle {1.0}\),
      \(\displaystyle {1.2}\)
    },
    y grid style={darkgray176},
    ylabel={\(\displaystyle \mathcal{L}_{t}\)},
    ymin=0, ymax=100,
    ytick style={color=black},
    ]
    \addplot [thick, red]
    table {%
    0.01 182.648956298828
    0.02 138.235626220703
    0.03 130.79035949707
    0.04 121.793327331543
    0.05 128.342742919922
    0.06 115.030586242676
    0.07 104.186859130859
    0.08 105.444831848145
    0.09 103.654624938965
    0.1 104.414688110352
    0.11 96.5981979370117
    0.12 93.131217956543
    0.13 93.7290649414062
    0.14 91.0798721313477
    0.15 89.1186599731445
    0.16 84.5350494384766
    0.17 93.0750732421875
    0.18 89.8771362304688
    0.19 86.3851547241211
    0.2 94.2025756835938
    0.21 84.696159362793
    0.22 84.9070587158203
    0.23 85.5024185180664
    0.24 86.8489608764648
    0.25 81.2818984985352
    0.26 82.7845611572266
    0.27 82.2338180541992
    0.28 81.298828125
    0.29 80.6306381225586
    0.3 79.9975891113281
    0.31 80.6664352416992
    0.32 79.3000946044922
    0.33 78.6838607788086
    0.34 79.8263397216797
    0.35 77.5668182373047
    0.36 79.1170120239258
    0.37 79.8742828369141
    0.38 76.3500061035156
    0.39 77.6649475097656
    0.4 79.0714416503906
    0.41 77.2842330932617
    0.42 76.9665603637695
    0.43 77.9435501098633
    0.44 77.3519287109375
    0.45 78.3345565795898
    0.46 76.1928405761719
    0.47 76.5116348266602
    0.48 75.4101486206055
    0.49 75.4235534667969
    0.5 77.8530807495117
    0.51 74.6759643554688
    0.52 76.5148010253906
    0.53 73.9760665893555
    0.54 76.9460220336914
    0.55 73.607551574707
    0.56 74.1822814941406
    0.57 75.8538055419922
    0.58 73.1226501464844
    0.59 73.8936309814453
    0.6 75.2544326782227
    0.61 74.2612762451172
    0.62 75.2601165771484
    0.63 74.4327774047852
    0.64 73.9657745361328
    0.65 72.3457794189453
    0.66 73.536247253418
    0.67 73.299201965332
    0.68 74.4795532226562
    0.69 73.5752410888672
    0.7 73.9506149291992
    0.71 72.0766296386719
    0.72 73.2578277587891
    0.73 72.8269882202148
    0.74 72.0265197753906
    0.75 72.3651962280273
    0.76 72.4191207885742
    0.77 71.8012924194336
    0.78 71.1602783203125
    0.79 71.6493911743164
    0.8 72.5336685180664
    0.81 71.4257431030273
    0.82 72.2184524536133
    0.83 72.2934722900391
    0.84 70.7857666015625
    0.85 71.5216827392578
    0.86 71.0553283691406
    0.87 70.759391784668
    0.88 70.9833374023438
    0.89 70.896110534668
    0.9 70.8670196533203
    0.91 70.9092102050781
    0.92 71.4205474853516
    0.93 70.654052734375
    0.94 70.7886657714844
    0.95 70.7557220458984
    0.96 70.4239349365234
    0.97 70.7647323608398
    0.98 70.4163208007812
    0.99 70.4398345947266
    1 70.2450637817383
    };
    \addlegendentry{$\sL_{t}(\denoiserpmem)$}
    \addplot [thick, goldenrod, dash pattern=on 7.4pt off 3.2pt]
    table {%
    -0.0395 70.3246459960937
    1.0495 70.3246459960937
    };
    \addlegendentry{$\sL_{t}(\denoisermem)$}
    \addplot [thick, green, dash pattern=on 7.4pt off 3.2pt]
    table {%
    -0.0395 6.49759244918823
    1.0495 6.49759244918823
    };
    \addlegendentry{$\sL_{t}(\denoisergen)$}
    \addplot [thick, blue, mark=*, mark size=1, mark options={solid}]
    table {%
    0.1 6.58203268051147
    0.2 6.75295495986938
    0.3 6.97911357879639
    0.4 7.45148324966431
    0.5 8.18113040924072
    0.6 9.52731895446777
    0.64 10.5725078582764
    0.68 11.3796195983887
    0.7 12.1607332229614
    0.72 13.0520534515381
    0.76 14.8885374069214
    0.8 17.8848819732666
    0.84 22.1123580932617
    0.88 28.0964393615723
    0.9 31.2064895629883
    0.92 35.3720779418945
    0.96 47.515811920166
    1 70.433837890625
    };
    \addlegendentry{$\sL_{t}(\denoiser_{\theta})$}
    \fill [gray!50!white, opacity=0.3] (axis cs:0.6,0.0) rectangle (axis cs:1,215.804467612237);
    \end{axis}
    
    \end{tikzpicture}
        }
    }

    \caption{\textbf{We observe a clear phase transition between from generalization to memorization in trained models.} \textit{Left:} A plot of the ``memorization ratio'' --- the ratio of generated samples that are memorized over a set of \(N_{\mathrm{eval}} = 50\) generations --- as the number of components of the trained denoiser increases. The start and end of the phase transition, as formally defined at the beginning of \Cref{sec:experiments}, are bracketed by the gray shaded region. \textit{Middle:} A plot of the training loss of the trained model in terms of the number of components, compared to partially memorizing, fully memorizing, and ground truth denoisers, at a representative \(t \approx 0.3\) with SNR \(\snr_{t} \approx 6.704\). \textit{Right:} A plot of the test loss \(\mathcal{L}_{t}\), estimated over a hold-out set, in the same setting. \textit{We  observe that there is a phase transition from generalization to memorization, visible through the behavior of the memorization ratio and loss plots}; the model stops generalizing (i.e., having an acceptable test loss) only when it starts memorizing (i.e., generating samples which are approximately contained in the training data).}
    \label{fig:memorization_phase_transition_loss}
    \vspace{-1.5em}
\end{figure}
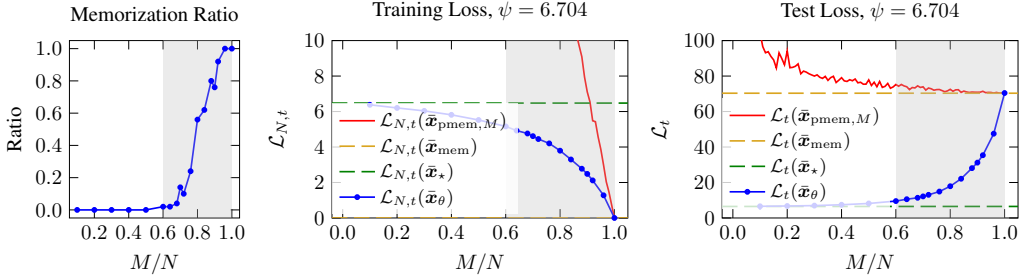

\begin{wrapfigure}{r}{2.5in}
    \vspace{-1.5em}
    \centering 
    \begin{tikzpicture}

\definecolor{darkgray176}{RGB}{176,176,176}
\definecolor{teal1293165}{RGB}{12,93,165}

\begin{axis}[
    width=1.5in,
    height=1.75cm,
    scale only axis,
log basis x={10},
tick pos=both,
title={Loss Weighting},
x grid style={darkgray176},
xlabel={\(\displaystyle \psi\)},
xmin=0.000735795363068942, xmax=2722173.25256142,
xmode=log,
xtick style={color=black},
xtick={1e-06,0.0001,0.01,1,100,10000,1000000,100000000,10000000000},
xticklabels={
  \(\displaystyle {10^{-6}}\),
  \(\displaystyle {10^{-4}}\),
  \(\displaystyle {10^{-2}}\),
  \(\displaystyle {10^{0}}\),
  \(\displaystyle {10^{2}}\),
  \(\displaystyle {10^{4}}\),
  \(\displaystyle {10^{6}}\),
  \(\displaystyle {10^{8}}\),
  \(\displaystyle {10^{10}}\)
},
y grid style={darkgray176},
ylabel={\(\displaystyle \tilde{\lambda}\)},
ymin=0.152767486125231, ymax=0.907359131425619,
ytick style={color=black},
ytick={
0.1,
0.2,
0.4,
0.6,
0.8,
1
},
yticklabels={
  \(\displaystyle {0.1}\),
  \(\displaystyle {0.2}\),
  \(\displaystyle {0.4}\),
  \(\displaystyle {0.6}\),
  \(\displaystyle {0.8}\),
  \(\displaystyle {1.0}\)
}
]
\addplot [teal1293165]
table {%
0.00200296496041119 0.384152472019196
0.0871520787477493 0.400420159101486
0.183634877204895 0.420388728380203
0.293555736541748 0.445237457752228
0.419530898332596 0.476641178131104
0.564846515655518 0.516896486282349
0.733673274517059 0.569263994693756
0.931372106075287 0.637539625167847
1.16492128372192 0.723576009273529
1.44354367256165 0.817021071910858
1.77963066101074 0.873059511184692
2.19014239311218 0.816907703876495
2.69879412651062 0.688570499420166
3.33955407142639 0.564615249633789
4.16247320175171 0.465678185224533
5.2437539100647 0.391677170991898
6.70406150817871 0.33677265048027
8.74368858337402 0.295649349689484
11.7151069641113 0.264496803283691
16.2861194610596 0.240740552544594
23.8506736755371 0.222644239664078
37.7325706481934 0.209055483341217
67.5730895996094 0.199142768979073
152.019729614258 0.192393571138382
596.212341308594 0.188411310315132
999998.75 0.187067106366158
};
\end{axis}

\end{tikzpicture}
    \vspace{-0.75em}
    \caption{\textbf{We can accurately predict the phase transition using our loss approximations.} The optimal loss weighting as per \eqref{eq:weighting_regression_problem} using normalized \textit{approximate losses}. 
    Train and test errors are \(\leq 2 \times 10^{-4}\) when the regression targets are  \( \approx 10^{0}\).%
    }
    \label{fig:phase_transition_predictable_approximation}
    \vspace{-1.5em}
\end{wrapfigure}
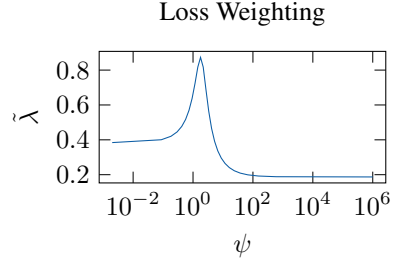
\paragraph{Existence of a phase transition.} The first and arguably most critical property of the trained denoisers in our laboratory is that \textit{there exists a phase transition from generalization to memorization as the model size increases}, which we show in \Cref{fig:memorization_phase_transition_loss}. Namely, we observe in the rightmost panel that initially, heavily underparameterized trained denoisers exhibit statistical generalization (i.e., similar train and test loss to the ground truth denoiser \(\denoisergen\)); then, as the model size \(M\) increases, we observe a rapid transition to \textit{memorization} where the generated samples are memorized (nearly) 100\% of the time (center and left panels). We emphasize that this replicates the central observations of many empirical studies of memorization in diffusion models, such as \citet{zhang2023emergence,kadkhodaie2023generalization}.%

\paragraph{Predicting the phase transition.} Next, we show (via \Cref{fig:phase_transition_predictable_approximation}) that the phase transition is \textit{predictable} using \textit{only the approximate training losses derived in \Cref{sec:training_losses}}; namely, it serves as the ``crossover point'' of the integrated approximate training loss $\hat{\sL}_{N}(\cdot, \tilde{\lambda})$ for a particular timestep weighting \(\tilde{\lambda}(t)\), and this crossover point is \textit{a linear function} of the number of training samples \(N\). To show this, we compute \(\tilde{\lambda}(t)\) as follows. First, we create a grid of \((N, d, K)\) tuples and train several denoisers in each setting in order to estimate the location of the phase transition \(M_{\mathrm{pt}}(N, d, K)\), the first $M$ such that the trained denoiser \(\bar{\vx}_{\theta}\) is memorizing on average. Then, we solve the following optimization problem in the loss weighting \(\tilde{\lambda}\)  to make \(M_{\mathrm{pt}}(N, d, K)\) close to the crossover point:
\begin{equation}\label{eq:weighting_regression_problem}
    \compressstyle\min_{\tilde{\lambda}}\sum_{(N, d, K)}\left(\frac{\tilde{M}_{\mathrm{pt}}(N, d, K, \tilde{\lambda})}{N} - \frac{M_{\mathrm{pt}}(N, d, K)}{N}\right)^{2} ,
\end{equation}
where \(\tilde{M}_{\mathrm{pt}}(N, d, K, \tilde{\lambda})\) is the nearest \(M\) to the crossover point with loss weighting \(\tilde{\lambda}\), i.e.,
\begin{equation}
\compressstyle
    \resizebox{0.925\textwidth}{!}{\(\compressstyle \tilde{M}_{\mathrm{pt}}(N, d, K, \tilde{\lambda}) = \argmin_{M}\left(\sum_{\ell = 0}^{L}\tilde{\lambda}(t_{\ell})\{\hat{\sL}_{N, t}(\denoiserpmem(N, d, K)) - \hat{\sL}_{N, t}(\denoisergen(N, d, K))\}\right)^{2}.\)}
\end{equation}
The optimization and normalization details are postponed to \Cref{app:sub_loss_weighting}.%
In \Cref{fig:phase_transition_predictable_approximation}, we report that the train error and test error (evaluated on a holdout set) are less than \(2 \times 10^{-4}\), signifying that we are able to compute the location of the memorization phase transition within one or two \(M\)'s on average, making our predictive model extremely accurate. %
Moreover, the recovered \(\tilde{M}_{\mathrm{pt}}\) is always a \textit{linear function} of \(N\), namely \(\tilde{M}_{\mathrm{pt}}(N, d, K, \tilde{\lambda}) = (4/5)N\). Therefore, \(M_{\mathrm{pt}}(N, d, K) \approx \tilde{M}_{\mathrm{pt}}(N, d, K, \tilde{\lambda}) = (4/5)N\), demonstrating experimentally that the consequences of our hypothesis (e.g., \eqref{eq:crossover-point}) do indeed hold.
Overall, \textit{the generalization-memorization phase transition is effectively predictable via our hypothesis and subsequent theoretical characterization of the loss}.

\subsection{Experiments for a Simple Image Model}\label{sub:experiments_image}

In the previous \Cref{sub:experiments_gmm}, we considered data which belonged to an isotropic Gaussian mixture model. To showcase the versatility of our Gaussian mixture model formulation, we construct a (low-rank) Gaussian mixture model whose samples resemble natural images. 
Namely, we consider a flattened and monochromatic $d \times d$ image (``template'') \(\bm{x}_{\star} \in \bbR^{d^{2}}\). Now, we endow \(\bm{x}_{\star}\) with a \textit{color} \(\bm{c}\), sampled randomly as \(\bm{c} \sim \mathcal{N}(\vu_{\star}, \sigma_{\star}^{2}\bm{I}) \in \bbR^{c}\). The colored output \(\bm{y} \in \bbR^{cd^{2}}\) %
is given by \(\vy = \vc \kron \vx_{\star}\)
where \(\kron\) is the Kronecker product. Defining the matrix \(\bm{A}_{\bm{x}} := \bm{I} \kron \bm{x}\) we have \(\bm{y} = \bm{A}_{\bm{x}_{\star}}\bm{c}\). The colored image \(\bm{Y} \in \bbR^{c \times d \times d}\) is obtained by reshaping \(\bm{y}\).  Since \(\bm{c} \sim \mathcal{N}(\vu_{\star}, \sigma_{\star}^{2}\vI)\) and \(\bm{y} = \bm{A}_{\bm{x}_{\star}}\bm{c}\), it holds that \(\bm{y} \sim \mathcal{N}(\bm{A}_{\bm{x}_{\star}}\vu_{\star}, \sigma_{\star}^{2}\bm{A}_{\bm{x}_{\star}}\bm{A}_{\bm{x}_{\star}}^{\top})\).
Note, in particular, that \(\bm{y}\) is a low-rank Gaussian random variable. 
By combining multiple templates and color distributions, we obtain a mixture of low-rank Gaussians for our data distribution, namely, \(\bm{y} \sim \pi_{\star} := (1/K)\sum_{i = 1}^{K}\mathcal{N}(\bm{A}_{\star}^{i}\vu_{\star}^{i}, \sigma_{\star}^{2}\bm{A}_{\star}^{i}(\bm{A}_{\star}^{i})^{\top})\). This is an instance of \Cref{eq:ground-truth-pi-anisotropic}, and so we can compute its denoiser via \Cref{lemma:denoiser_mog}. In \Cref{app:sub_image_model_simplified} we discuss some ways to efficiently implement this class of low-rank denoisers. We visualize some samples in \Cref{fig:image_data_samples} with FashionMNIST templates \citep{xiao2017fashion}.

\begin{figure}
\centering
    \includegraphics[width=.85\textwidth]{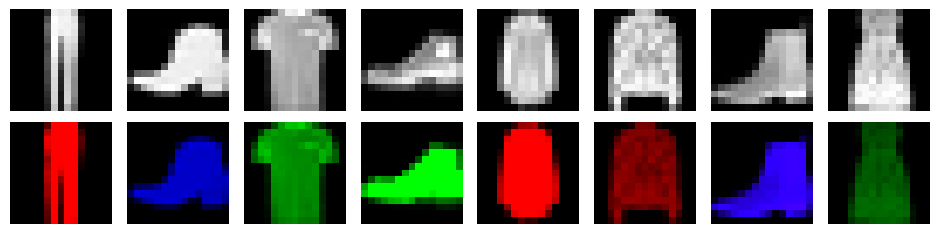}
    \caption{\textbf{Our memorization laboratory enables modeling of natural image distributions with latent low-dimensional structure.} {Some sample synthetic images from our simple image model (\textit{bottom}) visualized alongside their corresponding monochromatic image templates (\textit{top}).}}
    \label{fig:image_data_samples}
\end{figure}

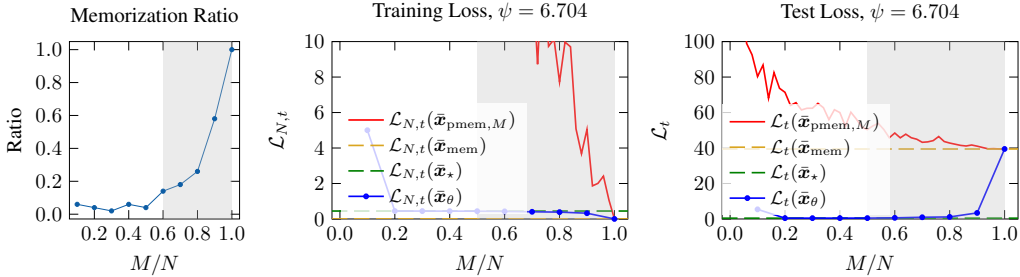
\begin{figure}
    \centering
    \resizebox{\textwidth}{!}{
        \makebox{
            \begin{tikzpicture}

    \definecolor{darkgray176}{RGB}{176,176,176}
    \definecolor{gray}{RGB}{128,128,128}
    \definecolor{teal1293165}{RGB}{12,93,165}
    
    \begin{axis}[
    width=2.875cm,
    height=3cm,
    scale only axis,
    tick pos=both,
    title={Memorization Ratio},
    x grid style={darkgray176},
    xlabel={\(\displaystyle M/N\)},
    xmin=0.055, xmax=1.045,
    xtick style={color=black},
    xtick={0,0.2,0.4,0.6,0.8,1,1.2},
    xticklabels={
      \(\displaystyle {0.0}\),
      \(\displaystyle {0.2}\),
      \(\displaystyle {0.4}\),
      \(\displaystyle {0.6}\),
      \(\displaystyle {0.8}\),
      \(\displaystyle {1.0}\),
      \(\displaystyle {1.2}\)
    },
    y grid style={darkgray176},
    ylabel={Ratio},
    ymin=-0.0290000004693866, ymax=1.04900000002235,
    ytick style={color=black},
    ytick={-0.2,0,0.2,0.4,0.6,0.8,1,1.2},
    yticklabels={
      \(\displaystyle {\ensuremath{-}0.2}\),
      \(\displaystyle {0.0}\),
      \(\displaystyle {0.2}\),
      \(\displaystyle {0.4}\),
      \(\displaystyle {0.6}\),
      \(\displaystyle {0.8}\),
      \(\displaystyle {1.0}\),
      \(\displaystyle {1.2}\)
    }
    ]
    \addplot [teal1293165, mark=*, mark size=1, mark options={solid}]
    table {%
    0.1 0.0599999986588955
    0.2 0.0399999991059303
    0.3 0.0199999995529652
    0.4 0.0599999986588955
    0.5 0.0399999991059303
    0.6 0.140000000596046
    0.7 0.179999992251396
    0.8 0.259999990463257
    0.9 0.579999983310699
    1 1
    };
    \fill [gray!50!white, opacity=0.3] (axis cs:0.6,-0.05) rectangle (axis cs:1,1.05);
    \end{axis}
    
    \end{tikzpicture}
            \begin{tikzpicture}

    \definecolor{darkgray176}{RGB}{176,176,176}
    \definecolor{goldenrod}{RGB}{218,165,32}
    \definecolor{gray}{RGB}{128,128,128}
    \definecolor{green}{RGB}{0,128,0}
    
    \begin{axis}[
        width=5cm,
        height=3cm,
        scale only axis,
    legend cell align={left},
    legend style={
      fill opacity=0.8,
      draw opacity=1,
      text opacity=1,
      at={(0.03,0.03)},
      anchor=south west,
      draw=none,
      scale=0.75,
      font=\footnotesize
    },
    tick pos=both,
    title={Training Loss, \(\displaystyle \snr=6.704\)},
    x grid style={darkgray176},
    xlabel={\(\displaystyle M/N\)},
    xmin=-0.029, xmax=1.049,
    xtick style={color=black},
    xtick={-0.2,0,0.2,0.4,0.6,0.8,1,1.2},
    xticklabels={
      \(\displaystyle {\ensuremath{-}0.2}\),
      \(\displaystyle {0.0}\),
      \(\displaystyle {0.2}\),
      \(\displaystyle {0.4}\),
      \(\displaystyle {0.6}\),
      \(\displaystyle {0.8}\),
      \(\displaystyle {1.0}\),
      \(\displaystyle {1.2}\)
    },
    y grid style={darkgray176},
    ylabel={\(\displaystyle \mathcal{L}_{N, t}\)},
    ymin=0, ymax=10,
    ytick style={color=black},
    ]
    \addplot [thick, red]
    table {%
    0.02 289.244140625
    0.04 199.549499511719
    0.06 96.8224029541016
    0.08 90.3000030517578
    0.1 76.6411361694336
    0.12 75.2649536132812
    0.14 66.6453475952148
    0.16 61.5230751037598
    0.18 69.0685195922852
    0.2 60.5547637939453
    0.22 51.6056442260742
    0.24 44.5550842285156
    0.26 48.357666015625
    0.28 42.9808502197266
    0.3 50.6403694152832
    0.32 47.947883605957
    0.34 41.1473999023438
    0.36 40.076229095459
    0.38 34.8987731933594
    0.4 30.8178768157959
    0.42 27.6831111907959
    0.44 29.5842247009277
    0.46 28.517068862915
    0.48 26.8566017150879
    0.5 23.3164958953857
    0.52 21.9266872406006
    0.54 25.1359081268311
    0.56 26.2124614715576
    0.58 17.0111713409424
    0.6 21.1423454284668
    0.62 15.2161741256714
    0.64 13.6478681564331
    0.66 17.7582530975342
    0.68 11.9048089981079
    0.7 14.3488616943359
    0.72 8.72116756439209
    0.74 15.0086107254028
    0.76 9.73533058166504
    0.78 10.1327590942383
    0.8 7.72761201858521
    0.82 9.97124481201172
    0.84 9.70571422576904
    0.86 5.07817649841309
    0.88 3.68169522285461
    0.9 5.04862117767334
    0.92 1.86594557762146
    0.94 2.00706005096436
    0.96 2.41379547119141
    0.98 1.06889677047729
    1 0.00210812175646424
    };
    \addlegendentry{$\sL_{N, t}(\denoiserpmem)$}
    \addplot [thick, goldenrod, dash pattern=on 7.4pt off 3.2pt]
    table {%
    -0.029 0.00381505605764687
    1.049 0.00381505605764687
    };
    \addlegendentry{$\sL_{N, t}(\denoisermem)$}
    \addplot [thick, green, dash pattern=on 7.4pt off 3.2pt]
    table {%
    -0.029 0.44788333773613
    1.049 0.44788333773613
    };
    \addlegendentry{$\sL_{N, t}(\denoisergen)$}
    \addplot [thick, blue, mark=*, mark size=1, mark options={solid}]
    table {%
    0.1 5.00232744216919
    0.2 0.447927176952362
    0.3 0.444556027650833
    0.4 0.44161194562912
    0.5 0.438873946666718
    0.6 0.427575618028641
    0.7 0.404855787754059
    0.8 0.386125713586807
    0.9 0.325212121009827
    1 0.00261645810678601
    };
    \addlegendentry{$\sL_{N, t}(\denoiser_{\theta})$}
    \fill [gray!50!white, opacity=0.3] (axis cs:0.5,0.00116688282153916) rectangle (axis cs:1,522.556210894483);
    \end{axis}
    
    \end{tikzpicture}
            \begin{tikzpicture}

    \definecolor{darkgray176}{RGB}{176,176,176}
    \definecolor{goldenrod}{RGB}{218,165,32}
    \definecolor{gray}{RGB}{128,128,128}
    \definecolor{green}{RGB}{0,128,0}
    
    \begin{axis}[
        width=5cm,
        height=3cm,
        scale only axis,
    legend cell align={left},
    legend style={
      fill opacity=0.8,
      draw opacity=1,
      text opacity=1,
      at={(0.0,0.02)},
      anchor=south west,
      draw=none,
      scale=0.75,
      font=\footnotesize
    },
    tick pos=both,
    title={Test Loss, \(\displaystyle \psi=6.704\)},
    x grid style={darkgray176},
    xlabel={\(\displaystyle M/N\)},
    xmin=-0.029, xmax=1.049,
    xtick style={color=black},
    xtick={-0.2,0,0.2,0.4,0.6,0.8,1,1.2},
    xticklabels={
      \(\displaystyle {\ensuremath{-}0.2}\),
      \(\displaystyle {0.0}\),
      \(\displaystyle {0.2}\),
      \(\displaystyle {0.4}\),
      \(\displaystyle {0.6}\),
      \(\displaystyle {0.8}\),
      \(\displaystyle {1.0}\),
      \(\displaystyle {1.2}\)
    },
    y grid style={darkgray176},
    ylabel={\(\displaystyle \mathcal{L}_{t}\)},
    ymin=0, ymax=100,
    ytick style={color=black},
    ]
    \addplot [thick, red]
    table {%
    0.02 316.628356933594
    0.04 245.716262817383
    0.06 99.0438232421875
    0.08 92.5423583984375
    0.1 80.3125381469727
    0.12 86.7876052856445
    0.14 68.0760345458984
    0.16 82.4900817871094
    0.18 73.6733551025391
    0.2 71.4006881713867
    0.22 64.0250778198242
    0.24 65.5560989379883
    0.26 61.4714813232422
    0.28 62.3947525024414
    0.3 62.441822052002
    0.32 65.1530838012695
    0.34 62.7442665100098
    0.36 56.0199241638184
    0.38 61.5885848999023
    0.4 59.9716987609863
    0.42 52.2216873168945
    0.44 60.1818199157715
    0.46 56.3050308227539
    0.48 51.8107986450195
    0.5 50.3637962341309
    0.52 52.4514389038086
    0.54 53.8685684204102
    0.56 50.4419708251953
    0.58 45.782341003418
    0.6 48.4453964233398
    0.62 45.7440605163574
    0.64 47.8025894165039
    0.66 45.7916374206543
    0.68 44.9627571105957
    0.7 43.2200546264648
    0.72 43.4549942016602
    0.74 44.1711235046387
    0.76 46.4839172363281
    0.78 43.658390045166
    0.8 42.8034362792969
    0.82 42.1383323669434
    0.84 41.4754066467285
    0.86 41.1378211975098
    0.88 41.6982498168945
    0.9 41.0512771606445
    0.92 40.2810134887695
    0.94 39.5459861755371
    0.96 39.628963470459
    0.98 39.4949684143066
    1 39.4652366638184
    };
    \addlegendentry{$\sL_{t}(\denoiserpmem)$}
    \addplot [thick, goldenrod, dash pattern=on 7.4pt off 3.2pt]
    table {%
    -0.029 39.4441986083985
    1.049 39.4441986083985
    };
    \addlegendentry{$\sL_{t}(\denoisermem)$}
    \addplot [thick, green, dash pattern=on 7.4pt off 3.2pt]
    table {%
    -0.029 0.449452310800552
    1.049 0.449452310800552
    };
    \addlegendentry{$\sL_{t}(\denoisergen)$}
    \addplot [thick, blue, mark=*, mark size=1, mark options={solid}]
    table {%
    0.1 5.42839956283569
    0.2 0.456748485565186
    0.3 0.466995030641556
    0.4 0.475321739912033
    0.5 0.481677114963531
    0.6 0.635057508945465
    0.7 0.966801583766937
    0.8 1.15631520748138
    0.9 3.40600347518921
    1 39.4684562683105
    };
    \addlegendentry{$\sL_{t}(\denoiser_{\theta})$}
    \fill [gray!50!white, opacity=0.3] (axis cs:0.5,0.323810038410463) rectangle (axis cs:1,439.484048695223);
    \end{axis}
    
    \end{tikzpicture}
        }
    }
    \caption{\textbf{A phase transition persists in the low-rank Gaussian natural image model.} \textit{Left:} The memorization ratio as the number of components of the trained denoiser increases while everything else is fixed. \textit{Middle:} The training loss of the trained model in terms of the number of components, compared to partially memorizing, fully memorizing, and ground truth denoisers, at a representative \(t \approx 0.3\) with SNR \(\snr \approx 6.704\). \textit{Right:} The test loss, computed over a hold-out set, in the same setting. 
     We emphasize the identical qualitative picture to \Cref{fig:memorization_phase_transition_loss}, wherein we can observe a phase transition from generalization to memorization from the memorization ratio and loss plots. Transient ``jaggedness'' may be explained by larger variance in the loss estimation.}
    \label{fig:image_data_phase_transition}
    \vspace{-1em}
\end{figure}

Our main result demonstrates that our framework is still expressive enough to model this facsimile of real-world image distributions, and that it exhibits similar phenomena to the isotropic case, enabling it to be studied in detail via our laboratory. Notably, we show in \Cref{fig:image_data_phase_transition} that we can still identify a phase transition phenomenon, which looks qualitatively similar to the isotropic Gaussian mixture model studied in \Cref{sub:experiments_gmm} and throughout the work.

\vspace{-0.1cm}
\section{Related Works}
\label{sec:related_works}

Understanding memorization and generalization properties of diffusion
models is crucial for practitioners \citep{somepalli2023diffusion,
    ren2024unveiling, rahman2024frame,
    wang2024replication,chen2024investigating,stein2024exposing}.
Indeed, concerns about privacy
\citep{ghalebikesabi2023differentially, carlini2023extracting,
    nasr2023scalable} and copyright infringement
\citep{cui2023diffusionshield,
    wang2024replication,vyas2023provable,franceschelli2022copyright} are
key issues as  these models are deployed. Several factors can
influence the memorization capabilities of diffusion models such as duplication \citep{carlini2023extracting,ross2024geometric,webster2023reproducible} or model architecture \citep{chavhan2024memorized}. We refer to \citep{gu2023memorization, kadkhodaie2023generalization} for an in-depth experimental
investigation of those issues. 

On the theoretical side, memorization in diffusion models has been
investigated through the lens of statistical physics leveraging the
concept of \emph{phase transition}
\citep{biroli2024dynamical,li2023generalization,ambrogioni2023statistical,ventura2024manifolds,raya2024spontaneous,
    sakamoto2024geometry, pavasovic2025understanding}. In particular in \citep{biroli2024dynamical},
the authors identify three critical transitions for the diffusion
model generative trajectories assuming that the score is assumed to be
perfectly learned. \citet{george2025denoising}  investigated generalization and memorization in trained random features neural network denoisers (nonparametric models) in the case where the data is a standard Gaussian. Our work is complementary to the theories of ``creativity'' and generalization of \citet{kamb2024analytic,niedoba2024towards,vastola2025generalization}, which suggest in different contexts that generalization in diffusion models arises from the implicit bias of an underparameterized denoiser (\citet{vastola2025generalization} also considers the landscape of the training objective, see \citep{bertrand2025closed} for a rebuttal). 
Our work disentangles the competing factors in the implicit biases discussed in these works and captures the essential features into our theoretical laboratory. While in this work we only study in detail the simplest and most interpretable instantiation of the laboratory, further connections from our framework to such previous work may be possible and enable us to broaden the scope of practical settings for which we have robust theoretical results. %

\vspace{-0.1cm}
\section{Conclusion}
\vspace{-0.1cm}

In this paper, we have introduced a theoretical laboratory for measuring and predicting memorization and creativity in diffusion models. 
Focusing on data drawn from $K$-component Gaussian mixture models and Gaussian mixture denoisers with a number of components ranging from $K$ to the number of training samples $N$, our laboratory disentangles different factors contributing to memorization and generalization, enabling us to compute tight theoretical approximations for the losses of representative denoisers within this model class and thereby {predict} the onset of memorization in trained models at inference time.

While our current framework allows us to study generalization and memorization behavior in a rigorous and testable setting, we highlight several avenues of improvement. First, our model can be extended to capture additional properties of larger and more realistic datasets such as intrinsic dimensionality or partial data replication. In future work, we plan on expanding the memorization/generalization laboratory to cover those cases and refining the asymptotics of our training losses. We envision the laboratory growing to encompass a framework for understanding memorization and generalization purely in terms of the \textit{geometry of the data}, with a robust and extensive experimental apparatus to test hypotheses and verify predictions, and a rich theoretical toolkit that reduces statistical questions of sampling in trained diffusion models to geometric questions about the data itself.

\paragraph{Acknowledgements.} YM acknowledges support from the joint Simons Foundation-NSF DMS grant \#2031899, the ONR grant N00014-22-1-2102, the NSF grant \#2402951, and the startup fund from the University of Hong Kong.

\bibliography{refs}
\bibliographystyle{unsrtnat}

\appendix

\newpage

\section*{Organization of the Appendix}

In \Cref{sec:gmm-denoiser-calcs}, we recall some basic calculations around
Gaussian Mixture Model denoisers and provide additional discussion of our
setting.
Justifications about our model class are given in \Cref{app:softmax-infinity}. 
Next, we recall some results on Gaussian concentration bounds in \Cref{sec:gmm-denoiser-calcs-cont}.
In \Cref{sec:softmax_approximation}, we present key results for approximating softmax operators.
We combine our concentration bounds and softmax approximation results in \Cref{sec:high_dim_denoiser} in order to provide approximation of the denoisers in generalizing and memorizing scenarios. 
Our main results regarding the training loss approximations such as
\Cref{thm:asymptotics-generalizing} are proved in \Cref{sec:training_loss_approx}. 
Finally, full experimental details are provided in
\Cref{app:experiments}.

\section{Gaussian Mixture Model Denoisers: Basic
  Calculations}\label{sec:gmm-denoiser-calcs}

\paragraph{Justification of Gaussian Mixture Models.} 
Gaussian mixture models represent a canonical model for assessing learning and
sampling algorithms in theoretical computer science, including in the
context of diffusion models \citep{Dasgupta2007-sh,Ge2018-ny,Shah2023-qw,Gatmiry2024-yg},
and they have the appealing ability to model data with \textit{geometric
structure}, including hierarchical structure \citep{Li2024-ly} and
low-dimensional structure in natural images \citep{Zoran2011-jx,wang2024diffusion}
Most importantly, the task of training and sampling with a diffusion model on
a Gaussian mixture target $\pi_\star$ via \eqref{eq:denoising-objective} represents an
ideal test-bed for investigating issues of memorization and generalization,
because as the number of components in the mixture is varied, and in particular
is as made as large as the number of training data samples $N$, one can
simultaneously represent the true distribution
\eqref{eq:ground-truth-pi-anisotropic} and the memorizing denoiser
\eqref{eq:memorizing-denoiser-empirically-optimal}
within the same class of models.

\paragraph{Discussion around the denoisers.}
The key technical challenges we investigate in our laboratory setting are the characterization
of the denoiser $\denoiser_\theta$ minimizing \eqref{eq:training_loss_defns} (or,
equivalently its parameters $\theta$), and using its properties to prove that
its associated sampling measure $\hat{\pi}$ either memorizes or generalizes.
For the first challenge, typical theoretical studies of
regression over a parametric class of models (as in
\eqref{eq:denoising-objective})
rely on the model class being well-specified with respect to the true model
parameters, here those corresponding to $\pi_\star$. In
\eqref{eq:denoising-objective},
this is the case when $M = K$, but as soon as $M > K$, existing studies of the
loss landscape for diffusion model training with Gaussian mixture model
data/denoisers break down \citep{wang2024diffusion,Shah2023-qw}. With regards to
both challenges, a clear complication that we have alluded to previously is that
as $M \to N$, the model class becomes expressive enough to represent the
optimal, memorizing denoiser \eqref{eq:memorizing-denoiser-empirically-optimal},
which is suggestive that at intermediate values of $M$, parameters $\theta$ that
lead to generalizing-denoiser-like behavior of $\hat{\pi}$ are no longer
prevalent.
To overcome these challenges, we adopt the hybrid theoretical-empirical
methodology outlined in the main body, which our laboratory setting enables.

\paragraph{Basic denoiser calculations.}

We recall \Cref{lemma:denoiser_mog}.

\begin{lemma}{}{denoiser_mog_appendix}
    Assume that $\pi_\theta = (1/M) \sum_{i=1}^{M} \sN(\vmu^{i},
        \vSigma^i)$. Then, we have that
    \begin{equation}\label{eq:model-class-general-appendix}
    \compressstyle
        \denoiser_\theta(t,\vx_t) = 
        \frac{1}{\alpha_t}\left(
        \vx_t -
        \sigma^{2}_t\sum_{i =
        1}^{M}(\alpha_t^{2}\bm{\Sigma}^{i}
        + \sigma_t^{2}\bm{I})^{-1}(\vx_t - \alpha_t \bm{\mu}^{i})
        \operatorname{softmax}(\bm{w}(t, \vx_t))_{i}
        \right),
    \end{equation}
    where for all $i \in [M]$
    \begin{equation}\label{eq:softmax-vector-model-general}
        w_i(t, \vx_t) = 
        -\frac{1}{2}\log\det(\alpha_t^{2}\bm{\Sigma}^{i} + \sigma_t^{2}\bm{I}) -
        \frac{1}{2}(\vx_t-
        \alpha_t\bm{\mu}^{i})^{\top}(\alpha_t^{2}\bm{\Sigma}^{i} +
        \sigma_t^{2}\bm{I})^{-1}(\vx_t - \alpha_t \bm{\mu}^{i}) 
        .
    \end{equation}
\end{lemma}

In the rest of this section, we are going to prove \Cref{lemma:denoiser_mog}.
This property is well-known and can be found in
\citep{peluchetti2023non,biroli2024dynamical,kamb2024analytic}
for instance but we include its proof for
completeness.
First, we compute $p_t$. Recalling the noising process,
\eqref{eq:noising_time_t}, we get that
for any $t \in [0,1]$ and $\vx_t \in \rset^d$

\begin{align}
    p_t(\vx_t) & = \int_{\rset^d} p_{t|0}(\vx_t|\vx_0) \rmd \pi(\vx_0)                         \\
               & =(1/M) \sum_{i=1}^{M} \int_{\rset^d} p_{t|0}(\vx_t|\vx_0) \sN(\vmu^{i},
    \vSigma^i)(\vx_0) \rmd \vx_0                                                    \\
               & = (1/M) \sum_{i=1}^{M} \sN(\vx_t; \alpha_t \vmu^{i}, 
               \sigma_t^2 \vI + \alpha_t^2 \vSigma^i) .
\end{align}
Therefore, we get that for any $t \in [0,1]$ and $\vx_t \in \rset^d$
\begin{align}
    \nabla \log p_t(\vx_t) & = \nabla p_t(\vx_t) / p_t(\vx_t)                                         \\
                           & = \frac{\sum_{i=1}^{M} (\sigma_t^2 \vI + \alpha_t^2
                           \vSigma^i)^{-1}(\alpha_t \vmu^{i}
                           - \vx_t)\sN(\vx_t; \alpha_t \vmu^{i}, \sigma_t^2 \vI
                           + \alpha_t^2 \vSigma^i)}{\sum_{i=1}^{M} \sN(\vx_t; \alpha_t \vmu^{i},
    \sigma_t^2 \vI + \alpha_t^2 \vSigma^i)}                                                  \\
                           & = {\sum_{i=1}^{M} (\sigma_t^2 \vI + \alpha_t^2
                           \vSigma^i)^{-1}(\alpha_t \vmu^{i}
                           - \vx_t)\softmax(\vw(\vx_t))_i}, \label{eq:mog_stein_score}
\end{align}
with $\vw$ defined as in the lemma statement.
Finally, using Tweedie's identity we get that for any $t \in [0,1]$
and $x_t \in \rset^d$
\begin{equation}
    \nabla \log p_t(\vx_t) = (\alpha_t \denoiser_\theta(t, \vx_t) - \vx_t) / \sigma_t^2 .
\end{equation}
Therefore combining this result and \eqref{eq:mog_stein_score}, we get that for
any $t \in [0,1]$ and $\vx_t \in \rset^d$
    \begin{equation}
        \denoiser_\theta(t,\vx_t) = 
        \frac{1}{\alpha_t}\left(
        \vx_t -
        \sigma^{2}_t\sum_{i =
        1}^{M}(\alpha_t^{2}\bm{\Sigma}^{i}
        + \sigma_t^{2}\bm{I})^{-1}(\vx_t - \alpha_t \bm{\mu}^{i})
        \operatorname{softmax}(\bm{w}(\vx_t))_{i}
        \right),
    \end{equation}
which concludes the proof.

\section{Nesting of Model Classes}\label{app:softmax-infinity}

Here we will sketch a rigorous proof of the claim that the loss
\eqref{eq:training_loss_defns} associated 
to any $M$-parameter GMM denoiser can be represented, in a limiting sense, by
that of a $(M+1)$-parameter GMM denoiser.
For simplicity, we will treat the case where $\vSigma^i = \vSigma$ for $i \in
[M]$, i.e.\ all components have the same variance.
Given such a $M$-parameter GMM denoiser following the form in
\Cref{lemma:denoiser_mog}, we show that the $M+1$ parameter
model given by the parameters $(\vmu^1, \vSigma, \dots, \vmu^M, \vSigma,
m \vmu^M, \vSigma)$, for $m \in \bbN$, provides a suitable loss approximation as
$m \to \infty$. Comparing the denoisers in \Cref{lemma:denoiser_mog}, we see
that it suffices to show a suitable degree of approximation of the softmax
autoregression
\begin{equation}
    \begin{bmatrix}
        \vmu^1
         &
         \hdots
         &
         \vmu^M
         &
        m \vmu^M
    \end{bmatrix}
    \softmax(\vw_{M+1}(\vx_t))
    \approx
    \begin{bmatrix}
        \vmu^1
         &
         \hdots
         &
         \vmu^M
    \end{bmatrix}
    \softmax(\vw_{M}(\vx_t)),
\end{equation}
where the weight vectors are subscripted to denote the number of parameters.
Because we are concerned with loss approximations, it suffices to show this
approximation for $\vx_t$ given by a noisy sample from the GMM $\pi_\star$,
following \Cref{eq:noising_time_t}. We have
\begin{align}
    &\begin{bmatrix}
        \vmu^1
         &
         \hdots
         &
         \vmu^M
         &
        m \vmu^M
    \end{bmatrix}
    \softmax(\vw_{M+1}(\vx_t))
    \\
    &\qquad=
    m \vmu^M
    \frac{
        e^{
            -\half(m\alpha_t \vmu^M - \vx_t)^\top (\alpha_t^2 \vSigma
            + \sigma_t^2 \vI)^{-1} (m\alpha_t \vmu^M - \vx_t)
            }
    }{\sum_{\ell=1}^M e^{w_{\ell}} + e^{
            -\half(m\alpha_t \vmu^M - \vx_t)^\top (\alpha_t^2 \vSigma
            + \sigma_t^2 \vI)^{-1} (m\alpha_t \vmu^M - \vx_t)
            }}
            \\
            &\qquad\qquad
    + \sum_{k=1}^M \vmu^k
    \frac{e^{w_k}}{\sum_{\ell=1}^M e^{w_{\ell}} + e^{
            -\half(m\alpha_t \vmu^M - \vx_t)^\top (\alpha_t^2 \vSigma
            + \sigma_t^2 \vI)^{-1} (m\alpha_t \vmu^M - \vx_t)
            }},
\end{align}
where we write $w_k$ for the $k$-th entry of $\vw_{M+1}$ above.
From here, we can construct a high-probability event (using Gaussian
concentration) on which $X_t^i$ is bounded for any sample $\vx^i$ initializing
the noising process (coming from \eqref{eq:training_loss_defns}), and then it
follows by an application of Gaussian concentration that on this event,
\begin{equation}
    \lim_{m\to\infty}
    e^{
            -\half(m\alpha_t \vmu^M - \vx_t)^\top (\alpha_t^2 \vSigma
            + \sigma_t^2 \vI)^{-1} (m\alpha_t \vmu^M - \vx_t)
            }
            =0,
\end{equation}
whence by the previous expression
\begin{equation}
    \begin{bmatrix}
        \vmu^1
         &
         \hdots
         &
         \vmu^M
         &
        m \vmu^M
    \end{bmatrix}
    \softmax(\vw_{M+1}(\vx_t))
    \to_{m\to\infty}
    \begin{bmatrix}
        \vmu^1
         &
         \hdots
         &
         \vmu^M
    \end{bmatrix}
    \softmax(\vw_{M}(\vx_t)).
\end{equation}
Given that loss approximations for \eqref{eq:training_loss_defns} only evaluate on $X_t^i$
for which we can construct the aforementioned high-probability event, this
establishes the claim that $(M+1)$-parameter models' losses can achieve any
$M$-parameter model's loss.

\section{Concentration Bounds}\label{sec:gmm-denoiser-calcs-cont}

\subsection{Chernoff-Cramér bound for $\chi_2$}

In this section, we give some basic results regarding the concentration bounds of $\chi_2$ random variables. Those lemmas will be key to establish our main sparsity results in \Cref{sec:high_dim_denoiser}.
First, we recall the Chernoff-Cramér bound, see \citep[page 21]{boucheron2003concentration} for instance.

\begin{theorem}{}{chernoff_bound}
    Let $X$ be a real-valued random variable. Let $M(t) =
        \mathbb{E}[\exp[tX]]$. Then, we have that for any $a \in \rset$,
    $\mathbb{P}(X \leq a) \leq \inf_{t > 0} M(t) \exp[-ta]$. In
    particular if $X \sim \chi_2(p)$ with $p \in \nset$, we get that for
    any $\vareps \in (0,1)$
    \begin{equation}
        \label{eq:chernoff_chi_square_up}
        \mathbb{P}(X \leq (1-\vareps)p) \leq \exp\left[ \frac{p}{2}(\vareps +
            \log(1-\vareps)) \right] ,
    \end{equation}
    and
    \begin{equation}
        \label{eq:chernoff_chi_square_down}
        \mathbb{P}(X \geq (1+\vareps)p) \leq \exp\left[ -\frac{p}{2}(\vareps
            - \log(1+\vareps)) \right] .
    \end{equation}
    As a consequence for any $\vareps \in [0,1]$
    \begin{equation}
        \label{eq:chernoff_chi_square_up_down}
        \mathbb{P}(\abs{X - p} \geq \veps p) \leq 2 \exp\left[ -\frac{p\veps^2}{8}
            \right].
    \end{equation}

\end{theorem}

\begin{proof}
    We have that \eqref{eq:chernoff_chi_square_up} and \eqref{eq:chernoff_chi_square_down} are direct consequences of the Chernoff-Cramér bounds. Then for any $x \in [0,1]$, we have that $\log(1-x) + x \leq -x^2/2$ and $\log(1+x) - x \leq -x^2/2 + x^3/6$. Hence, we have that for any $x \in [0,1]$, we have that $\log(1-x) + x \leq -x^2/8$ and $\log(1+x) - x \leq -x^2/8$ which concludes the proof of \eqref{eq:chernoff_chi_square_up_down} using an union bound.
\end{proof}

One of the main application of \Cref{thm:chernoff_bound} is to establish the concentration of the squared norm of Gaussian random variables.

\begin{lemma}{}{gaussian-norm-nonzero-mean}
    Let $\vg \sim \sN(\vmu, \sigma^2 \vI)$ be an isotropic Gaussian with mean
    $\vmu$ and covariance $\sigma^2 \vI$.
    Then for any $0 \leq \veps \leq 1$, one has
    \begin{equation}
        \Pr*{
            \abs*{\norm{\vg}^2 - (\norm{\vmu}^2 + \sigma^2 d)}
            \geq \veps \sigma\sqrt{d}(\sigma \sqrt{d} + \norm{\vmu})
        }
        \leq 4 \exp[-d \veps^2 / 8].
    \end{equation}

\end{lemma}
\begin{proof}
    We have $\vg \equid \vmu + \sigma \vw$, where $\vw \sim \sN(\Zero, \vI)$, so
    \begin{align}
        {\norm{\vg}^2} \equid {\norm{\vmu + \sigma \vw}^2}
        = \norm*{\vmu}^2 + \sigma^2\norm{\vw}^2 + 2 \sigma \ip{\vmu}{\vw}
         & \equid \norm*{\vmu}^2 + \sigma^2\norm{\vw}^2 + 2 \sigma \ip{\vmu}{\vw}  \\
         & \equid \norm{\vmu}^2 + \sigma^2\norm{\vw}^2 + 2\sigma \norm{\vmu} w_1
        \labelthis \label{eq:gaussian-nonzero-mean-rotinvar}
    \end{align}
    where the last line follows by rotational invariance of the Gaussian
    distribution,
    and therefore in particular
    \begin{equation}
        \E*{\norm{\vg}^2}
        = \norm*{\vmu}^2 + d \sigma^2.
    \end{equation}
    By \Cref{thm:chernoff_bound},
    we have for every $0 \leq \veps \leq 1$
    \begin{equation}
        \Pr*{ \abs*{\norm{\vw}^2 - d} \geq d\veps } \leq 2 \exp[-d\veps^2 / 8],
    \end{equation}
    and by Gaussian concentration, for any $t \geq 0$, we have
    \begin{equation}
        \Pr*{\abs{w_1} \geq t} \leq 2 \exp[-t^2 / 2],
    \end{equation}
    so in particular
    \begin{equation}
        \Pr*{\abs{w_1} \geq \frac{\veps \sqrt{d}}{2}} \leq 2 \exp[-d \veps^2 / 8].
    \end{equation}
    Thus, using a union bound, we have for any $\vareps \in [0,1]$
    \begin{align}
         & \Pr*{
            \abs*{\norm{\vg}^2 - (\norm{\vmu}^2 + \sigma^2 d)}
            \geq \veps \sigma\sqrt{d}(\sigma \sqrt{d} + \norm{\vmu})
        }                                       \\
         & \qquad = \Pr*{
            \abs*{\sigma^2\norm{\vw}^2 + 2\sigma \norm{\vmu} w_1 -  \sigma^2 d}
            \geq \veps \sigma\sqrt{d}(\sigma \sqrt{d} + \norm{\vmu})
        }                                       \\
         & \qquad \leq \Pr*{
            \sigma^2\abs*{\norm{\vw}^2 - d} + 2\sigma \norm{\vmu} \abs*{w_1}
            \geq \veps \sigma\sqrt{d}(\sigma \sqrt{d} + \norm{\vmu})
        }                                       \\
         & \qquad \leq \Pr*{
            \sigma^2 \left[\abs*{\norm{\vw}^2 - d} - d\veps \right] + 2\sigma \norm{\vmu} \left[ \abs*{w_1} - \frac{\vareps \sqrt{d}}{2}\right]
            \geq 0
        }                                       \\
         & \qquad \leq 4 \exp[-d \veps^2 / 8] ,
    \end{align}
    which concludes the proof.
\end{proof}

The following result will be used in the proof of \Cref{thm:partial-memorizing-denoiser-loss-approximation_appendix}. 

\begin{proposition}{}{control_softmax_norm}
    Let $(\vg^i)_{i=1}^n$ be i.i.d Gaussian random variables $\mathcal{N}(0, \sigma_\star^2 \Id)$ and $(\vw)_{i=1}^{n-1}$ a collection of positive random variables. Then, for any $\vareps \in (0,1)$, we have with probability at least $1 - 6n \exp[-d\vareps^2/2]$
    \begin{equation}
        d\sigma_\star^2(1-3\vareps) \leq \norm*{\sum_{j=1}^{n-1} \softmax(\vw)_j \vx^j - \vx^n}^2 \leq 2d\sigma_\star^2(1+ 3\vareps) . 
    \end{equation}
\end{proposition}

\begin{proof}
    First, we have that 
    \begin{align}
        \norm*{\sum_{j=1}^{n-1} \softmax(\vw)_j \vx^j - \vx^n}^2 &= \norm*{\sum_{j=1}^{n-1} \softmax(\vw)_j (\vx^j - \vx^n)}^2 \\
        &\leq \sum_{j=1}^{n-1} \softmax(\vw)_j \norm*{\vx^j - \vx^n}^2 . 
    \end{align}
    Therefore, using that $\vx^j - \vx^n$ is a Gaussian random variable $\mathcal{N}(0, 2\sigma_\star^2 \Id)$, we get using a union bound and \Cref{thm:chernoff_bound}, that with probability $1 - 2 n \exp[-\vareps^2 d / 8]$
    \begin{equation}
        \norm*{\sum_{j=1}^{n-1} \softmax(\vw)_j \vx^j - \vx^n}^2 \leq 2d\sigma_\star^2(1+\vareps) . 
    \end{equation}
    For the second part of the proof, we have that 
    \begin{align}
        \norm*{\sum_{j=1}^{n-1} \softmax(\vw)_j \vx^j - \vx^n}^2 &= \norm*{\vx^n}^2 + \norm*{\sum_{j=1}^{n-1} \softmax(\vw)_j \vx^j}^2 - 2 \sum_{j=1}^{n-1}  \softmax(\vw)_j \langle \vx^j , \vx^n \rangle \\
        &\geq \norm*{\vx^n}^2 - 2 \sum_{j=1}^{n-1}  \softmax(\vw)_j \langle \vx^j , \vx^n \rangle \\
        &\geq \norm*{\vx^n}^2 - 2 \sum_{j=1}^{n-1}  \softmax(\vw)_j |\vx^j_1| \norm*{\vx^n}  \\
        &\geq \sum_{j=1}^{n-1} \softmax(\vw)_j  \{\norm*{\vx^n}^2 - 2|\vx^j_1| \norm*{\vx^n}\} . 
    \end{align}
    We have that with probability at
    least $1 - 2 n \exp[-d\veps^2 / 8]$ that $\abs{\vx^j_1} \leq \tfrac{\sigma_\star \veps
            \sqrt{d}}{2}$.
    Combining this result with \Cref{thm:chernoff_bound} and a union bound we get that, on an event of probability at least
    $1 - 4 n \exp[-d\veps^2 / 8]$,
    \begin{align}
        \norm*{\sum_{j=1}^{n-1} \softmax(\vw)_j \vx^j - \vx^n}^2 &\geq \sum_{j=1}^{n-1} \softmax(\vw)_j  \{\norm*{\vx^n}^2 - |\vx^j_1| \norm*{\vx^n}\}\\
        &\geq \sum_{j=1}^{n-1} \softmax(\vw)_j  \{\norm*{\vx^n}^2 - \veps \sqrt{d} \norm*{\vx^n}\} \\
        &\geq \sigma_\star^2 d (1- \veps) - \veps \sigma_\star^2 d (1+\veps)^{1/2} \geq d(1 - 3\veps) ,
    \end{align}
    which concludes the proof upon using a union bound. 
\end{proof}

\subsection{Coupon Collector Bounds}

In order to derive our results in the case of the partial memorizing denoiser, we will consider the following result from the coupon collector's problem.

\begin{proposition}{}{coupon_collector}
    Let $K \in \nset$ be the number of means.
    Consider the distribution $\pi = (1/K) \sum_{k=1}^{K} \sN(\vmu_\star^{k},
        \sigma_{\star}^2 \vI)$ and let $(\vx^i)_{i=1}^N$ be a collection of i.i.d samples from $\pi$. For any $\vx^i$ denote $k_i \in [K]$ the index of the associated mean, i.e. $\vx^i = \mu^{k_i}_\star + \sigma_\star \vw^i$, with $\vw^i \sim \mathcal{N}(0, \vI)$.
    Let $\mathsf{S} = (\vx^i)_{i=1}^\ell$ such that $\ell \geq (1+\log(d)) K \log(K)$. For any $k \in [K]$, denote $\mathsf{A}_k$ the event such that there exists $i_1, i_2 \in [\ell]$ such that $k_{i_1} = k_{i_2} = k$. Finally, denote $\mathsf{A} = \bigcap_{k=1}^K \mathsf{A}_k$.
    Then
    \begin{equation}
        \Pr*{\mathsf{A}} \geq 1 - K^{-\log(d)} (1 + \log(K) \log(d)) .
    \end{equation}
\end{proposition}

This result is a simple control of the tails of the coupon collector problem. In fact, much more precise estimates could be derived, see \citep{erdHos1961classical} for instance.

\begin{proof}
    We only deal with the case $K \geq 2$. The case $K = 1$ is trivial.
    First, note that we have for any $k \in [K]$
    \begin{align}
        \Pr*{\mathsf{A}_k^{\mathrm{c}}} & = \left( 1 - \frac{1}{K}\right)^\ell + \frac{\ell}{K}\left( 1 - \frac{1}{K}\right)^{\ell - 1} \\
                                        & \leq \left( 1 - \frac{1}{K}\right)^\ell \left(1 + \frac{\ell}{K-1}\right)                     \\
                                        & \leq \exp[-\ell/K] \left(1 + \frac{\ell}{K-1}\right)  .
    \end{align}
    Note that $t \mapsto \exp[-t/K]\frac{t}{K-1}$ is decreasing on $[K, +\infty)$ and therefore we get that
    \begin{equation}
        \Pr*{\mathsf{A}_k^{\mathrm{c}}} \leq \exp[-(1+\log(d))\log(K)] \left(1 + \frac{K \log(K) \log(d)}{K-1}\right) .
    \end{equation}
    Using a union bound and that $K / (K-1) \leq 2$ we get
    \begin{equation}
        \Pr*{\mathsf{A}} \geq 1 - 2 K^{-\log(d)} \left( 1 + \log(K) \log(d) \right) ,
    \end{equation}
    which concludes the proof.
\end{proof}

\section{Softmax Approximation}
\label{sec:softmax_approximation}

\paragraph{Low-Temperature Behavior.}
The following elementary lemma is useful.
It shows that the key quantity controlling $1$-sparsity of the softmax is the
scale of the temperature $T$ relative to the gap between the largest and
second-largest element of the softmax weight vector. For completeness we recall the definition of the $\softmax$ operation. For any $\vv \in \rset^n$ we have that $\vv \in \rset^n$ and for any $i \in \{1, \dots, n\}$
\begin{equation}
    \softmax(\vv)_i =  \frac{\exp[ v_i]}{\sum_{k=1}^n \exp[ v_k]}.
\end{equation}

\begin{lemma}{}{softmax-1sparse-lp}
    Let $\vv \in \bbR^n$ be such that $v_i \neq v_j$ for $i \neq j$.
    Let $k = \argmax_{k' \in \{1, \dots, n\}} v_{k'}$ denote the index of the largest element of
    $\vv$, and define
    \begin{equation}
        \gamma = \min_{i \neq k} \left(v_k - v_i\right) ,
    \end{equation}
    as the ``gap'' between the largest and second-largest element of $\vv$.
    Then for any $p \geq 1$, one has
    \begin{equation}
        \norm*{\softmax(\vv / T) - \ve_k}_p \leq 2(n-1) e^{-\gamma / T}.
    \end{equation}

\end{lemma}

\begin{proof}
    The idea is to notice that taking the ratio between elements of the softmax
    has a simple expression, namely
    \begin{equation}
        \frac{\softmax(\vv / T)_i}{\softmax(\vv / T)_k}  = e^{-(v_k - v_i) / T},
    \end{equation}
    where $k$ is as in the statement of the lemma.
    By definition of $k$, we then have for every $i \neq k$
    \begin{equation}
        {\softmax(\vv / T)_i}
        \leq e^{-\gamma / T} {\softmax(\vv / T)_k}
        \leq e^{-\gamma / T},
    \end{equation}
    and thus
    \begin{equation}
        \softmax(\vv/T)_k \geq 1 - (n-1) e^{-\gamma / T}.
    \end{equation}
    This gives
    \begin{equation}
        \norm*{\softmax(\vv / T) - \ve_k}_p^p
        \leq
        (n-1)^p e^{-p\gamma / T} + (n-1)e^{-p\gamma / T},
    \end{equation}
    so in particular
    \begin{align}
        \norm*{\softmax(\vv / T) - \ve_k}_p
         & \leq
        (n-1) e^{-\gamma / T} \left(
        1 + (n-1)^{-1+1/p}
        \right)
        \\
         & \leq
        2(n-1) e^{-\gamma / T}.
    \end{align}

\end{proof}

Note that to guarantee approximation in $\ell^p$, in the worst case (reflected
in the proof) it is necessary that the temperature depends logarithmically on
the number of vector elements $n$.

This proof operates in a worst-case regime where every non-maximal element
of the vector $\vv$ may have the same magnitude. In reality, if there is
a more structured distribution of non-maximizers, the estimates improve
correspondingly. The proof could be improved to capture this by using
a different, more precise measure of the distribution of entries of $\vv$.
For example, if some precise rate of decay of the entry distribution could
be asserted, it seems reasonable that this, rather than the vector dimension
$n$, would force the ultimate dependence of $T$ for $\ell^p$ approximation.
It seems reasonable that something like this should obtain for weight
vectors of distances between random vectors.

The following lemma is a slight extension of \Cref{lemma:softmax-1sparse-lp}.

\begin{lemma}{}{softmax-extension}
    Let $\vv \in \bbR^n$ be such that $v_i \neq v_j$ for $i \neq j$.
    Let $k = \argmax_{k' \in \{1, \dots, n\}} v_{k'}$ denote the index of the largest element of
    $\vv$. Let $\mathsf{S}$ be a subset of $[n]$ such that $k \in \mathsf{S}$ and denote
    \begin{equation}
        \gamma_{\mathsf{S}} = \min_{i \notin \mathsf{S}} \left(v_k - v_i\right) ,
    \end{equation}
    as the ``gap'' between the largest of $\vv$ and the largest element not in $\mathsf{S}$.
    Denote
    \begin{equation}
        \softmax(\vv_{|\mathsf{S}} / T)_\ell = \frac{\exp[v_\ell / T]}{\sum_{i \in \mathsf{S}} \exp[v_i / T]} ,
    \end{equation}
    if $\ell \in \mathsf{S}$ and $\softmax(\vv_{|\mathsf{S}} / T)_\ell = 0$ otherwise.
    Then for any $p \geq 1$, one has
    \begin{equation}
        \norm*{\softmax(\vv / T) - \softmax(\vv_{|\mathsf{S}} / T)}_p \leq (1 + | \mathsf{S} |)^{1/p} (n - | \mathsf{S} |) \exp[-\gamma_{\mathsf{S}} / T].
    \end{equation}
\end{lemma}

Note that \Cref{lemma:softmax-1sparse-lp} is a special case of \Cref{lemma:softmax-extension} where $\mathsf{S} = \{k \}$.

\begin{proof}
    Let $\ell \not \in \mathsf{S}$, we have that
    \begin{equation}
        \softmax(\vv/T)_\ell = \softmax(\vv/T)_k \exp[(-v_k+v_i)/T] \leq \exp[-\gamma_\mathsf{S} / T] .
    \end{equation}
    In addition, we have that for any $\ell \in \mathsf{S}$
    \begin{align}
        \softmax(\vv_{|\mathsf{S}} / T)_\ell - \softmax(\vv / T)_\ell & = \frac{\exp[v_\ell / T]}{\sum_{i \in \mathsf{S}} \exp[v_i / T]} - \frac{\exp[v_\ell / T]}{\sum_{i \in \mathsf{S}} \exp[v_i / T] + \sum_{i \notin \mathsf{S}} \exp[v_i / T]}                       \\
                                                                      & = \frac{\exp[v_\ell / T]}{\sum_{i \in \mathsf{S}} \exp[v_i / T]} \frac{\sum_{i \notin \mathsf{S}} \exp[v_i / T]}{\sum_{i \in \mathsf{S}} \exp[v_i / T] + \sum_{i \notin \mathsf{S}} \exp[v_i / T]} \\
                                                                      & \leq \sum_{i \notin \mathsf{S}} \softmax(\vv / T)_i \leq (n - | \mathsf{S} |) \exp[- \gamma_{\mathsf{S}} / T] .
    \end{align}
    Therefore, we get that for any $p \geq 1$
    \begin{align}
        \norm*{\softmax(\vv / T) - \softmax(\vv_{|\mathsf{S}} / T)}_p^p & \leq |\mathsf{S}| (n- | \mathsf{S} |)^p \exp[- \gamma_{\mathsf{S}} p  / T] + (n- | \mathsf{S} |) \exp[- \gamma_{\mathsf{S}} p  / T] \\
                                                                        & \leq (n- | \mathsf{S} |)^p \exp[- \gamma_{\mathsf{S}} p  / T] (| \mathsf{S}| + (n- | \mathsf{S} |)^{-1+ 1/p})                       \\
                                                                        & \leq (n- | \mathsf{S} |)^p \exp[- \gamma_{\mathsf{S}} p  / T] (1 + | \mathsf{S}|) ,
    \end{align}
    which concludes the proof.
\end{proof}

\section{Results on High-Dimensional Denoiser Behavior}
\label{sec:high_dim_denoiser}

In \Cref{sec:gaussian_mixture_softmax}, we present some results which will help us control softmax approximation in \Cref{sec:denoiser_approximation}, where we leverage those results to obtain controls on different denoisers.

\subsection{Gaussian Mixtures and Softmax}
\label{sec:gaussian_mixture_softmax}

The following lemmas are the key to establish our main results.
They allow us to use our softmax sparsity results in the context of diffusion denoisers.
In particular, \Cref{lemma:softmax-weights-mean-gap-control} will be used in \Cref{lemma:true-denoiser-1sparse-lp} while \Cref{lemma:softmax-weights-samples-gap-control} will be used in \Cref{lemma:mem-denoiser-1sparse-lp}.

\begin{lemma}{}{softmax-weights-mean-gap-control}
    Given vectors $(\vmu_{\star}^k)_{k=1}^K$ satisfying
    \begin{equation}
        \min_{k \neq k'}\, \norm*{\vmu_{\star}^k - \vmu_{\star}^{k'}} \geq \gamma
        > 0,
    \end{equation}
    consider the distribution $\pi = (1/K) \sum_{k=1}^{K} \sN(\vmu_\star^{k},
        \sigma_{\star}^2 \vI)$.
    For $i \in [N]$, let $\vx^i \sim \pi$, fix $0 \leq t \leq 1$, and let
    $\vx \sim \alpha_t
        \vx^i + \sigma_t \vg$, where $\vg \sim \sN(\Zero, \vI)$ is independent from $\vx^i$.
    Let $k_i$ denote the index of the (uniquely defined) cluster centroid
    $\vmu_\star^k$ associated to $\vx^i$.
    Then for any $0 \leq \veps \leq 1$ satisfying the coupling condition
    \begin{equation}
        \vareps \leq
        \frac{
            \alpha_t \gamma
        }{
            4\sqrt{d(\alpha_t^2 \sigma_\star^2 + \sigma_t^2)}
        },
    \end{equation}
    one has with probability at least $1 - 4K \exp[-d\veps^2 / 8]$
    \begin{equation}
        \min_{k \neq k_i}\,
        \norm*{\alpha_t \vmu_\star^{k} - \vx}^2
        -
        \norm*{\alpha_t \vmu_\star^{k_i} - \vx}^2
        \geq
        \frac{\gamma^2\alpha_t^2}{2}.
    \end{equation}
\end{lemma}

\begin{proof}
    Start by writing $\vx^i \equid \vmu_{\star}^{k_i} + \sigma_{\star} \vw$, where
    $k_i \in [K]$ is unique and $\vw \sim \sN(\Zero, \vI)$ is independent, then
    write, for any $k \in [K]$,
    \begin{equation}
        \norm*{\alpha_t \vmu_\star^k - \vx}^2
        =
        \norm*{\alpha_t (\vmu_\star^k - \vx^i) - \sigma_t \vg}^2
        =
        \norm*{\alpha_t (\vmu_\star^k - \vmu_\star^{k_i} - \sigma_{\star} \vw)
            - \sigma_t \vg}^2.
    \end{equation}
    This is equal in distribution to the squared $\ell^2$ norm of a Gaussian
    random variable with mean $\alpha_t(\vmu_{\star}^k - \vmu_{\star}^{k_i})$
    and covariance $(\alpha_t^2 \sigma_{\star}^2 + \sigma_t^2)\vI$.
    Applying \Cref{lemma:gaussian-norm-nonzero-mean}, it follows
    \begin{equation}
        \Pr*{
            \abs*{
                \norm*{\alpha_t \vmu_\star^k - \vx}^2
                -
                ( \alpha_t^2\norm*{\vmu_\star^k - \vmu_\star^{k_i}}^2  +
                (\alpha_t^2 \sigma_\star^2 + \sigma_t^2) d )
            }
            \geq \veps \Xi^k
        }
        \leq 3 \exp[-d \veps^2 / 8],
    \end{equation}
    where for concision $\Xi^k = \sqrt{d(\alpha_t^2 \sigma_\star^2 + \sigma_t^2)}
        (\sqrt{d(\alpha_t^2 \sigma_\star^2 + \sigma_t^2)}
        + \alpha_t\norm{\vmu_\star^k - \vmu_\star^{k_i}}) $.
    Taking a union bound over $k \in [K]$, the above control holds for all $k$
    simultaneously with probability at least $1 - 3K \exp[-d\veps^2 / 8]$.
    More precisely, we have that with probability at least $1 - 3K \exp[-d\veps^2 / 8]$, for any $k \in [K]$ with $k \neq k_i$
    \begin{equation}
        \norm*{\alpha_t \vmu_\star^k - \vx}^2
        -
        ( \alpha_t^2\norm*{\vmu_\star^k - \vmu_\star^{k_i}}^2  +
        (\alpha_t^2 \sigma_\star^2 + \sigma_t^2) d )
        \geq - \veps \Xi^k , \label{eq:ineq_on_other_k}
    \end{equation}
    and
    \begin{equation}
        \norm*{\alpha_t \vmu_\star^{k_i} - \vx}^2
        \leq (1 + \veps) {d(\alpha_t^2 \sigma_\star^2 + \sigma_t^2)},
    \end{equation}
    Using \eqref{eq:ineq_on_other_k}, we have that for any $k \in [K]$ with $k \neq k_i$
    \begin{align}
        \norm*{\alpha_t \vmu_\star^{k} - \vx}^2
         & \geq
        (1 - \veps) {d(\alpha_t^2 \sigma_\star^2 + \sigma_t^2)}
        +
        \alpha_t^2\norm*{\vmu_\star^k - \vmu_\star^{k_i}}^2
        - \veps \alpha_t \sqrt{d(\alpha_t^2 \sigma_\star^2
            + \sigma_t^2)}\norm*{\vmu_\star^k - \vmu_\star^{k_i}}.
    \end{align}

    Consequently, we have on this event that, for each $k \in [K]$ with $k \neq k_i$,
    \begin{align}
        \norm*{\alpha_t \vmu_\star^{k} - \vx}^2
        -
        \norm*{\alpha_t \vmu_\star^{k_i} - \vx}^2
         & \geq
        \alpha_t^2\norm*{\vmu_\star^k - \vmu_\star^{k_i}}^2
        \\
         & \quad-\veps\left(
        2{d(\alpha_t^2 \sigma_\star^2 + \sigma_t^2)}
        + \alpha_t \sqrt{d(\alpha_t^2 \sigma_\star^2 + \sigma_t^2)}
        \norm*{\vmu_\star^k - \vmu_\star^{k_i}}
        \right).
    \end{align}
    To simplify this bound, notice that we have for all $k \in [K]$ with $k \neq k_i$
    \begin{equation}
        \frac{\alpha_t^2}{2}\norm*{\vmu_\star^k - \vmu_\star^{k_i}}^2
        - \veps\alpha_t\sqrt{d(\alpha_t^2 \sigma_\star^2 + \sigma_t^2)}
        \norm*{\vmu_\star^k - \vmu_\star^{k_i}}
        \geq
        \frac{\alpha_t^2}{4}\norm*{\vmu_\star^k - \vmu_\star^{k_i}}^2
    \end{equation}
    if and only if
    \begin{equation}
        \alpha_t\gamma \geq
        4 \veps\sqrt{d(\alpha_t^2 \sigma_\star^2 + \sigma_t^2)}.
    \end{equation}
    Similarly, we have for all $k \neq k_i$
    \begin{equation}
        \frac{\alpha_t^2}{2}\norm*{\vmu_\star^k - \vmu_\star^{k_i}}^2
        - 2\veps d(\alpha_t^2 \sigma_\star^2 + \sigma_t^2)
        \geq
        \frac{\alpha_t^2}{4}\norm*{\vmu_\star^k - \vmu_\star^{k_i}}^2
    \end{equation}
    if and only if
    \begin{equation}
        \alpha_t\gamma \geq
        \sqrt{8 \veps d(\alpha_t^2 \sigma_\star^2 + \sigma_t^2)}.
    \end{equation}
    So it is enough to enforce
    \begin{equation}
        \max \set{\veps, \sqrt{\veps}} \leq
        \frac{
            \alpha_t \gamma
        }{
            4\sqrt{d(\alpha_t^2 \sigma_\star^2 + \sigma_t^2)}
        }
    \end{equation}
    to get that on the aforementioned event, for every $k \in [K]$ with $k \neq k_i$
    \begin{equation}
        \min_{k \neq k_i}\,
        \norm*{\alpha_t \vmu_\star^{k} - \vx}^2
        -
        \norm*{\alpha_t \vmu_\star^{k_i} - \vx}^2
        \geq
        \frac{\gamma^2\alpha_t^2}{2},
    \end{equation}
    which concludes the proof.
\end{proof}

The following lemma is more involved than \Cref{lemma:softmax-weights-mean-gap-control}. The main reason is that contrary to \Cref{lemma:softmax-weights-mean-gap-control}, the softmax weights we are investigating involve not only the means $(\vmu^k_\star)_{k=1}^K$ but the datapoints $(\vx^j)_{j=1}^N$ which are random.

\begin{lemma}{}{softmax-weights-samples-gap-control}
    Given vectors $(\vmu_{\star}^k)_{k=1}^K$ satisfying
    \begin{equation}
        \min_{k \neq k'}\, \norm*{\vmu_{\star}^k - \vmu_{\star}^{k'}} \geq \gamma
        > 0,
    \end{equation}
    consider the distribution $\pi = (1/K) \sum_{k=1}^{K} \sN(\vmu_\star^{k},
        \sigma_{\star}^2 \vI)$.
    For each $j \in [N]$, let $\vx^j \sim \pi$. Fix $i \in [N]$ and $0
        \leq t \leq 1$,
    and let $\vx \sim \alpha_t \vx^i + \sigma_t \vg$, where $\vg \sim \sN(\Zero,
        \vI)$ is independent from $(x^j)_{j=1}^N$.
    Then for any $0 \leq \veps \leq 1$ satisfying the coupling conditions
    \begin{equation}
        \veps \leq
        \frac{
            \alpha_t \gamma
        }{
            2\sqrt{d(2\alpha_t^2 \sigma_\star^2 + \sigma_t^2)}
        }
    \end{equation}
    and
    \begin{equation}
        \frac{\veps^2}{1-\veps} \leq \frac{\alpha_t^2
            \sigma_\star^2}{2\sigma_t^2},
    \end{equation}
    one has with probability at least $1 - 4N \exp[-d\veps^2 / 8]$
    \begin{equation}
        \min_{j \neq i}\,
        \norm*{
            \alpha_t \vx^j - \vx
        }^2
        -
        \norm*{
            \alpha_t \vx^i - \vx
        }^2
        \geq
        \alpha_t^2 \sigma_\star^2(1-\veps) d.
    \end{equation}
\end{lemma}
\begin{proof}
    We start by writing $\vx^j \equid \vmu_{\star}^{k_j}
        + \sigma_{\star} \vw^j$ for each $j \in [N]$, where
    $k_j \in [K]$ is unique and $\vw^j \sim \sN(\Zero, \vI)$ is independent. We will first argue that it is enough to consider only those indices
    $j$ for which the class assignments agree: that is, $k_j = k_i$, where we recall that $\vx \sim \alpha_t \vx^i + \sigma_t \vg$ and therefore $\vx$ is associated with $k_i$.
    We have for any $j \in [N]$
    \begin{equation}
        \alpha_t \vx^j - \vx
        \equid
        \alpha_t \sigma_\star(\vw^j - \vw^i) + \alpha_t (\vmu_\star^{k_j}
        - \vmu_\star^{k_i}) - \sigma_t \vg,
    \end{equation}
    so by rotational invariance
    \begin{align}
        \norm*{
            \alpha_t \vx^j - \vx
        }^2
         & \equid
        \alpha_t^2 \norm*{\vmu_\star^{k_j} - \vmu_\star^{k_i}}^2
        + \norm*{\alpha_t \sigma_\star(\vw^j - \vw^i) - \sigma_t \vg}^2
        \\
         & \quad+ 2 \alpha_t \norm*{\vmu_\star^{k_j} - \vmu_{\star}^{k_i}}
        (\alpha_t \sigma_\star(w^j_1 - w^i_1) - \sigma_t g_1). \label{eq:main_inequality_different_x}
    \end{align}
    In what follows, we consider the case $j \neq i$.

    \paragraph{A first lower bound in the case of different means.} First, suppose that $k_j \neq k_i$.
    The random variable $\alpha_t \sigma_\star(w^j_1 - w^i_1) - \sigma_t g_1$ is
    equal in distribution to a Gaussian random variable with mean zero and
    variance $2\alpha_t^2 \sigma_\star^2 + \sigma_t^2$. Call this random
    variable $X$; by Gaussian concentration, for any $t \geq 0$, we have
    \begin{equation}
        \Pr*{\abs{X} \geq t} \leq 2 \exp[-t^2 / 2(2\alpha_t^2 \sigma_\star^2
            + \sigma_t^2)],
    \end{equation}
    so in particular
    \begin{equation}
        \Pr*{\abs{X} \geq \frac{\veps \sqrt{d(2\alpha_t^2 \sigma_\star^2
                    + \sigma_t^2)}}{2}} \leq 2 \exp[-d \veps^2 / 8 ].
    \end{equation}
    Combining this result with \eqref{eq:main_inequality_different_x}, we get that with probability at least $1 - 2 \abs{\set{j \in [N] \given k_j \neq
                k_i}}\exp[-d\veps^2/8]$, we have for every $j \in [N]$ for which $k_j \neq k_i$
    that
    \begin{align}
         & \norm*{
            \alpha_t \vx^j - \vx
        }^2         \\
         & \qquad\geq
        \alpha_t^2 \norm*{\vmu_\star^{k_j} - \vmu_\star^{k_i}}^2
        + \norm*{\alpha_t \sigma_\star(\vw^j - \vw^i) - \sigma_t \vg}^2
        -  \veps \alpha_t \norm*{\vmu_\star^{k_j} - \vmu_{\star}^{k_i}}
        \sqrt{d(2\alpha_t^2 \sigma_\star^2 + \sigma_t^2)}.
    \end{align}
    Given that
    \begin{equation}
        \min_{k \neq k'}\, \norm*{\vmu_{\star}^k - \vmu_{\star}^{k'}} \geq
        \gamma,
    \end{equation}
    if $\alpha_t \gamma \geq 2\veps \sqrt{d(2\alpha_t^2 \sigma_{\star}^2
            + \sigma_t^2)}$,
    we have on the previous event
    \begin{align}
        \norm*{
            \alpha_t \vx^j - \vx
        }^2
         & \geq
        \frac{\alpha_t^2}{2} \norm*{\vmu_\star^{k_j} - \vmu_\star^{k_i}}^2
        + \norm*{\alpha_t \sigma_\star(\vw^j - \vw^i) - \sigma_t \vg}^2
        \\
         & >
        \norm*{\alpha_t \sigma_\star(\vw^j - \vw^i) - \sigma_t \vg}^2 .
    \end{align}
    Now, note that for those $j$ for which $k_j = k_i$, we have as above
    \begin{equation}
        \norm*{
            \alpha_t \vx^j - \vx
        }^2
        \equid
        \norm*{\alpha_t \sigma_\star(\vw^j - \vw^i) - \sigma_t \vg}^2.
    \end{equation}
    Therefore, we get that with probability $1 - 2 N \exp[-d \vareps^2 / 8]$ we have that
    \begin{equation}
        \norm*{
            \alpha_t \vx^j - \vx
        }^2 \geq \norm*{\alpha_t \sigma_\star(\vw^j - \vw^i) - \sigma_t \vg}^2. \label{eq:lower_bound_first_first}
    \end{equation}

    \paragraph{Lower bound on the difference (first stochasticity level).} We are going to give a lower bound for $j \neq i$ on the quantity
    \begin{equation}
        \norm*{\alpha_t \sigma_\star(\vw^j - \vw^i) - \sigma_t \vg}^2
        -
        \norm*{
            \alpha_t \vx^i - \vx
        }^2 .
    \end{equation}
    We have for $j \neq i$
    \begin{align}
        \norm*{\alpha_t \sigma_\star(\vw^j - \vw^i) - \sigma_t \vg}^2
        -
        \norm*{
            \alpha_t \vx^i - \vx
        }^2
         & \equid
        \norm*{\alpha_t \sigma_\star(\vw^j - \vw^i) - \sigma_t \vg}^2
        -
        \norm*{\sigma_t \vg}^2
        \\
         & =
        \alpha_t^2 \sigma_\star^2\norm*{\vw^j - \vw^i}^2
        - 2\sigma_t\alpha_t\sigma_\star\ip*{\vw^j - \vw^i}{ \vg}
        \\
         & \equid
        \alpha_t^2 \sigma_\star^2\norm*{\vw^j - \vw^i}^2
        - 2\sigma_t\alpha_t\sigma_\star g_1\norm*{\vw^j - \vw^i}
    \end{align}
    where the last line uses rotational invariance of the Gaussian distribution.
    Once again using Gaussian concentration, we have that with probability at
    least $1 - 2 \exp[-d\veps^2 / 8]$ that $\abs{g_1} \leq \tfrac{\veps
            \sqrt{d}}{2}$.
    It follows from a union bound that, on an event of probability at least
    $1 - 2 N  \exp[-d\veps^2 / 8]$,
    it holds
    \begin{equation}
        \norm*{\alpha_t \sigma_\star(\vw^j - \vw^i) - \sigma_t \vg}^2
        -
        \norm*{
            \alpha_t \vx^i - \vx
        }^2
        \geq
        \alpha_t^2 \sigma_\star^2\norm*{\vw^j - \vw^i}^2
        - \veps \sqrt{d}\sigma_t\alpha_t\sigma_\star \norm*{\vw^j - \vw^i}.
    \end{equation}
    \paragraph{Lower bound on the difference (second stochasticity level).} Now, as before, if it holds for all such $j$
    \begin{equation}
        \norm{\vw^j - \vw^i} \geq
        \frac{2\veps \sigma_t\sqrt{d} }{\alpha_t \sigma_\star},
    \end{equation}
    then the preceding bound can be simplified to
    \begin{equation}
        \norm*{\alpha_t \sigma_\star(\vw^j - \vw^i) - \sigma_t \vg}^2
        -
        \norm*{
            \alpha_t \vx^i - \vx
        }^2
        \geq
        \frac{\alpha_t^2 \sigma_\star^2}{2}\norm*{\vw^j - \vw^i}^2.
    \end{equation}
    This leads us to consider the lower tail of the random variable $\min_{j
            \neq i} \norm{\vw^j - \vw^i}$,
    which was studied in \Cref{sec:gmm-denoiser-calcs-cont}.
    A coarser approach will be sufficient for our purposes: when $j \neq i$, the
    random variable $\half\norm{\vw^j - \vw^i}^2$ is distributed as a $\chi_2(d)$
    random variable, so \Cref{thm:chernoff_bound}
    implies that for any $0 \leq \veps' \leq 1$,
    \begin{equation}
        \Pr*{\norm*{\vw^j - \vw^i}^2\geq 2(1-\veps')d}
        \geq 1- \exp\left[ \frac{-d(\veps')^2}{8}\right].
    \end{equation}
    We can for simplicity simply enforce $\veps' = \veps$. In this case, if we
    add the additional condition
    \begin{equation}
        \frac{\veps^2}{1-\veps} \leq \frac{\alpha_t^2
            \sigma_\star^2}{2\sigma_t^2},
    \end{equation}
    then by a union bound, it holds with probability at least $1
        - N \exp[-d\veps^2/8]$
    that for all such $j$,
    \begin{equation}
        \norm*{\alpha_t \sigma_\star(\vw^j - \vw^i) - \sigma_t \vg}^2
        -
        \norm*{
            \alpha_t \vx^i - \vx
        }^2
        \geq \frac{\alpha_t^2 \sigma_\star^2}{2}\norm*{\vw^j - \vw^i}^2 \geq
        \alpha_t^2 \sigma_\star^2(1-\veps) d, \label{eq:lower_bound}
    \end{equation}
    Finally combining \eqref{eq:lower_bound}, \eqref{eq:lower_bound_first_first} and a union bound, we get that with probability at least $1 - 4N \exp[-d\veps^2 / 8]$, we have
    \begin{equation}
        \min_{j \neq i}\,
        \norm*{
            \alpha_t \vx^j - \vx
        }^2
        -
        \norm*{
            \alpha_t \vx^i - \vx
        }^2
        \geq
        \alpha_t^2 \sigma_\star^2(1-\veps) d,
    \end{equation}
    which concludes the proof.
\end{proof}

\begin{lemma}{}{softmax-weights-samples-gap-control-between-mean}
    Given vectors $(\vmu_{\star}^k)_{k=1}^K$ satisfying
    \begin{equation}
        \min_{k \neq k'}\, \norm*{\vmu_{\star}^k - \vmu_{\star}^{k'}} \geq \gamma
        > 0,
    \end{equation}
    consider the distribution $\pi = (1/K) \sum_{k=1}^{K} \sN(\vmu_\star^{k},
        \sigma_{\star}^2 \vI)$.
    For each $j \in [N]$, let $\vx^j \sim \pi$. Fix $i \in [N]$ and $0
        \leq t \leq 1$,
    and let $\vx \sim \alpha_t \vx^i + \sigma_t \vg$, where $\vg \sim \sN(\Zero,
        \vI)$ is independent from $(x^j)_{j=1}^N$.
    Then for any $0 \leq \veps \leq 1$ satisfying the coupling conditions
    \begin{equation}
        \veps \leq
        \frac{
            \alpha_t \gamma
        }{
            2\sqrt{d(2\alpha_t^2 \sigma_\star^2 + \sigma_t^2)}
        }
    \end{equation}
    and
    \begin{equation}
        \frac{\veps^2}{1-\veps} \leq \frac{\alpha_t^2
            \sigma_\star^2}{2\sigma_t^2},
    \end{equation}
    one has with probability at least $1 - 4N \exp[-d\veps^2 / 8]$
    \begin{equation}
        \min_{j \notin \mathsf{S}_i}\,
        \norm*{
            \alpha_t \vx^j - \vx
        }^2
        -
        \norm*{
            \alpha_t \vx^i - \vx
        }^2
        \geq \frac{\alpha_t^2 \gamma^2}{2} ,
    \end{equation}
    where $\mathsf{S}_i$, where $j \in \mathsf{S}_i$ if $k_j = k_i$, where $k_j$ is the (unique) index of the mean in $(\vmu_\star^k)_{k=1}^K$ associated with $\vx^j$.
\end{lemma}

The proof of this lemma is similar to the one of \Cref{lemma:softmax-weights-samples-gap-control}.

\begin{proof}
    We start by writing $\vx^j \equid \vmu_{\star}^{k_j}
        + \sigma_{\star} \vw^j$ for each $j \in [N]$, where
    $k_j \in [K]$ is unique and $\vw^j \sim \sN(\Zero, \vI)$ is independent.
    We have for any $j \in [N]$
    \begin{equation}
        \alpha_t \vx^j - \vx
        \equid
        \alpha_t \sigma_\star(\vw^j - \vw^i) + \alpha_t (\vmu_\star^{k_j}
        - \vmu_\star^{k_i}) - \sigma_t \vg,
    \end{equation}
    so by rotational invariance
    \begin{align}
        \norm*{
            \alpha_t \vx^j - \vx
        }^2
         & \equid
        \alpha_t^2 \norm*{\vmu_\star^{k_j} - \vmu_\star^{k_i}}^2
        + \norm*{\alpha_t \sigma_\star(\vw^j - \vw^i) - \sigma_t \vg}^2
        \\
         & \quad+ 2 \alpha_t \norm*{\vmu_\star^{k_j} - \vmu_{\star}^{k_i}}
        (\alpha_t \sigma_\star(w^j_1 - w^i_1) - \sigma_t g_1).
        \label{eq:main_inequality_different_x_2}
    \end{align}
    In what follows, we consider the case $j \neq i$. First, suppose that $k_j \neq k_i$.
    The random variable $\alpha_t \sigma_\star(w^j_1 - w^i_1) - \sigma_t g_1$ is
    equal in distribution to a Gaussian random variable with mean zero and
    variance $2\alpha_t^2 \sigma_\star^2 + \sigma_t^2$. Call this random
    variable $X$; by Gaussian concentration, for any $t \geq 0$, we have
    \begin{equation}
        \Pr*{\abs{X} \geq t} \leq 2 \exp[-t^2 / 2(2\alpha_t^2 \sigma_\star^2
            + \sigma_t^2)],
    \end{equation}
    so in particular
    \begin{equation}
        \Pr*{\abs{X} \geq \frac{\veps \sqrt{d(2\alpha_t^2 \sigma_\star^2
                    + \sigma_t^2)}}{2}} \leq 2 \exp[-d \veps^2 / 8 ].
    \end{equation}
    Combining this result with \eqref{eq:main_inequality_different_x_2}, we get that with probability at least $1 - 2 \abs{\set{j \in [N] \given k_j \neq
                k_i}}\exp[-d\veps^2/8]$, we have for every $j \in [N]$ for which $k_j \neq k_i$
    that
    \begin{align}
         & \norm*{
            \alpha_t \vx^j - \vx
        }^2         \\
         & \qquad\geq
        \alpha_t^2 \norm*{\vmu_\star^{k_j} - \vmu_\star^{k_i}}^2
        + \norm*{\alpha_t \sigma_\star(\vw^j - \vw^i) - \sigma_t \vg}^2
        -  \veps \alpha_t \norm*{\vmu_\star^{k_j} - \vmu_{\star}^{k_i}}
        \sqrt{d(2\alpha_t^2 \sigma_\star^2 + \sigma_t^2)}.
    \end{align}
    Given that
    \begin{equation}
        \min_{k \neq k'}\, \norm*{\vmu_{\star}^k - \vmu_{\star}^{k'}} \geq
        \gamma,
    \end{equation}
    if $\alpha_t \gamma \geq 2\veps \sqrt{d(2\alpha_t^2 \sigma_{\star}^2
            + \sigma_t^2)}$,
    we have on the previous event
    \begin{align}
        \norm*{
            \alpha_t \vx^j - \vx
        }^2
         & \geq
        \frac{\alpha_t^2}{2} \norm*{\vmu_\star^{k_j} - \vmu_\star^{k_i}}^2
        + \norm*{\alpha_t \sigma_\star(\vw^j - \vw^i) - \sigma_t \vg}^2
        \\
         & \geq
        \norm*{\alpha_t \sigma_\star(\vw^j - \vw^i) - \sigma_t \vg}^2 + \frac{\alpha_t^2 \gamma^2}{2} .
    \end{align}
    Therefore, we have that with probability at least $1 - 2N \exp[-d\vareps^2/8]$ for any $j \notin \mathsf{S}_i$
    \begin{align}
        \norm*{
            \alpha_t \vx^j - \vx
        }^2\geq
        \norm*{\alpha_t \sigma_\star(\vw^j - \vw^i) - \sigma_t \vg}^2 + \frac{\alpha_t^2 \gamma^2}{2} .
        \label{eq:lower_bound_first_first_new}
    \end{align}

    \paragraph{Lower bound on the difference (first stochasticity level).} We are going to give a lower bound for $j \neq i$ on the quantity
    \begin{equation}
        \norm*{\alpha_t \sigma_\star(\vw^j - \vw^i) - \sigma_t \vg}^2
        -
        \norm*{
            \alpha_t \vx^i - \vx
        }^2 .
    \end{equation}
    We have for $j \neq i$
    \begin{align}
        \norm*{\alpha_t \sigma_\star(\vw^j - \vw^i) - \sigma_t \vg}^2
        -
        \norm*{
            \alpha_t \vx^i - \vx
        }^2
         & \equid
        \norm*{\alpha_t \sigma_\star(\vw^j - \vw^i) - \sigma_t \vg}^2
        -
        \norm*{\sigma_t \vg}^2
        \\
         & =
        \alpha_t^2 \sigma_\star^2\norm*{\vw^j - \vw^i}^2
        - 2\sigma_t\alpha_t\sigma_\star\ip*{\vw^j - \vw^i}{ \vg}
        \\
         & \equid
        \alpha_t^2 \sigma_\star^2\norm*{\vw^j - \vw^i}^2
        - 2\sigma_t\alpha_t\sigma_\star g_1\norm*{\vw^j - \vw^i}
    \end{align}
    where the last line uses rotational invariance of the Gaussian distribution.
    Once again using Gaussian concentration, we have that with probability at
    least $1 - 2 \exp[-d\veps^2 / 8]$ that $\abs{g_1} \leq \tfrac{\veps
            \sqrt{d}}{2}$.
    It follows from a union bound that, on an event of probability at least
    $1 - 2 N  \exp[-d\veps^2 / 8]$,
    it holds
    \begin{equation}
        \norm*{\alpha_t \sigma_\star(\vw^j - \vw^i) - \sigma_t \vg}^2
        -
        \norm*{
            \alpha_t \vx^i - \vx
        }^2
        \geq
        \alpha_t^2 \sigma_\star^2\norm*{\vw^j - \vw^i}^2
        - \veps \sqrt{d}\sigma_t\alpha_t\sigma_\star \norm*{\vw^j - \vw^i}.
    \end{equation}
    \paragraph{Lower bound on the difference (second stochasticity level).} Now, as before, if it holds for all such $j$
    \begin{equation}
        \norm{\vw^j - \vw^i} \geq
        \frac{2\veps \sigma_t\sqrt{d} }{\alpha_t \sigma_\star},
    \end{equation}
    then the preceding bound can be simplified to
    \begin{equation}
        \norm*{\alpha_t \sigma_\star(\vw^j - \vw^i) - \sigma_t \vg}^2
        -
        \norm*{
            \alpha_t \vx^i - \vx
        }^2
        \geq
        \frac{\alpha_t^2 \sigma_\star^2}{2}\norm*{\vw^j - \vw^i}^2.
    \end{equation}
    This leads us to consider the lower tail of the random variable $\min_{j
            \neq i} \norm{\vw^j - \vw^i}$,
    which was studied in \Cref{sec:gmm-denoiser-calcs-cont}.
    A coarser approach will be sufficient for our purposes: when $j \neq i$, the
    random variable $\half\norm{\vw^j - \vw^i}^2$ is distributed as a $\chi_2(d)$
    random variable, so \Cref{thm:chernoff_bound}
    implies that for any $0 \leq \veps' \leq 1$,
    \begin{equation}
        \Pr*{\norm*{\vw^j - \vw^i}^2\geq 2(1-\veps')d}
        \geq 1- \exp\left[ \frac{-d(\veps')^2}{8}\right].
    \end{equation}
    We can for simplicity simply enforce $\veps' = \veps$. In this case, if we
    add the additional condition
    \begin{equation}
        \frac{\veps^2}{1-\veps} \leq \frac{\alpha_t^2
            \sigma_\star^2}{2\sigma_t^2},
    \end{equation}
    then by a union bound, it holds with probability at least $1
        - N \exp[-d\veps^2/8]$
    that for all such $j$,
    \begin{equation}
        \norm*{\alpha_t \sigma_\star(\vw^j - \vw^i) - \sigma_t \vg}^2
        -
        \norm*{
            \alpha_t \vx^i - \vx
        }^2
        \geq \frac{\alpha_t^2 \sigma_\star^2}{2}\norm*{\vw^j - \vw^i}^2 \geq
        \alpha_t^2 \sigma_\star^2(1-\veps) d, \label{eq:lower_bound_new}
    \end{equation}
    Finally combining \eqref{eq:lower_bound_new}, \eqref{eq:lower_bound_first_first_new} and a union bound, we get that with probability at least $1 - 4N \exp[-d\veps^2 / 8]$, we have
    \begin{equation}
        \min_{j \notin \mathsf{S}_i}\,
        \norm*{
            \alpha_t \vx^j - \vx
        }^2
        -
        \norm*{
            \alpha_t \vx^i - \vx
        }^2
        \geq
        \alpha_t^2 \sigma_\star^2(1-\veps) d + \frac{\alpha_t^2 \gamma^2}{2},
    \end{equation}
    which concludes the proof.
\end{proof}

\subsection{Denoiser Approximations}
\label{sec:denoiser_approximation}

Finally, we prove \Cref{lemma:true-denoiser-1sparse-lp} and \Cref{lemma:mem-denoiser-1sparse-lp}. Those results control the approximation of the true denoiser in \Cref{lemma:true-denoiser-1sparse-lp} and the memorizing denoiser in \Cref{lemma:mem-denoiser-1sparse-lp}. Those approximations express that under certain conditions the true denoiser can be replaced by a Gaussian denoiser and that under certain conditions the memorizing denoiser can be replaced by a point in the dataset.

\begin{lemma}{}{true-denoiser-1sparse-lp}
    Given vectors $(\vmu_{\star}^k)_{k=1}^K$ satisfying
    \begin{equation}
        \min_{k \neq k'}\, \norm*{\vmu_{\star}^k - \vmu_{\star}^{k'}} \geq \gamma
        > 0,
    \end{equation}
    consider the distribution $\pi = (1/K) \sum_{k=1}^{K} \sN(\vmu_\star^{k},
        \sigma_{\star}^2 \vI)$.
    For $i \in [N]$, let $\vx^i \sim \pi$, fix $0 \leq t \leq 1$, and let
    $\vx_t^i \sim \alpha_t
        \vx^i + \sigma_t \vg$, where $\vg \sim \sN(\Zero, \vI)$ is independent from $\vx^i$.
    Let $k_i$ denote the index of the (uniquely defined) cluster centroid
    $\vmu_\star^k$ associated to $\vx^i$.
    Define the nearest-neighbor (one sparse) denoiser $\denoiser^{\nom}(t,
        \vx_t^i)$ associated to $\denoisergen(t, \vx_t^i)$ (recall
    \Cref{lemma:denoiser_mog}) by
    \begin{equation}
        \denoiser^{\nom}(t, \vx_t^i)
        = \frac{\alpha_t \sigma_{\star}^2}{\alpha_t^2
            \sigma_{\star}^2 + \sigma_t^2} \vx_t^i + \frac{\sigma_t^2}{\alpha_t^2
            \sigma_{\star}^2 + \sigma_t^2} \vmu_\star^{k_i}.
    \end{equation}
    Then for any $0 \leq \veps \leq 1$ satisfying the coupling condition
    \begin{equation}
        \veps \leq
        \frac{
            \alpha_t \gamma
        }{
            4\sqrt{d(\alpha_t^2 \sigma_\star^2 + \sigma_t^2)}
        },
    \end{equation}
    one has with probability at least $1 - 4K \exp[-d\veps^2 / 8]$
    \begin{align}
        \norm*{
        \denoisergen(t, \vx_t^i) - \denoiser^{\nom}(t, \vx_t^i)
        }
        \leq
        \frac{2K \sigma_t^2\max_{k \in [K]}\, \norm{\vmu_\star^k}
        }{\alpha_t^2 \sigma_{\star}^2 + \sigma_t^2}
        \exp\left[ -
            \frac{\gamma^2\alpha_t^2}{4(\alpha_t^2 \sigma_{\star}^2 + \sigma_t^2)}
            \right],
    \end{align}

\end{lemma}
\begin{proof}
    We apply \Cref{lemma:softmax-1sparse-lp,lemma:softmax-weights-mean-gap-control}.
    Our hypotheses let us apply \Cref{lemma:softmax-weights-mean-gap-control},
    which gives that with probability at least
    $1 - 4K \exp[-d\veps^2 / 8]$
    \begin{equation}
        \min_{k \neq k_i}\,
        \norm*{\alpha_t \vmu_\star^{k} - \vx_t^i}^2
        -
        \norm*{\alpha_t \vmu_\star^{k_i} - \vx_t^i}^2
        \geq
        \frac{\gamma^2\alpha_t^2}{2}.
    \end{equation}
    Recall the expression for the denoiser (\Cref{lemma:denoiser_mog},
    \eqref{eq:softmax-vector-model}).
    The weight vector
    \begin{equation}
        w_k(\vx_t^i) = -\frac{1}{2}\| \alpha_t \vmu_\star^{k} - \vx_t^i \|^2/
        (\alpha_t^2 \sigma_{\star}^2 + \sigma_t^2)
    \end{equation}
    is related to the gap condition we have asserted: in particular
    with probability at least $1 - 4K \exp[-d\veps^2 / 8]$
    \begin{equation}
        \min_{k \neq k_i}\,
        w_{k_i}(\vx_t^i)
        -
        w_k(\vx_t^i)
        \geq
        \frac{\gamma^2\alpha_t^2}{4(\alpha_t^2 \sigma_{\star}^2 + \sigma_t^2)}.
    \end{equation}
    Hence, an application of \Cref{lemma:softmax-1sparse-lp} gives that for any
    $p \geq 1$,
    \begin{equation}
        \norm*{\softmax(\vw(\vx_t^i)) - \ve_{k_i}}_p \leq 2(K-1) \exp\left[ -
            \frac{\gamma^2\alpha_t^2}{4(\alpha_t^2 \sigma_{\star}^2 + \sigma_t^2)}
            \right].
    \end{equation}
    Using this result, we get
    \begin{align}
        \norm*{
        \denoisergen(t, \vx_t^i) - \denoiser^{\nom}(t, \vx_t^i)
        }
         & =
        \frac{\sigma_t^2}{\alpha_t^2 \sigma_{\star}^2 + \sigma_t^2}
        \norm*{
            \vM (\ve_{k_i} - \vw(\vx_t^i))
        }
        \\
         & \leq
        \frac{\sigma_t^2}{\alpha_t^2 \sigma_{\star}^2 + \sigma_t^2}
        \Gamma
        \norm*{
            \ve_{k_i} - \vw(\vx_t^i)
        }_1
        \\
         & \leq
        \frac{2\Gamma(K-1) \sigma_t^2}{\alpha_t^2 \sigma_{\star}^2 + \sigma_t^2}
        \exp\left[ -
            \frac{\gamma^2\alpha_t^2}{4(\alpha_t^2 \sigma_{\star}^2 + \sigma_t^2)}
            \right],
    \end{align}
    where we have defined
    \begin{equation}
        \vM =
        \begin{bmatrix}
            \vmu_\star^1 & \hdots & \vmu_\star^K
        \end{bmatrix}
        \in \bbR^{d \times K},
        \quad
        \Gamma =
        \max_{k \in [K]}\, \norm{\vmu_\star^k},
    \end{equation}
    which concludes the proof.

\end{proof}

\begin{lemma}{}{mem-denoiser-1sparse-lp}
    Given vectors $(\vmu_{\star}^k)_{k=1}^K$ satisfying
    \begin{equation}
        \min_{k \neq k'}\, \norm*{\vmu_{\star}^k - \vmu_{\star}^{k'}} \geq \gamma
        > 0,
    \end{equation}
    consider the distribution $\pi = (1/K) \sum_{k=1}^{K} \sN(\vmu_\star^{k},
        \sigma_{\star}^2 \vI)$.
    For $i \in [N]$, let $\vx^i \sim \pi$, fix $0 \leq t \leq 1$, and let
    $\vx_t^i \sim \alpha_t
        \vx^i + \sigma_t \vg$, where $\vg \sim \sN(\Zero, \vI)$ is independent from $\vx^i$.
    Then for any $0 \leq \veps \leq 1$ satisfying the coupling conditions
    \begin{equation}
        \veps \leq
        \frac{
            \alpha_t \gamma
        }{
            2\sqrt{d(2\alpha_t^2 \sigma_\star^2 + \sigma_t^2)}
        } ,
    \end{equation}
    and
    \begin{equation}
        \frac{\veps^2}{1-\veps} \leq \frac{\alpha_t^2
            \sigma_\star^2}{2\sigma_t^2},
    \end{equation}
    one has with probability at least $1 - 8N\exp[-d\veps^2 / 8]$
    \begin{equation}
        \norm*{
            \denoisermem(t, \vx_t^i) - \vx^i
        }
        \leq
        2N \Gamma_\star
        \exp\left[- \frac{\alpha_t^2 \sigma_\star^2(1-\veps) d}{2\sigma_t^2} \right],
    \end{equation}
    with
    \begin{equation}
        \Gamma_\star =     \max_{k\in [K]}\,
        \sqrt{\norm{\vmu_\star^{k}}^2 + \veps \sigma_\star\sqrt{d}
            \norm{\vmu_\star^k} + (1+\veps) \sigma_\star^2 d} .
    \end{equation}
\end{lemma}
\begin{proof}
    The proof is similar to that of \Cref{lemma:true-denoiser-1sparse-lp}:
    we apply
    \Cref{lemma:softmax-1sparse-lp,lemma:softmax-weights-samples-gap-control}.
    Our hypotheses let us apply \Cref{lemma:softmax-weights-samples-gap-control},
    which gives that with probability at least
    $1 - 4N \exp[-d\veps^2 / 8]$
    \begin{equation}
        \min_{j \neq i}\,
        \norm*{
            \alpha_t \vx^j - \vx_t^i
        }^2
        -
        \norm*{
            \alpha_t \vx^i - \vx_t^i
        }^2
        \geq
        \alpha_t^2 \sigma_\star^2(1-\veps) d.
    \end{equation}
    Recall the expression for the denoiser (\Cref{lemma:denoiser_mog},
    \eqref{eq:softmax-vector-model}). The memorizing denoiser corresponds to
    $\sigma_\star = 0$, so the weight vector
    \begin{equation}
        w_j(\vx_t^i) = -\frac{1}{2}\| \alpha_t \vx^j - \vx_t^i \|^2/ \sigma_t^2 ,
    \end{equation}
    is related to the gap condition we have asserted: in particular
    with probability at least $1 - 4N \exp[-d\veps^2 / 8]$
    \begin{equation}
        \min_{j \neq i}\,
        w_{j}(\vx_t^i)
        -
        w_i(\vx_t^i)
        \geq
        \frac{\alpha_t^2 \sigma_\star^2(1-\veps) d}{2\sigma_t^2}.
    \end{equation}
    Hence, an application of \Cref{lemma:softmax-1sparse-lp} gives that for any
    $p \geq 1$,
    \begin{equation}
        \norm*{\softmax(\vw(\vx_t^i)) - \ve_{i}}_p \leq 2(N-1) \exp\left[ -
            \frac{\alpha_t^2 \sigma_\star^2(1-\veps) d}{2\sigma_t^2}
            \right].
    \end{equation}
    Using this result, we get
    \begin{align}
        \norm*{
        \denoisermem(t, \vx_t^i) - \vx^i
        }
         & =
        \norm*{
            \vM (\ve_{i} - \vw(\vx_t^i))
        }
        \\
         & \leq \Gamma
        \norm*{
            \ve_{k_i} - \vw(\vx_t^i)
        }_1
        \\
         & \leq
        2\Gamma(N-1)
        \exp\left[ -
            \frac{\alpha_t^2 \sigma_\star^2(1-\veps) d}{2\sigma_t^2}
            \right],
    \end{align}
    where we have defined
    \begin{equation}
        \vM =
        \begin{bmatrix}
            \vx^1 & \hdots & \vx^N
        \end{bmatrix}
        \in \bbR^{d \times N},
        \quad
        \Gamma =
        \max_{j \in [N]}\, \norm{\vx^j}.
    \end{equation}
    Finally,  using \Cref{lemma:gaussian-norm-nonzero-mean}, we have
    that for any $0 \leq \veps \leq 1$
    \begin{equation}
        \Pr*{
            \abs*{\norm{\vx^j}^2 - (\norm{\vmu_\star^{k_j}}^2 + \sigma_\star^2 d)}
            \geq \veps \sigma_\star\sqrt{d}(\sigma_\star \sqrt{d}
            + \norm{\vmu_\star^{k_j}})
        }
        \leq 3 \exp[-d \veps^2 / 8 ],
    \end{equation}
    so with probability at least $1 - 4N \exp[-d\veps^2 / 8]$ we have
    \begin{equation}
        \Gamma
        \leq
        \max_{k\in [K]}\,
        \sqrt{\norm{\vmu_\star^{k}}^2 + \veps \sigma_\star\sqrt{d}
            \norm{\vmu_\star^k} + (1+\veps) \sigma_\star^2 d}.
    \end{equation}
    By a union bound, we conclude that with probability at least $1 - 8N
        \exp[-d\veps^2 / 8]$ that
    \begin{equation}
        \norm*{
            \denoisermem(t, \vx_t^i) - \vx^i
        }
        \leq
        {2\Gamma(N-1)}
        \exp\left[ -
            \frac{\alpha_t^2 \sigma_\star^2(1-\veps) d}{2\sigma_t^2}
            \right] ,
    \end{equation}
    for each $i$, which concludes the proof.

\end{proof}

Finally, we derive similar results for the partially memorizing denoiser.
We recall that $\denoiserpmem(t, \vx_t^i)$ is given by
\begin{equation}
    \denoiserpmem(t, \vx_t^i) = \sum_{j \in [\ell]} \softmax(\vw)_j \vx^j ,
\end{equation}
where $\vw_j = -\frac{1}{2\sigma_t^2}\| \alpha_t \vx^j - \vx_t^i \|^2$.
\begin{lemma}{}{p-mem-denoiser-1sparse-lp}
    Given vectors $(\vmu_{\star}^k)_{k=1}^K$ satisfying
    \begin{equation}
        \min_{k \neq k'}\, \norm*{\vmu_{\star}^k - \vmu_{\star}^{k'}} \geq \gamma
        > 0,
    \end{equation}
    consider the distribution $\pi = (1/K) \sum_{k=1}^{K} \sN(\vmu_\star^{k},
        \sigma_{\star}^2 \vI)$.
    For $i \in [N]$, let $\vx^i \sim \pi$, fix $0 \leq t \leq 1$, and let
    $\vx_t^i \sim \alpha_t
        \vx^i + \sigma_t \vg$, where $\vg \sim \sN(\Zero, \vI)$ is independent from $\vx^i$.
    Then for any $0 \leq \veps \leq 1$ satisfying the coupling conditions
    \begin{equation}
        \veps \leq
        \frac{
            \alpha_t \gamma
        }{
            2\sqrt{d(2\alpha_t^2 \sigma_\star^2 + \sigma_t^2)}
        } ,
    \end{equation}
    and
    \begin{equation}
        \frac{\veps^2}{1-\veps} \leq \frac{\alpha_t^2
            \sigma_\star^2}{2\sigma_t^2}.
    \end{equation}
    We consider two cases, first assume that $i \in [\ell]$ then one  has with probability at least $1 - 8N\exp[-d\veps^2 / 8]$
    \begin{equation}
        \norm*{
            \denoiserpmem(t, \vx_t^i) - \vx^i
        }
        \leq
        2N \Gamma_\star
        \exp\left[- \frac{\alpha_t^2 \sigma_\star^2(1-\veps) d}{2\sigma_t^2} \right],
    \end{equation}
    with
    \begin{equation}
        \Gamma_\star =     \max_{k\in [K]}\,
        \sqrt{\norm{\vmu_\star^{k}}^2 + \veps \sigma_\star\sqrt{d}
            \norm{\vmu_\star^k} + (1+\veps) \sigma_\star^2 d} .
    \end{equation}
    Second, assume that $i \notin [\ell]$, with probability $1 - 8N\exp[-d\veps^2 / 8] - K^{-\log(d)}(1 + \log(K) \log(d))$, $\mathsf{S}_i = \{ j \in \ell, \ k_j = k_i \}$ is not empty and
    \begin{equation}
        \norm*{
            \denoiserpmem(t, \vx_t^i) - \sum_{j \in \mathsf{S}_i} \softmax(\vw|_{\mathsf{S}_i})_j \vx^j
        }
        \leq \Gamma_\star (1 + N)^{2} \exp\left[ -\frac{\alpha_t^2 \gamma^2}{2\sigma_t^2} \right] .
    \end{equation}
\end{lemma}
\begin{proof}
    The first part of the proof where $i \in [\ell]$ is identical to \Cref{lemma:mem-denoiser-1sparse-lp}. We now assume that $i \notin [\ell]$. First, using \Cref{prop:coupon_collector} we have that on an event with probability $1 - K^{-\log(d)}(1 + \log(K) \log(d))$, $\mathsf{S}_i$ is not empty. In addition, using a union bound we have that with probability at least $1 - 4N\exp[-d\veps^2 / 8] - K^{-\log(d)}(1 + \log(K) \log(d))$, $\mathsf{S}_i = \{ j \in \ell, \ k_j = k_i \}$ is not empty and
    \begin{equation}
        \min_{j \notin \mathsf{S}_i}\,
        \norm*{
            \alpha_t \vx^j - \vx
        }^2
        -
        \norm*{
            \alpha_t \vx^i - \vx
        }^2
        \geq \frac{\alpha_t^2 \gamma^2}{2} ,
    \end{equation}
    Therefore, using \Cref{lemma:softmax-extension}, we get that
    \begin{align}
        \norm*{\softmax(\vw) - \softmax(\vw_{|\mathsf{S}_i})}_p & \leq (1 + | \mathsf{S}_i |)^{1/p} (|\mathsf{S}| - | \mathsf{S}_i |) \exp\left[-\frac{\alpha_t^2 \gamma^2}{2 \sigma_t^2}\right] \\
                                                                & \leq (1 + N)^{1/p} N \exp\left[-\frac{\alpha_t^2 \gamma^2}{2 \sigma_t^2}\right] .
    \end{align}
    Using this result, we get
    \begin{align}
        \norm*{
        \denoiserpmem(t, \vx_t^i) - \sum_{j \in \mathsf{S}_i} \softmax(\vw|_{\mathsf{S}_i})_j \vx^j
        }
         & =
        \norm*{
            \vM (\softmax(\vw|_{\mathsf{S}_i}) - \softmax(\vw))
        }
        \\
         & \leq \Gamma
        \norm*{
            \softmax(\vw|_{\mathsf{S}_i}) - \softmax(\vw)
        }_1
        \\
         & \leq
        (N+1)^{2} \Gamma
        \exp\left[ -
            \frac{\alpha_t^2 \gamma^2}{2 \sigma_t^2}
            \right],
    \end{align}
    where we have defined
    \begin{equation}
        \vM =
        \begin{bmatrix}
            \vx^1 & \hdots & \vx^\ell
        \end{bmatrix}
        \in \bbR^{d \times \ell},
        \quad
        \Gamma =
        \max_{j \in [N]}\, \norm{\vx^j}.
    \end{equation}
    Finally,  using \Cref{lemma:gaussian-norm-nonzero-mean}, we have
    that for any $0 \leq \veps \leq 1$
    \begin{equation}
        \Pr*{
            \abs*{\norm{\vx^j}^2 - (\norm{\vmu_\star^{k_j}}^2 + \sigma_\star^2 d)}
            \geq \veps \sigma_\star\sqrt{d}(\sigma_\star \sqrt{d}
            + \norm{\vmu_\star^{k_j}})
        }
        \leq 3 \exp[-d \veps^2 / 8 ],
    \end{equation}
    so with probability at least $1 - 4N \exp[-d\veps^2 / 8]$ we have
    \begin{equation}
        \Gamma
        \leq
        \Gamma_\star = \max_{k\in [K]}\,
        \sqrt{\norm{\vmu_\star^{k}}^2 + \veps \sigma_\star\sqrt{d}
            \norm{\vmu_\star^k} + (1+\veps) \sigma_\star^2 d}.
    \end{equation}
    By a union bound, we conclude that with probability at least $1 - K^{-\log(d)}(1 + \log(K) \log(d)) - 8N
        \exp[-d\veps^2 / 8]$ that
    \begin{equation}
        \norm*{
            \denoiserpmem(t, \vx_t^i) - \sum_{j \in \mathsf{S}_i} \softmax(\vw|_{\mathsf{S}_i})_j \vx^j
        }
        \leq  \Gamma_\star (1 + N)^{2} \exp\left[ -\frac{\alpha_t^2 \gamma^2}{2\sigma_t^2} \right],
    \end{equation}
    which concludes the proof.

\end{proof}

\section{Training Loss Approximations}
\label{sec:training_loss_approx}

In this section, we give proofs of results that imply our main results stated in the main body, namely \Cref{thm:asymptotics-generalizing} and \Cref{thm:simplification_p_mem_scaling}.
In the proofs, we will use a convenient shorthand for the training loss defined in  \eqref{eq:training_loss_defns} when specialized to the Gaussian mixture model denoisers of interest. Namely, since in this case the loss is effectively a function of the learnable mean parameters and the shared learnable variance parameter in the model, we will write 
\begin{equation}
\E*{\sL_N(\vmu^1_\star, \dots, \vmu_\star^K, \sigma_{\star}^2)}
\end{equation}
to denote the loss of the generalizing denoiser  $\sL_{N, t}(\denoisergen)$,
\begin{equation}
\E*{\sL_N(\vx^1, \dots, \vx^N, 0)}
\end{equation}
to denote the loss of the memorizing denoiser $\sL_{N, t}(\denoisermem)$, and
\begin{equation}
\E*{\sL_N(\vx^1, \dots, \vx^\ell, 0)}
\end{equation}
to denote the loss of the partially memorizing denoiser $\sL_{N, t}(\denoiserpmem)$.

\begin{theorem}{}{generalizing-denoiser-loss-approximation_appendix}
    Given vectors $(\vmu_{\star}^k)_{k=1}^K$ satisfying
    \begin{equation}
        \min_{k \neq k'}\, \norm*{\vmu_{\star}^k - \vmu_{\star}^{k'}} \geq \gamma
        > 0,
    \end{equation}
    consider the distribution $\pi = (1/K) \sum_{k=1}^{K} \sN(\vmu_\star^{k},
        \sigma_{\star}^2 \vI)$.
    For $i \in [N]$, let $\vx^i \sim \pi$, fix $0 \leq t \leq 1$, and let
    $\vx_t^i \sim \alpha_t
        \vx^i + \sigma_t \vg$, where $\vg \sim \sN(\Zero, \vI)$ is independent from $\vx^i$.
    Consider the denoiser $\denoisergen(t, \vx_t)$ associated to the true
    distribution $\pi_\star$.
    Then for any $0 \leq \veps \leq 1$ satisfying the coupling conditions
    \begin{equation}
        \veps \leq
        \frac{
            \alpha_t \gamma
        }{
            2\sqrt{d(\alpha_t^2 \sigma_\star^2 + \sigma_t^2)}
        },
    \end{equation}
    and
    \begin{equation}
        \frac{\veps^2}{1-\veps} \leq \frac{\alpha_t^2
            \sigma_\star^2}{2\sigma_t^2},
    \end{equation}

    \begin{equation}
        \abs*{
            \left[
                \E*{\sL_N(\vmu^1_\star, \dots, \vmu_\star^K, \sigma_{\star}^2)}
                -
                \E*{\sL_N(\vx^1, \dots, \vx^N, 0)}
                \right]
            - \frac{d\sigma_\star^2 \sigma_t^2}{\alpha_t^2 \sigma_\star^2
                + \sigma_t^2}
        }
        \leq \Xi(\vareps, t,d),
    \end{equation}
    where the residual $\Xi$ is explicit in the proof.
\end{theorem}

We denote $\snr_t = \alpha_t^2 / \sigma_t^2$ and $\snrinv_t$ its inverse function.

\begin{proof}
    The proof will be an application of \Cref{lemma:true-denoiser-1sparse-lp} and
    \Cref{lemma:mem-denoiser-1sparse-lp}.
    \paragraph{Nominal value decomposition.} First, we define for any $i \in [N]$
    \begin{equation}
        \denoiser^{\nom}(t, \vx_t^i)
        = \frac{\alpha_t \sigma_{\star}^2}{\alpha_t^2
            \sigma_{\star}^2 + \sigma_t^2} \vx_t^i + \frac{\sigma_t^2}{\alpha_t^2
            \sigma_{\star}^2 + \sigma_t^2} \vmu_\star^{k_i}.
    \end{equation}
    This nominal value $\denoiser^{\nom}(t, \vx_t^i)$ will be crucial for the rest of our analysis.
    Using \eqref{eq:training-loss-ortho}, we have that
    \begin{align}
         & \left|\E*{\sL_N(\vmu_\star^1, \dots, \vmu_\star^K, \sigma_\star^2)}
        -
        \E*{\sL_N(\vx^1, \dots, \vx^N, 0)}  - \frac{d\sigma_\star^2 \sigma_t^2}{\alpha_t^2 \sigma_\star^2
        + \sigma_t^2}\right|                                                   \\
         & \quad =
        \frac{1}{N}
        \sum_{i=1}^N \left|
        \E*{
        \norm*{
        \denoisergen(t, \vx_t^i) - \denoisermem(t, \vx_t^i)
        }^2
        } - \frac{d\sigma_\star^2 \sigma_t^2}{\alpha_t^2 \sigma_\star^2
        + \sigma_t^2} \right|                                                  \\
         & \quad =
        \frac{1}{N}
        \sum_{i=1}^N \left|
        \E*{
        \norm*{
        \denoisergen(t, \vx_t^i) - \denoiser^\nom(t, \vx_t^i) + \denoiser^\nom(t,
        \vx_t^i) - \vx_i + \vx_i - \denoisermem(t, \vx_t^i)
        }^2
        } - \frac{d\sigma_\star^2 \sigma_t^2}{\alpha_t^2 \sigma_\star^2
            + \sigma_t^2} \right| .
    \end{align}
    In what follows, for simplicity, we denote $\Lambda_t^i$
    \begin{equation}
        \Lambda_t^i = \denoisergen(t, \vx_t^i) - \denoiser^\nom(t, \vx_t^i) + \vx_i
        - \denoisermem(t, \vx_t^i) .
    \end{equation}
    We have that
    \begin{align}
        \label{eq:key_inequality}
         & \left|\E*{\sL_N(\vmu_\star^1, \dots, \vmu_\star^K, \sigma_\star^2)}
        -
        \E*{\sL_N(\vx^1, \dots, \vx^N, 0)}  - \frac{d\sigma_\star^2 \sigma_t^2}{\alpha_t^2 \sigma_\star^2
        + \sigma_t^2}\right|                                                                                                                                                                                    \\
         & \quad =
        \frac{1}{N}
        \sum_{i=1}^N \left|
        \E*{
            \norm*{
                \Lambda_t^i +  \denoiser^\nom(t, \vx_t^i) - \vx_i
            }^2
        } - \frac{d\sigma_\star^2 \sigma_t^2}{\alpha_t^2 \sigma_\star^2
        + \sigma_t^2} \right|                                                                                                                                                                                   \\
         & \quad \leq
        \frac{1}{N}
        \sum_{i=1}^N \left|
        \E*{
            \norm*{
                \denoiser^\nom(t, \vx_t^i) - \vx_i
            }^2
        } - \frac{d\sigma_\star^2 \sigma_t^2}{\alpha_t^2 \sigma_\star^2
        + \sigma_t^2} \right|                                                                                                                                                                                   \\
         & \qquad \quad + \frac{1}{N}
        \sum_{i=1}^N \E*{\norm*{\Lambda_t^i}^2} + \frac{2}{N} \sum_{i=1}^N
        \E*{\norm*{\Lambda_t^i}^2}^{1/2} \E*{\norm*{\denoiser^\nom(t, \vx_t^i) - \vx_i }^2}^{1/2}                                          \\
         & \quad \leq
        \frac{1}{N}
        \sum_{i=1}^N \left|
        \E*{
            \norm*{
                \denoiser^\nom(t, \vx_t^i) - \vx_i
            }^2
        } - \frac{d\sigma_\star^2 \sigma_t^2}{\alpha_t^2 \sigma_\star^2
        + \sigma_t^2} \right|                                                                                                                                                                                   \\
         & \qquad \quad + \frac{2}{N}
        \sum_{i=1}^N \left\lbrace \E*{\norm*{\denoisergen(t, \vx_t^i)
        - \denoiser^\nom(t, \vx_t^i)}^2} + \E*{\norm*{\vx_i - \denoisermem(t, \vx_t^i)}^2} \right\rbrace +                                 \\
         & \qquad \quad \frac{4}{N} \sum_{i=1}^N \left\lbrace
         \E*{\norm*{\denoisergen(t, \vx_t^i) - \denoiser^\nom(t, \vx_t^i)}^2}
         + \E*{\norm*{\vx_i - \denoisermem(t, \vx_t^i)}^2} \right\rbrace^{1/2} \\
         & \qquad \qquad \qquad \times \E*{\norm*{\denoiser^\nom(t, \vx_t^i) - \vx_i }^2}^{1/2} .
    \end{align}

    In the rest of the proof, we control each term in \eqref{eq:key_inequality}.

    \paragraph{Control of generalizing denoiser.} First, we are going to control
    $\E*{\norm*{\denoisergen(t, \vx_t^i) - \denoiser^\nom(t, \vx_t^i)}^2}$. First, we recall that by \eqref{eq:model-class}
    \begin{equation}
        \label{eq:model-class-recall}
        \denoisergen(t,\vx_t^i) = \frac{\alpha_t
            \sigma^2_\star}{\alpha_t^2 \sigma^2_\star + \sigma_t^2} \vx_t^i
        + \frac{\sigma_t^2}{\alpha_t^2\sigma^2_\star + \sigma_t^2}
        \sum_{i=1}^{M} \vmu^{i} \softmax(\vw)_i ,
    \end{equation}
    In what follows, we denote
    \begin{equation}
        \Gamma = \max_{i \in [K]} \| \mu_\star^i \| .
    \end{equation}
    In addition, we have that
    \begin{equation}
        \vx_t^i = \alpha_t \vmu_\star^{k_i} + \alpha_t \sigma_\star \vw + \sigma_t \vg \equid \alpha_t \vmu_\star^{k_i} + (\alpha_t^2 \sigma_\star^2 + \sigma_t^2)^{1/2} \vg .
    \end{equation}
    where $k_i$ is the index of the mean corresponding to $\vx^i$.
    Finally, we have that for any $p \in \nset$
    \begin{equation}
        \E*{\| \vg\|^p} \leq \E*{\| \vg\|^{2p}}^{1/2} \leq 2^{p/2} (\Gamma(d/2+p) / \Gamma(d/2))^{1/2} \leq 2^{p/2} (d/2 + p)^{p/2} .
    \end{equation}
    Combining those results, we get that for any $p \in \nset$
    \begin{equation}
        \E*{\norm*{\vx_t^i}^p} \leq 2^{2p} \left( \Gamma^p + (\alpha_t^2 \sigma_\star^2 + \sigma_t^2)^{p/2} (d/2 + p)^{p/2} \right).
    \end{equation}
    Similarly, we have that for any $p \in \nset$
    \begin{equation}
        \E*{\norm*{\denoiser^\nom(t, \vx_t^i)}^p} \leq 2^{2p} \left( \Gamma^p + (\alpha_t^2 \sigma_\star^2 + \sigma_t^2)^{p/2} (d/2 + p)^{p/2} \right).
    \end{equation}
    For any event $\mathsf{A}$ such that $\norm*{\denoisergen(t, \vx_t^i)
    - \denoiser^\nom(t, \vx_t^i)}^2 \leq C$, we have that
    \begin{align}
        \E*{\norm*{\denoisergen(t, \vx_t^i) - \denoiser^\nom(t, \vx_t^i)}^2} & \leq
        C \Pr*{\mathsf{A}} + 2 \Pr*{\mathsf{A}^{\mathrm{c}}}^{1/2} \left(
        \E*{\norm*{\vx_t^i}^4}^{1/2} + \E*{\norm*{\denoiser^\nom(t, \vx_t^i)}^4}^{1/2} \right) \\
                                                                              & \leq C  + 16 \left( \Gamma^2 + (\alpha_t^2 \sigma_\star^2 + \sigma_t^2) (d/2 + 4) \right) \Pr*{\mathsf{A}^{\mathrm{c}}}^{1/2} .
    \end{align}
    Now, combining this result and \Cref{lemma:true-denoiser-1sparse-lp} we get that
    \begin{align}
         & \E*{\norm*{\denoisergen(t, \vx_t^i) - \denoiser^\nom(t, \vx_t^i)}^2}                                                                                      \\
         & \quad \leq \left( \frac{2K \sigma_t^2\max_{k \in [K]}\, \norm{\vmu_\star^k}
        }{\alpha_t^2 \sigma_{\star}^2 + \sigma_t^2} \right)^2
        \exp\left[ -
            \frac{\gamma^2\alpha_t^2}{2(\alpha_t^2 \sigma_{\star}^2 + \sigma_t^2)}
        \right]                                                                                                                                                       \\
         & \qquad \quad + 64 \left( \max_{k \in [K]}\, \norm{\vmu_\star^k}^2 + (\alpha_t^2 \sigma_\star^2 + \sigma_t^2) (d/2 + 4) \right) K \exp[-d \vareps^2/16] .
    \end{align}
    Therefore, there exists a numerical constant $C_0 \geq 0$ such that
    \begin{align}
         & \E*{\norm*{\denoisergen(t, \vx_t^i) - \denoiser^\nom(t, \vx_t^i)}^2}                                                      \\
         & \qquad \leq C_0 \left( \max_{k \in [K]}\, \norm{\vmu_\star^k}^2 + (1 + \sigma_\star^2) d  \right) K^2 \left(\exp\left[ -
            \frac{\gamma^2\alpha_t^2}{2(\alpha_t^2 \sigma_{\star}^2 + \sigma_t^2)}
            \right] + \exp[-d \vareps^2/16] \right) .
    \end{align}

    \paragraph{Control of memorizing denoiser.} Second, we are going to control
    $\E*{\norm*{\vx_i - \denoisermem(t, \vx_t^i)}^2}$. The proof is similar
    to the control of $\E*{\norm*{\denoisergen(t, \vx_t^i) - \denoiser^\nom(t, \vx_t^i)}^2}$.  We recall that we have
    \begin{equation}\label{eq:model-class-mem}
        \denoisermem(t,\vx_t) = \sum_{i=1}^{N} \vx^{i} \softmax(\vw)_i ,
    \end{equation}
    where
    \begin{equation}\label{eq:softmax-vector-model-mem}
        w_i(\vx_t) = -\frac{1}{2}\| \alpha_t \vx^{i} - \vx_t^i \|^2 / \sigma_t^2
    \end{equation}
    Similarly as before, we have for any $p \in \nset$
    \begin{equation}
        \E*{ \max_{i \in [N]} \| \vx_i \|^p}  \leq 2^{2p}N (\Gamma^p + (d/2 + p)^{p/2}) .
    \end{equation}
    Therefore, we get that
    \begin{equation}
        \E*{\| \vx_i \|^p} \leq 2^{2p} N (\Gamma^p + \sigma_\star^{p}(d/2
        + p)^{p/2}) , \qquad \E*{\| \denoisermem(t, \vx_t^i) \|^p} \leq 2^{2p} N (\Gamma^p + \sigma_\star^{p}(d/2 + p)^{p/2}) .
    \end{equation}
    Note that this upper-bound is rather loose but given the rate of growth of
    $N$ with respect to $d$ that we will consider we won't need a tighter bound.
    For any event $\mathsf{A}$ such that $\norm*{\denoisermem(t, \vx_t^i) - \vx^i}^2 \leq C$, we have that
    \begin{align}
        \E*{\norm*{\denoisermem(t, \vx_t^i) - \vx^i}^2} & \leq
        C \Pr*{\mathsf{A}} + 2 \Pr*{\mathsf{A}^{\mathrm{c}}}^{1/2} \left(
        \E*{\norm*{\vx_t^i}^4}^{1/2} + \E*{\norm*{\denoiser^\nom(t, \vx_t^i)}^4}^{1/2} \right) \\
                                                              & \leq C  + 16 \left( \Gamma^2 + \sigma_\star^2 (d/2 + 4) \right) \Pr*{\mathsf{A}^{\mathrm{c}}}^{1/2} .
    \end{align}
    Therefore, combining this result with \Cref{lemma:mem-denoiser-1sparse-lp}, there exists a numerical constant $C_1 \geq 0$ such that
    \begin{align}
         & \E*{\norm*{\denoisermem(t, \vx_t^i) - \vx^i}^2}                                                                                                                                                                       \\
         & \qquad \leq C_1 \left( \max_{k \in [K]}\, \norm{\vmu_\star^k}^2 + (1 + \sigma_\star^2) d  \right) N^2 \left(\exp\left[-\frac{\alpha_t^2 \sigma_\star^2(1-\vareps)d}{2\sigma_t^2}\right] + \exp[-d \vareps^2/16] \right) .
    \end{align}

    \paragraph{Control of the nominal expectation.} For the nominal error
    $ \norm{ \denoiser^\nom(t, \vx_t^i) - \vx^i }$,
    we have
    by the nominal denoiser's definition in \Cref{lemma:true-denoiser-1sparse-lp}
    and the distributional assumption $\vx^i \simiid \pi$
    \begin{align}
        \norm*{
            \denoiser^\nom(t, \vx_t^i)
            - \vx^i
        }^2
         & =
        \norm*{
            \frac{\alpha_t \sigma_\star^2}{\alpha_t^2 \sigma_\star^2 +
                \sigma_t^2}\vx_t^i
            +
            \frac{\sigma_t^2}{\alpha_t^2 \sigma_\star^2 +
                \sigma_t^2}\vmu_\star^{k_i}
            - \vx^i
        }^2
        \\
         & =
        \norm*{
            \frac{\alpha_t^2 \sigma_\star^2}{\alpha_t^2 \sigma_\star^2 + \sigma_t^2}
            \sigma_\star \vw^i
            +
            \frac{\alpha_t\sigma_\star^2}{\alpha_t^2 \sigma_\star^2 +
                \sigma_t^2}\sigma_t \vg^i
            - \sigma_\star \vw^i
        }^2,
    \end{align}
    where $\vw^i \sim \sN(\Zero, \vI)$ is independent of $\vg^i \sim \sN(\Zero,
        \vI)$.
    In particular, this is equal in distribution to the squared $\ell^2$ norm of
    a Gaussian random variable, which has zero mean and isotropic diagonal
    covariance with diagonal elements
    \begin{align}
        \sigma_\star^2 \left(
        1- \frac{\alpha_t^2 \sigma_\star^2}{\alpha_t^2 \sigma_\star^2 +
            \sigma_t^2}
        \right)^2
        + \sigma_t^2 \left(
        \frac{\alpha_t\sigma_\star^2}{\alpha_t^2 \sigma_\star^2 +
            \sigma_t^2}
        \right)^2
        =
        \frac{\sigma_\star^2 \sigma_t^2}{\alpha_t^2 \sigma_\star^2 + \sigma_t^2}.
    \end{align}
    An application of \Cref{thm:chernoff_bound} then gives
    that with probability at least $1 - 2 \exp[-d\veps^2 / 8]$
    \begin{equation}
        \abs*{
            \norm*{
                \denoiser^\nom(t, \vx_t^i)
                - \vx^i
            }^2
            -
            \frac{d\sigma_\star^2 \sigma_t^2}{\alpha_t^2 \sigma_\star^2 + \sigma_t^2}
        }
        \leq
        \veps \frac{d\sigma_\star^2 \sigma_t^2}{\alpha_t^2 \sigma_\star^2
            + \sigma_t^2}
    \end{equation}
    For any event $\mathsf{A}$ such that
    \begin{equation}
        \abs*{
            \norm*{
                \denoiser^\nom(t, \vx_t^i)
                - \vx^i
            }^2
            -
            \frac{d\sigma_\star^2 \sigma_t^2}{\alpha_t^2 \sigma_\star^2 + \sigma_t^2}} \leq C ,
    \end{equation}
    on that event, we have that
    \begin{align}
         & \abs*{
            \E*{\norm*{
                    \denoiser^\nom(t, \vx_t^i)
                    - \vx^i
                }^2}
            -
            \frac{d\sigma_\star^2 \sigma_t^2}{\alpha_t^2 \sigma_\star^2 + \sigma_t^2}
        } \leq \E*{\abs*{
                \norm*{
                    \denoiser^\nom(t, \vx_t^i)
                    - \vx^i
                }^2
                -
                \frac{d\sigma_\star^2 \sigma_t^2}{\alpha_t^2 \sigma_\star^2 + \sigma_t^2}
        }}                                                                                                                                                   \\
         & \qquad \leq C + \Pr*{\mathsf{A}^{\mathrm{c}}}^{1/2} \left(
         \E*{\norm*{\vx^i}^4}^{1/2} + \E*{\norm*{\denoiser^\nom(t, \vx_t^i)}^4}^{1/2} \right)
    \end{align}
    Therefore, there exists $C_2 \geq 0$, a numerical constant, such that
    \begin{align}
         & \abs*{
            \E*{\norm*{
                    \denoiser^\nom(t, \vx_t^i)
                    - \vx^i
                }^2}
            -
            \frac{d\sigma_\star^2 \sigma_t^2}{\alpha_t^2 \sigma_\star^2 + \sigma_t^2}
        }                                                                                \\
         & \qquad \leq \veps \frac{d\sigma_\star^2 \sigma_t^2}{\alpha_t^2 \sigma_\star^2
            + \sigma_t^2} + C_2 \left( \max_{k \in [K]}\, \norm{\vmu_\star^k}^2 + (1 + \sigma_\star^2) d  \right) \exp[-\vareps^2 d/ 16] .
    \end{align}
    Therefore, there exists $C_3 \geq 0$, a numerical constant, such that
    \begin{align}
         & \left|\E*{\sL_N(\vmu_\star^1, \dots, \vmu_\star^K, \sigma_\star^2)}
        -
        \E*{\sL_N(\vx^1, \dots, \vx^N, 0)}  - \frac{d\sigma_\star^2 \sigma_t^2}{\alpha_t^2 \sigma_\star^2
        + \sigma_t^2}\right|                                                                                  \\ & \qquad \leq \veps \frac{d\sigma_\star^2 \sigma_t^2}{\alpha_t^2 \sigma_\star^2
            + \sigma_t^2}
        + C_3 \left( \max_{k \in [K]}\, \norm{\vmu_\star^k}^2 + (1 + \sigma_\star^2) d  \right) (K^2 + N^2) \\
         & \qquad \qquad \times \left\lbrace \exp\left[ -
            \frac{\gamma^2\alpha_t^2}{2(\alpha_t^2 \sigma_{\star}^2 + \sigma_t^2)}
            \right] + \exp\left[-\frac{\alpha_t^2 \sigma_\star^2(1-\vareps)d}{2\sigma_t^2}\right]  + \exp[-\vareps^2 d/ 16] \right\rbrace .
    \end{align}
    Let us further simplify the bound.
    We have that
    \begin{align}
         & \left|\E*{\sL_N(\vmu_\star^1, \dots, \vmu_\star^K, \sigma_\star^2)}
        -
        \E*{\sL_N(\vx^1, \dots, \vx^N, 0)}  - \frac{d\sigma_\star^2 \sigma_t^2}{\alpha_t^2 \sigma_\star^2
        + \sigma_t^2}\right|                                                   \\ & \qquad \leq \vareps \frac{d\sigma_\star^2 \sigma_t^2}{\alpha_t^2 \sigma_\star^2
            + \sigma_t^2}  + C_3 \left(1 + \max_{k \in [K]}\, \norm{\vmu_\star^k}^2 / d + \sigma_\star^2 \right) \mathrm{Poly}(d)  \frac{d\sigma_\star^2 \sigma_t^2}{\alpha_t^2 \sigma_\star^2
            + \sigma_t^2} \left(1 / \sigma_\star^2 + \snr_t \right)
        \\
         & \qquad \qquad \times \left\lbrace \exp\left[ -
            \frac{\gamma^2\alpha_t^2}{2(\alpha_t^2 \sigma_{\star}^2 + \sigma_t^2)}
            \right] + \exp\left[-\frac{\alpha_t^2 \sigma_\star^2(1-\vareps)d}{2\sigma_t^2}\right]  + \exp[-\vareps^2 d/ 16] \right\rbrace ,
    \end{align}
    where we have used that
    \begin{equation}
        \frac{d\sigma_\star^2 \sigma_t^2}{\alpha_t^2 \sigma_\star^2
            + \sigma_t^2} \left(1 / \sigma_\star^2 + \snr_t \right) \geq 1 ,
    \end{equation}
    with $\snr_t = \alpha_t^2 / \sigma_t^2$.

\end{proof}

\begin{proposition}{}{simplification_generalization_scaling_appendix}
    Assume that $N = \mathrm{poly}(d)$,
    $\gamma^2 = \Theta(d)$, $\max_{k \in [K]}\, \norm{\vmu_\star^k}^2
    = \Theta(d)$ and $\sigma_\star^2 = \Theta(1)$. Let $\kappa(d)
    = \snrinv(\Theta(\log(d)^2/d))$. We have that uniformly on $t \in [0, \kappa(d)]$
    \begin{equation}
                    \E*{\sL_N(\vmu^1_\star, \dots, \vmu_\star^K, \sigma_{\star}^2)}
                    -
                    \E*{\sL_N(\vx^1, \dots, \vx^N, 0)}
                     = \Theta\left(\frac{d\sigma_\star^2 \sigma_t^2}{\alpha_t^2 \sigma_\star^2
                + \sigma_t^2}\right).
    \end{equation}
    The claim will then follow by an appropriate choice of constant in the
    definition of $\kappa(d)$, which will make the residuals in the previous
    expression of the correct order of magnitude to be nontrivial, and guarantee
    the approximation for $t$ within the claimed interval.
\end{proposition}

\begin{proof}
    First note that $\vareps = \Theta(\log(d)/d^{1/2})$ satisfies the coupling condition. Next, if we assume that $d$ is large enough so that $1 - \vareps \leq 1/2$. We have
    \begin{align}
         & \left|\E*{\sL_N(\vmu_\star^1, \dots, \vmu_\star^K, \sigma_\star^2)}
        -
        \E*{\sL_N(\vx^1, \dots, \vx^N, 0)}  - \frac{d\sigma_\star^2 \sigma_t^2}{\alpha_t^2 \sigma_\star^2
        + \sigma_t^2}\right|                                                                                                                                                         \\ & \qquad \leq \vareps \frac{d\sigma_\star^2 \sigma_t^2}{\alpha_t^2 \sigma_\star^2
            + \sigma_t^2}  + C_4 \left(1 + \max_{k \in [K]}\, \norm{\vmu_\star^k}^2 / d + \sigma_\star^2 \right) \mathrm{Poly}(d)  \frac{d\sigma_\star^2 \sigma_t^2}{\alpha_t^2 \sigma_\star^2
            + \sigma_t^2} \left(1 / \sigma_\star^2 + \snr_t \right)
        \\
         & \qquad \qquad \times \exp\left[ -\frac{1}{16} \min\left( \log(d)^2,
         \snr_t \sigma_\star^2 d, \frac{\gamma^2 \snr_t}{1 + \sigma_\star^2
         \snr_t} \right) \right] .
    \end{align}

\end{proof}

\begin{theorem}{}{partial-memorizing-denoiser-loss-approximation_appendix}
    Given vectors $(\vmu_{\star}^k)_{k=1}^K$ satisfying
    \begin{equation}
        \min_{k \neq k'}\, \norm*{\vmu_{\star}^k - \vmu_{\star}^{k'}} \geq \gamma
        > 0,
    \end{equation}
    consider the distribution $\pi = (1/K) \sum_{k=1}^{K} \sN(\vmu_\star^{k},
        \sigma_{\star}^2 \vI)$.
    For $i \in [N]$, let $\vx^i \sim \pi$, fix $0 \leq t \leq 1$, and let
    $\vx_t^i \sim \alpha_t
        \vx^i + \sigma_t \vg$, where $\vg \sim \sN(\Zero, \vI)$ is independent from $\vx^i$.
    Consider the partial memorizing denoiser $\denoiserpmem(t, \vx_t)$. 
    Then for any $0 \leq \veps \leq 1$ satisfying the coupling conditions
    \begin{equation}
        \veps \leq
        \frac{
            \alpha_t \gamma
        }{
            2\sqrt{d(\alpha_t^2 \sigma_\star^2 + \sigma_t^2)}
        },
    \end{equation}
    and
    \begin{equation}
        \frac{\veps^2}{1-\veps} \leq \frac{\alpha_t^2
            \sigma_\star^2}{2\sigma_t^2},
    \end{equation}
we have that 
    \begin{align}
            &\left(1 - \frac{\ell}{N}\right) d \sigma_\star^2 - \Xi(\vareps, t, d) \\
            & \qquad \leq 
                \E*{\sL_N(\vx^1, \dots, \vx^\ell, 0)}
                -
                \E*{\sL_N(\vx^1, \dots, \vx^N, 0} \\
                &\qquad \qquad \leq 2 \left(1 - \frac{\ell}{N}\right) d \sigma_\star^2 + \Xi(\vareps, t, d),
    \end{align}
    where the residual $\Xi$ is explicit in the proof.
\end{theorem}

\begin{proof}
The proof will be an application of \Cref{lemma:p-mem-denoiser-1sparse-lp} and
    \Cref{lemma:mem-denoiser-1sparse-lp}.
    \paragraph{Softmax value decomposition.} 
    Using \eqref{eq:training-loss-ortho}, we have that
    \begin{align}
         & \E*{\sL_N(\vx^1, \dots, \vx^\ell, 0)}
        -
        \E*{\sL_N(\vx^1, \dots, \vx^N, 0)}                                     \\
         & \quad =
        \frac{1}{N}
        \sum_{i=1}^N 
        \E*{
        \norm*{
        \denoiserpmem(t, \vx_t^i) - \denoisermem(t, \vx_t^i)
        }^2
        }                                           \\
         & \quad =
        \frac{1}{N}
        \sum_{i=1}^N 
        \E*{
        \norm*{
        \denoiserpmem(t, \vx_t^i) - \vx_i + \vx_i - \denoisermem(t, \vx_t^i)
        }^2
        }   .
    \end{align}

    We have that 
    \begin{align}
        &\E*{\sL_N(\vx^1, \dots, \vx^\ell, 0)}
        -
        \E*{\sL_N(\vx^1, \dots, \vx^N, 0)} \\
        &\qquad = \frac{1}{N}
        \sum_{i=1}^N 
        \E*{
        \norm*{
        \denoiserpmem(t, \vx_t^i) - \vx_i + \vx_i - \denoisermem(t, \vx_t^i)
        }^2
        } \\
        &\qquad = \frac{1}{N}
        \sum_{i=1}^N 
        \E*{
        \norm*{
        \denoiserpmem(t, \vx_t^i) - \vx^i}^2
        }
        + 
                \E*{
        \norm*{
        \denoisermem(t, \vx_t^i) - \vx^i}^2
        }
        + 2 \Lambda_t^i ,
    \end{align}
    where 
    \begin{equation}
        |\Lambda_t^i| \leq \E*{
        \norm*{
        \denoisermem(t, \vx_t^i) - \vx^i}^2
        }^{1/2} \E*{
        \norm*{
        \denoiserpmem(t, \vx_t^i)- \vx^i}^2
        }^{1/2} . 
    \end{equation}

    In the rest of the proof, we control each term in \eqref{eq:key_inequality}.

    \paragraph{Control of memorizing denoiser.} Second, we are going to control
    $\E*{\norm*{\vx_i - \denoisermem(t, \vx_t^i)}^2}$. The proof is similar to the one of \Cref{thm:generalizing-denoiser-loss-approximation_appendix} but we reproduce it for completeness.  We recall that we have
    \begin{equation}
        \denoisermem(t,\vx_t) = \sum_{i=1}^{N} \vx^{i} \softmax(\vw)_i ,
    \end{equation}
    where
    \begin{equation}
        w_i(\vx_t) = -\frac{1}{2}\| \alpha_t \vx^{i} - \vx_t^i \|^2 / \sigma_t^2
    \end{equation}
    Similarly as before, we have for any $p \in \nset$
    \begin{equation}
        \E*{ \max_{i \in [N]} \| \vx_i \|^p}  \leq 2^{2p}N (\Gamma^p + (d/2 + p)^{p/2}) .
    \end{equation}
    Therefore, we get that
    \begin{equation}
        \E*{\| \vx_i \|^p} \leq 2^{2p} N (\Gamma^p + \sigma_\star^{p}(d/2
        + p)^{p/2}) , \qquad \E*{\| \denoisermem(t, \vx_t^i) \|^p} \leq 2^{2p} N (\Gamma^p + \sigma_\star^{p}(d/2 + p)^{p/2}) .
    \end{equation}
    Note that this upper-bound is rather loose but given the rate of growth of
    $N$ with respect to $d$ that we will consider we won't need a tighter bound.
    For any event $\mathsf{A}$ such that $\norm*{\denoisermem(t, \vx_t^i) - \vx^i}^2 \leq C$, we have that
    \begin{align}
        \E*{\norm*{\denoisermem(t, \vx_t^i) - \vx^i}^2} & \leq
        C \Pr*{\mathsf{A}} + 2 \Pr*{\mathsf{A}^{\mathrm{c}}}^{1/2} \left(
        \E*{\norm*{\vx_t^i}^4}^{1/2} + \E*{\norm*{\denoiser^\nom(t, \vx_t^i)}^4}^{1/2} \right) \\
                                                              & \leq C  + 16 \left( \Gamma^2 + \sigma_\star^2 (d/2 + 4) \right) \Pr*{\mathsf{A}^{\mathrm{c}}}^{1/2} .
    \end{align}
    Therefore, combining this result with \Cref{lemma:mem-denoiser-1sparse-lp}, there exists a numerical constant $C_1 \geq 0$ such that
    \begin{align}
         & \E*{\norm*{\denoisermem(t, \vx_t^i) - \vx^i}^2}                                                                                                                                                                       \\
         & \qquad \leq C_1 \left( \max_{k \in [K]}\, \norm{\vmu_\star^k}^2 + (1 + \sigma_\star^2) d  \right) N^2 \left(\exp\left[-\frac{\alpha_t^2 \sigma_\star^2(1-\vareps)d}{2\sigma_t^2}\right] + \exp[-d \vareps^2/16] \right) .
    \end{align}

    \paragraph{Control of  the partial memorizing denoiser. } Third, we are
    going to control $\E*{\norm*{\vx_i - \denoiserpmem(t, \vx_t^i)}^2}$. First, similarly as before, we have that 
        \begin{equation}
        \E*{\| \vx_i \|^p} \leq 2^{2p} N (\Gamma^p + \sigma_\star^{p}(d/2
            + p)^{p/2}) , \qquad \E*{\| \denoisermem(t, \vx_t^i) \|^p} \leq 2^{2p} N (\Gamma^p + \sigma_\star^{p}(d/2 + p)^{p/2}) .
    \end{equation}
For any event $\mathsf{A}$, we have that
    \begin{align}
        &\abs*{\E*{\norm*{\denoiserpmem(t, \vx_t^i) - \vx^i}^2}
        - \E*{\norm*{\denoiserpmem(t, \vx_t^i) - \vx^i}^2 1_{\mathsf{A}}}}& \\
        &\qquad \leq  2 \Pr*{\mathsf{A}^{\mathrm{c}}}^{1/2} \left(
        \E*{\norm*{\vx_t^i}^4}^{1/2} + \E*{\norm*{\denoiserpmem(t, \vx_t^i)}^4}^{1/2} \right) \\
                                                              &\qquad \leq 16 \left( \Gamma^2 + \sigma_\star^2 (d/2 + 4) \right) \Pr*{\mathsf{A}^{\mathrm{c}}}^{1/2} . \label{eq:intermediate_result_pmem}
    \end{align}
Hence for $i \in [\ell]$, using \Cref{lemma:p-mem-denoiser-1sparse-lp} we have 
    \begin{align}
        &\E*{\norm*{
            \denoiserpmem(t, \vx_t^i) - \vx^i
        }^2}
        \\
        & \qquad \leq
        2N \Gamma_\star
        \exp\left[- \frac{\alpha_t^2 \sigma_\star^2(1-\veps) d}{\sigma_t^2} \right] + 16 \left( \Gamma^2 + \sigma_\star^2 (d/2 + 4) \right)  (8N)^{1/2}\exp[-d\veps^2 / 16] .
    \end{align}
Now, for $i \notin [\ell]$, we have that 
\begin{align}
    &\norm*{
            \denoiserpmem(t, \vx_t^i) - \vx^i
        }^2 = \norm*{
            \denoiserpmem(t, \vx_t^i) - \sum_{j \in \mathsf{S}_i} \softmax(\vw|_{\mathsf{S}_i})_j \vx^j + \sum_{j \in \mathsf{S}_i} \softmax(\vw|_{\mathsf{S}_i})_j \vx^j - \vx^i
        }^2 \\
        &= \norm*{
            \denoiserpmem(t, \vx_t^i) - \sum_{j \in \mathsf{S}_i} \softmax(\vw|_{\mathsf{S}_i})_j \vx^j}^2 \\
            &\qquad + \norm*{\sum_{j \in \mathsf{S}_i} \softmax(\vw|_{\mathsf{S}_i})_j \vx^j - \vx^i
        }^2 \\
        &\qquad \quad + 2 \left\langle \denoiserpmem(t, \vx_t^i) - \sum_{j \in \mathsf{S}_i} \softmax(\vw|_{\mathsf{S}_i})_j \vx^j, \sum_{j \in \mathsf{S}_i} \softmax(\vw|_{\mathsf{S}_i})_j \vx^j - \vx^i \right\rangle 
\end{align}
Therefore, we have that 
\begin{align}
    &\abs*{\norm*{
            \denoiserpmem(t, \vx_t^i) - \vx^i
        }^2 - \norm*{\sum_{j \in \mathsf{S}_i} \softmax(\vw|_{\mathsf{S}_i})_j \vx^j - \vx^i
        }^2} \\
        &\leq \norm*{
            \denoiserpmem(t, \vx_t^i) - \sum_{j \in \mathsf{S}_i} \softmax(\vw|_{\mathsf{S}_i})_j \vx^j}^2 \\
            &\qquad + \norm*{\denoiserpmem(t, \vx_t^i) - \sum_{j \in \mathsf{S}_i} \softmax(\vw|_{\mathsf{S}_i})_j \vx^j} \norm*{\sum_{j \in \mathsf{S}_i} \softmax(\vw|_{\mathsf{S}_i})_j \vx^j - \vx^i} 
\end{align}
Therefore, we get that 
\begin{align}
    &\abs*{\norm*{
            \denoiserpmem(t, \vx_t^i) - \vx^i
        }^2 - \norm*{\sum_{j \in \mathsf{S}_i} \softmax(\vw|_{\mathsf{S}_i})_j \vx^j - \vx^i
        }^2} \\
        &\leq \norm*{
            \denoiserpmem(t, \vx_t^i) - \sum_{j \in \mathsf{S}_i} \softmax(\vw|_{\mathsf{S}_i})_j \vx^j}^2 \\
            &\qquad + \norm*{\denoiserpmem(t, \vx_t^i) - \sum_{j \in \mathsf{S}_i} \softmax(\vw|_{\mathsf{S}_i})_j \vx^j} \norm*{\sum_{j \in \mathsf{S}_i} \softmax(\vw|_{\mathsf{S}_i})_j \vx^j - \vx^i}  
         . 
\end{align}
Hence, using this result and \Cref{lemma:p-mem-denoiser-1sparse-lp} we get that with probability at least $1 - 8N\exp[-d\veps^2 / 8] - K^{-\log(d)}(1 + \log(K) \log(d))$
\begin{align}
   &\E*{\abs*{\norm*{
            \denoiserpmem(t, \vx_t^i) - \vx^i
        }^2 - \norm*{\sum_{j \in \mathsf{S}_i} \softmax(\vw|_{\mathsf{S}_i})_j \vx^j - \vx^i
        }^2} 1_{\mathsf{A}}}  \\
        & \qquad \leq  8 \Gamma_\star (1 + N)^{2} \exp\left[ -\frac{\alpha_t^2 \gamma^2}{\sigma_t^2} \right] N (\Gamma + \sigma_\star(d/2 + p)^{1/2}) ,
\end{align}
where $\mathsf{A}$ is the event of \Cref{lemma:p-mem-denoiser-1sparse-lp}. Hence, combining this result and \eqref{eq:intermediate_result_pmem} we have 
    \begin{align}
        &\abs*{\E*{\norm*{
            \denoiserpmem(t, \vx_t^i) - \vx^i
        }^2} - \E*{\norm*{\sum_{j \in \mathsf{S}_i} \softmax(\vw|_{\mathsf{S}_i})_j \vx^j - \vx^i}^2} }
        \\
        & \qquad \leq
        \Gamma_\star (1 + N)^{2} \exp\left[ -\frac{\alpha_t^2 \gamma^2}{\sigma_t^2} \right] \\
        & \qquad + 16 \left( \Gamma^2 + \sigma_\star^2 (d/2 + 4) \right)  \left[(8N)^{1/2}\exp[-d\veps^2 / 16] + K^{-\log(d)/2}(1 + \log(K) \log(d))^{1/2} \right] \\
        &\qquad \qquad + 8 \Gamma_\star (1 + N)^{2} \exp\left[ -\frac{\alpha_t^2 \gamma^2}{\sigma_t^2} \right] N (\Gamma + \sigma_\star(d/2 + p)^{1/2}) . 
    \end{align}
    Finally, in order to control $\E*{\norm*{\sum_{j \in \mathsf{S}_i} \softmax(\vw|_{\mathsf{S}_i})_j \vx^j - \vx^i}^2} $, we use \Cref{prop:control_softmax_norm} which concludes the proof. 
\end{proof}

\begin{proposition}{}{simplification_p_mem_scaling_appendix}
    Assume that $N = \mathrm{poly}(d)$,
    $\gamma^2 = \Theta(d)$, $\max_{k \in [K]}\, \norm{\vmu_\star^k}^2
    = \Theta(d)$ and $\sigma_\star^2 = \Theta(1)$. Let $\kappa(d) = \snrinv(\log(d)^2/d)$. We have that uniformly on $t \in [0, \kappa(d)]$
    \begin{equation}
                    \E*{\sL_N(\vx^1, \dots, \vx^\ell, 0)}
                    -
                    \E*{\sL_N(\vx^1, \dots, \vx^N, 0)}
                     = \Theta\left(\left(1 - \frac{\ell}{N}\right) d \sigma_\star^2\right).
    \end{equation}
\end{proposition}

\begin{proof}
    The proof of this result is similar to \Cref{prop:simplification_generalization_scaling_appendix} by letting $\vareps = \Theta(\log(d)/d^{1/2})$. 
\end{proof}

\section{Additional Related Work}

\label{sec:related_works_app}

\paragraph{Memorization in diffusion models.} Understanding memorization and generalization properties of diffusion
models is crucial for practitioners \citep{somepalli2023diffusion,
    ren2024unveiling, rahman2024frame,
    wang2024replication,chen2024investigating,stein2024exposing}.
Indeed, concerns about privacy
\citep{ghalebikesabi2023differentially, carlini2023extracting,
    nasr2023scalable} and copyright infringement
\citep{cui2023diffusionshield,
    wang2024replication,vyas2023provable,franceschelli2022copyright} are
key issues as  these models are deployed. Several factors can
influence the memorization capabilities of diffusion models.
Duplication and out-of-distribution samples have been shown to lead
to replication in diffusion models
\citep{carlini2023extracting,ross2024geometric,webster2023reproducible}
and solutions have been proposed by curating the dataset
\citep{chen2024towards} or introducing dummy data and adapting
diffusion models to unseen data
\citep{daras2024ambient,yoon2023diffusion}. Neural network
architectures have also been shown to play a role in memorization
with \citep{chavhan2024memorized} identifying specific neurons
causing memorization and \citep{wen2024detecting,wang2024discrepancy}
analyzing variability in the prediction function to identify
problematic prompts in image models.
We refer to \citep{gu2023memorization} for an in-depth experimental
investigation of those issues. Finally, we highlight the work of
\citep{kadkhodaie2023generalization}, in which the authors
investigate the inductive bias image denoisers to investigate the
generalization properties of state-of-the-art diffusion models.

\paragraph{Definitions of memorization.} Note that our definition of memorization (\Cref{def:mem}) is connected to the one of Eidetic Memorization as introduced by \citet{carlini2023extracting}. Similarly, to \citep{carlini2023extracting}, we acknowledge that this notion of memorization is \emph{strong} as it implies that an image is memorized if there exists a near perfect copy in the training dataset which does not fully capture \emph{copyright infringement} or \emph{data privacy} issues. We refer to \citep{elkin2023can} for an in-depth discussion of those issues.

\paragraph{Memorization and overfitting.} \textit{A priori}, it may not be clear
why memorization in diffusion models is substantially different from
\textit{overfitting} which has been a long-standing and well-studied property of
machine learning systems \cite{hastie2009elements}. Overfitting is characterized
by a large gap between training and validation losses. However, overfitting is
not directly connected to memorization. For a small number of parameters, the
partially memorizing denoisers defined in \Cref{sec:training_losses} have large
training and validation losses which ultimately have a small gap, meaning that
such models are not significantly overfit; however, every single sample they
generate is memorized. Meanwhile, the model corresponding to the fully
memorizing denoiser defined in \Cref{sec:setup} is both very overfit and does memorize. It is also worthwhile to discuss the connection of diffusion model memorization with \textit{benign overfitting} \citep{bartlett2020benign} and \textit{double descent} \citep{belkin2019reconciling}. In particular, we wish to clarify why the setting of diffusion models may be separated from benign overfitting. The core difference is as follows: in the usual double descent setting, one studies an \textit{overparameterized} learning problem, such as regression or classification, and there are many models at a fixed parameter count which obtain the minimal training loss; then it is up to (implicit) regularization to choose the best solution, which benignly interpolates the training data. As the number of parameters increase, the training loss of the trained model never increases, and the validation loss eventually also decreases (after an initial increase for the purpose of overfitting). Meanwhile, in the case of diffusion, there is exactly one model which obtains the minimal training loss, which is the memorizing denoiser, and this does not change no matter how many parameters are allocated. Certainly the memorizing denoiser does not benignly interpolate the training data. As the number of parameters increase past a certain point, the training loss never decreases, and in fact may increase, showing that double descent or benign overfitting indeed do not apply straightforwardly to the case of diffusion.

\section{Experimental Details} \label{app:experiments}

We run all experiments on several Nvidia A100 80GB GPUs using Jax 0.6.0 and Equinox 0.12 \citep{kidger2021equinox}. Each training/evaluation job occurs on a single A100 and the results are saved to file to be aggregated later. Aggregation, analysis, and visualization occur on a single A100.

\subsection{Datasets, Optimization, and Initialization Details} \label{app:sub_optimization}

In our experiments in \Cref{sub:experiments_gmm}, we generate synthetic Gaussian mixture models, by generating the \(K\) means \(\vmu_{\star}^{i}\) uniformly on the sphere of radius \(\sqrt{d}\) and setting the ground truth variance \(\sigma_{\star}^{2} = 1\), as prescribed by the results in \Cref{sec:training_losses}. The experiments conducted in \Cref{fig:train-loss-approx-verify,fig:memorization_phase_transition_loss,fig:learned_means_variance} are all under the setting \(N = 200\), \(d = 50\), and \(K = 12\). For the sweep in \Cref{fig:phase_transition_predictable_approximation} we take \((N, d, K)\) tuples from \([50, 100, 150, 200] \times [30, 40, 50, 60] \times [3, 6, 9, 12]\), obtaining a total of 64 \((N, d, K)\) tuples. For each setting of \((N, d, K)\) (including the sweep in \Cref{fig:phase_transition_predictable_approximation}) we train 20 models at different model sizes \(M\): starting from \(M = \lfloor N/10\rfloor\) and moving in increments of \(\lfloor N/10\rfloor\) to \(M = 10 \lfloor N/10 \rfloor\); then, training 10 more models where \(M\) is equally spaced between the \(M\) where the phase transition starts and the \(M\) where it ends (using the empirical criterion in \Cref{sec:experiments}).

In our experiments in \Cref{sub:experiments_image}, we generate synthetic colored FashionMNIST data by sampling \(K\) FashionMNIST \citep{xiao2017fashion} images uniformly at random to use as ``templates'', then using the PIL (Pillow) utility to reshape them to \(15 \times 15\) resolution. For each of the \(K\) templates (components) we generate a color vector using a Gaussian with ground truth mean and variance \((\vu_{\star}^{i}, \vsigma_{\star}^{2}) = (\bm{0}, 1) \in \bbR^{3} \times \bbR\). We take the Kronecker product of the color vector and the template as described in \Cref{sub:experiments_image} in order to form the sample. \Cref{fig:image_data_samples} uses the setting \(N = K = 8\) while the experiment in \Cref{fig:image_data_phase_transition} uses the setting \(N = 100\), \(K = 4\), and color dimension \(d = 3\). Here we train \(10\) models at different model sizes \(M\), starting from \(M = \lfloor N / 10\rfloor\) and moving in increments of \(\lfloor N / 10\rfloor\) to \(M = 10\lfloor N/10\rfloor\).

For all experiments, we use the ``variance preserving'' process which yields  \(\alpha_{t} = \sqrt{1 - t^{2}}\) and \(\sigma_{t} = t\) for  \(t \in [0, 1]\). We use the objective \eqref{eq:denoising-objective} to train our model denoisers. For training, we always train with the loss weighting \(\lambda(t) := \alpha_{t}^{2}/\sigma_{t}^{2}\), which is equivalent to using \textit{noise prediction} (\citep{karras2022elucidating}), and \(t \sim \Unif{(t_{\ell})_{\ell = 0}^{L}}\) where we use \(L = 25\) decreasing timesteps \(t_{\ell} = 0.01 + 0.998(L - \ell) / L = 0.999 - 0.998\ell/L \in (0, 1)\). In lieu of computing the (obviously intractable) inner expectation in \(\mathcal{L}_{N, t}\), we use \(N_{\mathrm{dup}} := 100\) Gaussian noise draws for each of the \(N\) samples to estimate the expectation. We use full-batch Adam for \(N_{\mathrm{epochs}}\) epochs (also, iterations) to optimize the objective; for experiments in \Cref{sub:experiments_gmm} we have \(N_{\mathrm{epochs}} = 50,000\) and for experiments in \Cref{sub:experiments_image} we have \(N_{\mathrm{epochs}} = 100,000\). We use a ``warmup-decay'' learning rate schedule: for \(N_{\mathrm{warmup}} := N_{\mathrm{epochs}}/10\) epochs the learning rate linearly increases from \(0\) to \(10^{-3}\); for the remaining \(N_{\mathrm{decay}} := N_{\mathrm{epochs}} - N_{\mathrm{warmup}}\) epochs the learning rate linearly decreases from \(10^{-3}\) to \(10^{-6}\).

For all models we train, we use a ``partial memorization initialization'' along with the Adam optimizer We use this initialization because the loss at a truly random initialization is extremely high, in many cases often \textit{at least eight orders of magnitude} larger than the loss at optimum, and Adam is often unable to learn effectively given this massive conditioning. The partial memorizing initialization, in the context of the isotropic Gaussian mixture model, sets the initial \(\sigma^{2}\) to \(10^{-6}\), and each mean \(\vmu^{i}\) to a random sample. In the context of the simple image model, it sets each initial template parameter \(\vA_{\vx^{i}}\) to the template which generates a sample, the corresponding initial color vector to the unique vector which generates the (not identically zero) sample given that template, and the initial color variance \(\sigma^{2}\) to \(10^{-6}\). Notice that in the experiments such as \Cref{fig:memorization_phase_transition_loss}, even models initialized with this partial memorizing initialization end up learning (nearly) generalizing solutions --- unless of course it is more favorable to memorize, which occurs with very large \(M\).

\subsection{Formal Description of Sampling Scheme} \label{app:sub_sampling}

We use the implementation of the DDIM sample prescribed in \citet{De-Bortoli2025-ng}, i.e., using the above notation and given a denoiser \(\bar{\vx}\)
\begin{equation}
    \hat{\vx}_{t_{\ell + 1}} = \frac{\sigma_{t_{\ell + 1}}}{\sigma_{t_{\ell}}}\hat{\vx}_{t_{\ell}} + \left(\alpha_{t_{\ell + 1}} - \frac{\sigma_{t_{\ell + 1}}}{\sigma_{t_{\ell}}}\alpha_{t_{\ell}}\right)\bar{\vx}(t_{\ell}, \hat{\vx}_{t_{\ell}}), \qquad \hat{\vx}_{t_{0}} \sim \mathcal{N}(\bm{0}, \bm{I}).
\end{equation}
As previously stated, we use increasing timesteps \(t_{\ell} = 0.999 - \ell/L\) for \(L = 25\) and \(\ell \in \{0, 1, \dots, L\}\). Notice that we use the same timesteps for sampling as for training. We implement this iteration using the DiffusionLab PyPI package \citep{pai25diffusionlab}.

\subsection{Formal Description of Loss Weighting Regression} \label{app:sub_loss_weighting}

Recall that in \Cref{sub:experiments_gmm} we predicted the phase transition using the loss approximations derived in \Cref{sec:training_losses}. To do this, we solved a regression problem \eqref{eq:weighting_regression_problem}. Here, we will discuss how we solve this problem efficiently by smoothing and regularization.

Namely, as a bilevel semi-discrete quadratic program, the problem \eqref{eq:weighting_regression_problem} is hard to optimize outright via gradient methods. As a result we parameterize the distribution over \(M\) via a temperature-weighted softmax with high temperature \(\tau = 1/20\), placing an entropy penalty (\(\beta_{\mathrm{sparsity}} = 10^{-3}\)) on the softmax output to ensure that the learned distribution over \(M\) is sparse. The overall problem is (using the notation from \eqref{eq:weighting_regression_problem})
\begin{align}\label{eq:weighting_regression_full_mse}
    \min_{\tilde{\lambda}} \quad 
    &\sum_{(N, d, K)}\left(\frac{\bar{M}_{\mathrm{pt}}(N, d, K, \tilde{\lambda})}{N} - \frac{M_{\mathrm{pt}}}{N}\right)^{2} \\
    &\qquad - \beta_{\mathrm{sparsity}}\sum_{(N, d, K)}\sum_{M}\tilde{p}(N, d, K, M, \tilde{\lambda})\log \tilde{p}(N, d, K, M, \tilde{\lambda}) \\
    \text{where} \qquad &\bar{M}_{\mathrm{pt}}(N, d, K, \tilde{\lambda}) = \sum_{M} M \cdot \tilde{p}(N, d, K, M, \tilde{\lambda}) \\
    \text{and} \qquad & \tilde{p}(N, d, K, \cdot, \tilde{\lambda}) = \resizebox{0.625\textwidth}{!}{\(\operatorname{softmax}\mathopen{}\left(\left\{-\frac{1}{\tau}\left(\sum_{\ell = 0}^{L}\tilde{\lambda}(t_{\ell})(\check{L}_{N, t}(\theta_{\pmem, M}(N, d, K)) - \check{L}_{N, t}(\theta_{\star}(N, d, K))\right)^{2}\right\}_{M}\right)\)}
\end{align}
In all cases, when we report the loss we refer to just the MSE component described in \eqref{eq:weighting_regression_full_mse}.

\subsection{Simple Image Model Calculations} \label{app:sub_image_model_simplified}

In this section we simplify the computation of the colored image denoiser in \Cref{sub:experiments_image}. Recall our setup in \Cref{sub:experiments_image} along with the notation in \Cref{lemma:denoiser_mog}. Under the settings \(\vmu_{\star}^{i} = \vA_{\star}^{i}\vu_{\star}^{i}\) and \(\vSigma_{\star}^{i} = \sigma_{\star}^{2}\vA_{i}^{\star}(\vA_{i}^{\star})^{\top}\), it holds that
\begin{align}
    \vmu_{\star}^{i}
     & = \vA_{\star}^{i}\vu_{\star}^{i} = (\bm{I} \kron \bm{x}_{\star}^{i})\vu_{\star}^{i} = \vu_{\star}^{i} \kron \vx_{\star}^{i} \\
     \vSigma_{\star}^{i}
     &= \sigma_{\star}^{2}(\vI \kron \vx_{\star}^{i})(\vI \kron \vx_{\star}^{i})^{\top} = \sigma_{\star}^{2}(\vI \kron \vx_{\star}^{i})(\vI \kron (\vx_{\star}^{i})^{\top}) = \sigma_{\star}^{2}(\vI \kron \vx_{\star}^{i}(\vx_{\star}^{i})^{\top})
\end{align}
Also, simplifying the inverse \((\alpha^{2}\vSigma_{\star}^{i} + \sigma^{2}\vI)^{-1}\) using the Sherman-Morrison-Woodbury identity obtains
\begin{align}
    (\alpha^{2}\bm{\Sigma}_{\star}^{i} + \sigma^{2}\bm{I})^{-1}
     & = \frac{1}{\sigma^{2}}\left(\vI +
    \frac{\alpha^{2}}{\sigma^{2}}\vSigma_{\star}^{i}\right)^{-1} \\
     & = \frac{1}{\sigma^{2}}\left(\vI +
    \frac{\alpha^{2}\sigma_{\star}^{2}}{\sigma^{2}}\vA_{\star}^{i}(\vA_{\star}^{i})^{\top}\right)^{-1}
    \\
     & = \frac{1}{\sigma^{2}}\left(\vI -
    \vA_{\star}^{i}\left[\frac{\sigma^{2}}{\alpha^{2}\sigma_{\star}^{2}}\vI +
        (\vA_{\star}^{i})^{\top}\vA_{\star}^{i}\right]^{-1}(\vA_{\star}^{i})^{\top}\right).
\end{align}
Calculating the interior term first, we obtain
\begin{align}
    (\vA_{\star}^{i})^{\top}\vA_{\star}^{i} 
    &= (\vI \kron \vx_{\star}^{i})^{\top}(\vI \kron \vx_{\star}^{i}) = (\vI \kron (\vx_{\star}^{i})^{\top})(\vI \kron \vx_{\star}^{i}) = \vI \kron [(\vx_{\star}^{i})^{\top}(\vx_{\star}^{i})] \\
    &= \vI \kron [\norm{\vx_{\star}^{i}}^{2}] = \norm{\vx_{\star}^{i}}^{2}\vI,
\end{align}
which yields 
\begin{align}
    (\alpha^{2}\vSigma_{\star}^{i} + \sigma^{2}\vI)^{-1} 
    &= \frac{1}{\sigma^{2}}\left(\vI - \frac{1}{\frac{\sigma^{2}}{\alpha^{2}\sigma_{\star}^{2}} + \norm{\vx_{\star}^{i}}^{2}}\vA_{\star}^{i}(\vA_{\star}^{i})^{\top}\right) \\
    &= \frac{1}{\sigma^{2}}\left(\vI - \frac{\alpha^{2}\sigma_{\star}^{2}}{\sigma^{2} + \alpha^{2}\sigma_{\star}^{2}\norm{\vx_{\star}^{i}}^{2}}\vA_{\star}^{i}(\vA_{\star}^{i})^{\top}\right)
\end{align}
Another part that requires elaboration is the action of
\((\vA_{\star}^{i})^{\top}\) on a (block) vector, say \(\vtheta\), which has
\begin{equation}
    (\vA_{\star}^{i})^{\top}\vtheta = (\vI \kron \vx_{\star}^{i})^{\top}\vtheta = (\vI \kron (\vx_{\star}^{i})^{\top}) \vtheta = \begin{bmatrix}(\vx_{\star}^{i})^{\top}\vtheta_{1} \\ \vdots \\ (\vx_{\star}^{i})^{\top}\vtheta_{c}\end{bmatrix}.
\end{equation}
Finally, the last part that can do with simplification is the
log-determinant, which obtains
\begin{align}
    \log\det(\alpha^{2}\vSigma_{\star}^{i} + \sigma^{2}\vI)
     & = \log\det\left(\sigma^{2}\left[\vI +
    \frac{\alpha^{2}}{\sigma^{2}}\vSigma_{\star}^{i}\right]\right) \\
     & = \log\det(\sigma^{2}\vI) + \log\det\left(\vI +
    \frac{\alpha^{2}}{\sigma^{2}}\vSigma_{\star}^{i}\right)        \\
     & = 2cd^{2}\log(\sigma) + \log\det\left(\vI + \frac{\alpha^{2}\sigma_{\star}^{2}}{\sigma^{2}}\vA_{\star}^{i}(\vA_{\star}^{i})^{\top}\right) \\
     & = 2cd^{2}\log(\sigma) + \log\det\left(\vI + \frac{\alpha^{2}\sigma_{\star}^{2}}{\sigma^{2}}(\vA_{\star}^{i})^{\top}\vA_{\star}^{i}\right) \\
     & = 2cd^{2}\log(\sigma) + \log\det\left(\vI + \frac{\alpha^{2}\sigma_{\star}^{2}}{\sigma^{2}}\norm{\vx_{\star}^{i}}^{2}\vI\right) \\
     & = 2cd^{2}\log(\sigma) + \log\det\left(\left\{1 + \frac{\alpha^{2}\sigma_{\star}^{2}}{\sigma^{2}}\norm{\vx_{\star}^{i}}^{2}\right\}\vI\right) \\
     & = 2cd^{2}\log(\sigma) + c \log\left(1 + \frac{\alpha^{2}\sigma_{\star}^{2}}{\sigma^{2}}\norm{\vx_{\star}^{i}}^{2}\right) \\ 
     &= c\left(2 d^{2}\log (\sigma) + \log\left\{1 + \frac{\alpha^{2}\sigma_{\star}^{2}}{\sigma^{2}}\norm{\vx_{\star}^{i}}^{2}\right\}\right).
\end{align}
These terms are all simple to compute, and therefore so is the true denoiser (as well as any we parameterize).

\subsection{More Experimental Results}\label{app:sub_more_experiments}

\paragraph{Interpreting the learned models.}

In this section we take advantage of the white-box nature of our denoiser and specifically how it relates to the data generating process, to try to uncover some mechanistic aspects of memorization in this setup. There are two questions we wish to answer:
\begin{itemize}
    \item Is the learned variance \(\sigma_{\mathrm{train}}^{2}\) an effective proxy for memorization or generalization?
    \item How do the learned means \(\vmu_{\mathrm{train}}^{i}\) behave in memorized and generalized models?
\end{itemize}
Recall that the ground truth (generalizing) denoiser has variance \(\sigma_{\star}^{2}\) and means \(\vmu_{\star}^{i}\), and the memorizing denoiser has variance  \(0\) and means \(\vx^{i}\). Thus, one intuition would be that a smaller variance implies a propensity of the trained model to memorize, and learned means closer to samples would imply the same. To verify this behavior we plot the learned variance and the average distance of the learned means to training samples and ground truth means, respectively, in \Cref{fig:learned_means_variance}. We can confirm that our basic intuition is true, with a twist: while the learned variance \(\sigma_{\mathrm{train}}^{2}\) decreases as the model size \(M\) increases, it does so linearly, and \textit{even in the generalization regime}. Similarly the average distance of each learned mean to the closest point in the training dataset decreases linearly to \(0\). In this sense, these and other mechanistically derived quantities may serve as a continuous proxy for the rapid phase transition, a behavior also observed in large models with far more complicated ``circuits'' \citep{nanda2023progress,schaeffer2023emergent}.

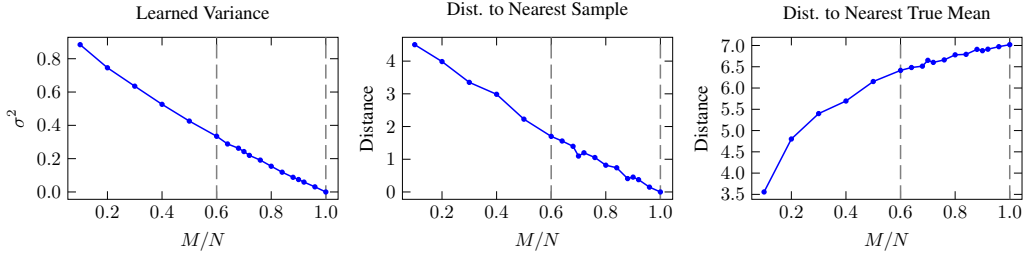
\begin{figure}

    \resizebox{\textwidth}{!}{
        \makebox{
            \begin{tikzpicture}

    \definecolor{darkgray176}{RGB}{176,176,176}
    \definecolor{gray}{RGB}{128,128,128}
    
    \begin{axis}[
    width=5cm,
    height=3cm,
    scale only axis,
    tick pos=both,
    title={Learned Variance},
    x grid style={darkgray176},
    xlabel={\(\displaystyle M/N\)},
    xmin=0.055, xmax=1.045,
    xtick style={color=black},
    xtick={0,0.2,0.4,0.6,0.8,1,1.2},
    xticklabels={
      \(\displaystyle {0.0}\),
      \(\displaystyle {0.2}\),
      \(\displaystyle {0.4}\),
      \(\displaystyle {0.6}\),
      \(\displaystyle {0.8}\),
      \(\displaystyle {1.0}\),
      \(\displaystyle {1.2}\)
    },
    y grid style={darkgray176},
    ylabel={\(\displaystyle \sigma^{2}\)},
    ymin=-0.0442193857557533, ymax=0.928607118090594,
    ytick style={color=black},
    ytick={-0.2,0,0.2,0.4,0.6,0.8,1},
    yticklabels={
      \(\displaystyle {\ensuremath{-}0.2}\),
      \(\displaystyle {0.0}\),
      \(\displaystyle {0.2}\),
      \(\displaystyle {0.4}\),
      \(\displaystyle {0.6}\),
      \(\displaystyle {0.8}\),
      \(\displaystyle {1.0}\)
    }
    ]
    \addplot [thick, blue, mark=*, mark size=1, mark options={solid}]
    table {%
    0.1 0.884387731552124
    0.2 0.745699346065521
    0.3 0.635148167610168
    0.4 0.526133179664612
    0.5 0.425619006156921
    0.6 0.334497421979904
    0.64 0.288249999284744
    0.68 0.262623012065887
    0.7 0.242677569389343
    0.72 0.219632029533386
    0.76 0.191179499030113
    0.8 0.154407575726509
    0.84 0.118640154600143
    0.88 0.0879866108298302
    0.9 0.074664905667305
    0.92 0.0591967403888702
    0.96 0.0306813102215528
    1 7.82717057834503e-10
    };
    \addplot [thick, gray, dash pattern=on 7.4pt off 3.2pt]
    table {%
    0.6 -0.0442193857557533
    0.6 0.928607118090595
    };
    \addplot [thick, gray, dash pattern=on 7.4pt off 3.2pt]
    table {%
    1 -0.0442193857557533
    1 0.928607118090595
    };
    \end{axis}
    
    \end{tikzpicture}
    
            \begin{tikzpicture}

\definecolor{darkgray176}{RGB}{176,176,176}
\definecolor{gray}{RGB}{128,128,128}

\begin{axis}[
    width=5cm,
    height=3cm,
    scale only axis,
tick pos=both,
title={Dist. to Nearest Sample},
x grid style={darkgray176},
xlabel={\(\displaystyle M/N\)},
xmin=0.055, xmax=1.045,
xtick style={color=black},
xtick={0,0.2,0.4,0.6,0.8,1,1.2},
xticklabels={
  \(\displaystyle {0.0}\),
  \(\displaystyle {0.2}\),
  \(\displaystyle {0.4}\),
  \(\displaystyle {0.6}\),
  \(\displaystyle {0.8}\),
  \(\displaystyle {1.0}\),
  \(\displaystyle {1.2}\)
},
y grid style={darkgray176},
ylabel={Distance},
ymin=-0.225098208926829, ymax=4.72782559983043,
ytick style={color=black},
ytick={-1,0,1,2,3,4,5},
yticklabels={
  \(\displaystyle {\ensuremath{-}1}\),
  \(\displaystyle {0}\),
  \(\displaystyle {1}\),
  \(\displaystyle {2}\),
  \(\displaystyle {3}\),
  \(\displaystyle {4}\),
  \(\displaystyle {5}\)
}
]
\addplot [thick, blue, mark=*, mark size=1, mark options={solid}]
table {%
0.1 4.50269269943237
0.2 3.98611307144165
0.3 3.35026669502258
0.4 2.98557329177856
0.5 2.22525930404663
0.6 1.69916903972626
0.64 1.55755591392517
0.68 1.3953572511673
0.7 1.09731209278107
0.72 1.1975599527359
0.76 1.05270433425903
0.8 0.817983627319336
0.84 0.739579975605011
0.88 0.409011840820312
0.9 0.45385667681694
0.92 0.376038670539856
0.96 0.1473518460989
1 3.46914712281432e-05
};
\addplot [thick, gray, dash pattern=on 7.4pt off 3.2pt]
table {%
0.6 -0.225098208926829
0.6 4.72782559983043
};
\addplot [thick, gray, dash pattern=on 7.4pt off 3.2pt]
table {%
1 -0.225098208926829
1 4.72782559983043
};
\end{axis}

\end{tikzpicture}
            \begin{tikzpicture}

\definecolor{darkgray176}{RGB}{176,176,176}
\definecolor{gray}{RGB}{128,128,128}

\begin{axis}[
    width=5cm,
    height=3cm,
    scale only axis,
tick pos=both,
title={Dist. to Nearest True Mean},
x grid style={darkgray176},
xlabel={\(\displaystyle M/N\)},
xmin=0.055, xmax=1.045,
xtick style={color=black},
xtick={0,0.2,0.4,0.6,0.8,1,1.2},
xticklabels={
  \(\displaystyle {0.0}\),
  \(\displaystyle {0.2}\),
  \(\displaystyle {0.4}\),
  \(\displaystyle {0.6}\),
  \(\displaystyle {0.8}\),
  \(\displaystyle {1.0}\),
  \(\displaystyle {1.2}\)
},
y grid style={darkgray176},
ylabel={Distance},
ymin=3.38363699913025, ymax=7.19258065223694,
ytick style={color=black},
ytick={3,3.5,4,4.5,5,5.5,6,6.5,7,7.5},
yticklabels={
  \(\displaystyle {3.0}\),
  \(\displaystyle {3.5}\),
  \(\displaystyle {4.0}\),
  \(\displaystyle {4.5}\),
  \(\displaystyle {5.0}\),
  \(\displaystyle {5.5}\),
  \(\displaystyle {6.0}\),
  \(\displaystyle {6.5}\),
  \(\displaystyle {7.0}\),
  \(\displaystyle {7.5}\)
}
]
\addplot [thick, blue, mark=*, mark size=1, mark options={solid}]
table {%
0.1 3.55677080154419
0.2 4.80165958404541
0.3 5.39901542663574
0.4 5.69466352462769
0.5 6.15304136276245
0.6 6.41429662704468
0.64 6.4819860458374
0.68 6.51340866088867
0.7 6.65123271942139
0.72 6.60252952575684
0.76 6.66096210479736
0.8 6.78037977218628
0.84 6.79149627685547
0.88 6.90878582000732
0.9 6.87750959396362
0.92 6.91250419616699
0.96 6.9719386100769
1 7.019446849823
};
\addplot [thick, gray, dash pattern=on 7.4pt off 3.2pt]
table {%
0.6 3.38363699913025
0.6 7.19258065223694
};
\addplot [thick, gray, dash pattern=on 7.4pt off 3.2pt]
table {%
1 3.38363699913025
1 7.19258065223694
};
\end{axis}

\end{tikzpicture}
        }
    }
    \label{fig:learned_means_variance}
    \caption{\small \textit{Left:} A plot of the learned variance as a function of the model size \(M\). \textit{Middle:} The average distance of a learned mean to its nearest sample in the training data, as a function of \(M\). \textit{Right:} The average distance of a learned mean to the nearest ground truth mean, as a function of \(M\). All plots include the start and end of the phase transition. While the variance eventually decays to \(0\), it surprisingly only does so \textit{linearly}, and for every \(M\) before the \textit{end} of the phase transition the ground truth variance does not collapse to \(0\). Similarly, as the memorization ratio increases and the phase transition occurs, the average distance from a learned mean to the nearest sample decreases \textit{linearly} as a function of \(M\). Meanwhile, the average distance from a learned mean to the nearest true mean \textit{increases}. Surprisingly, this behavior happens during the generalization phase as well. Note that the provided example is representative among our trained models.}
\end{figure}

\paragraph{Assessing the variance of different seeds.} In \Cref{fig:memorization_phase_transition_loss_errorbars}, we examine multiple runs using different random seeds to see their effect on the loss and memorization plots (akin to \Cref{fig:train-loss-approx-verify}). The shading on the memorization plot, which showcases the minimum and maximum value of the quantity across three seeds, amounts to error bars on the regression experiment in \Cref{fig:phase_transition_predictable_approximation}, since the approximated losses will be the same (as they are computed deterministically), and the only remaining variation is the regression target, i.e., the location of the phase transition.

\begin{figure}
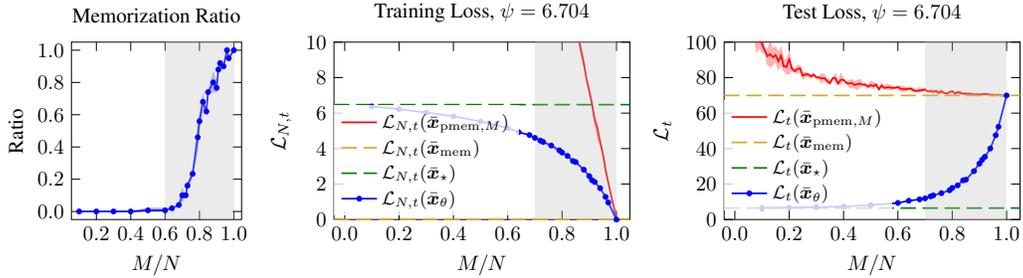

    \centering

    \resizebox{\textwidth}{!}{
        \makebox{
            \begin{tikzpicture}

\definecolor{darkgray176}{RGB}{176,176,176}
\definecolor{gray}{RGB}{128,128,128}

\begin{axis}[
    width=2.875cm,
    height=3cm,
    scale only axis,
tick pos=both,
title={Memorization Ratio},
x grid style={darkgray176},
xlabel={\(\displaystyle M/N\)},
xmin=0.055, xmax=1.045,
xtick style={color=black},
xtick={0,0.2,0.4,0.6,0.8,1,1.2},
xticklabels={
  \(\displaystyle {0.0}\),
  \(\displaystyle {0.2}\),
  \(\displaystyle {0.4}\),
  \(\displaystyle {0.6}\),
  \(\displaystyle {0.8}\),
  \(\displaystyle {1.0}\),
  \(\displaystyle {1.2}\)
},
y grid style={darkgray176},
ylabel={Ratio},
ymin=-0.05, ymax=1.05,
ytick style={color=black},
ytick={-0.2,0,0.2,0.4,0.6,0.8,1,1.2},
yticklabels={
  \(\displaystyle {\ensuremath{-}0.2}\),
  \(\displaystyle {0.0}\),
  \(\displaystyle {0.2}\),
  \(\displaystyle {0.4}\),
  \(\displaystyle {0.6}\),
  \(\displaystyle {0.8}\),
  \(\displaystyle {1.0}\),
  \(\displaystyle {1.2}\)
}
]
\path [draw=blue, fill=blue, opacity=0.3]
(axis cs:0.1,0)
--(axis cs:0.1,0)
--(axis cs:0.2,0)
--(axis cs:0.3,0)
--(axis cs:0.4,0)
--(axis cs:0.5,0)
--(axis cs:0.6,0)
--(axis cs:0.64,0.0199999995529652)
--(axis cs:0.68,0.0399999991059303)
--(axis cs:0.7,0.0599999986588955)
--(axis cs:0.72,0.0999999940395355)
--(axis cs:0.73,0.140000000596046)
--(axis cs:0.76,0.219999998807907)
--(axis cs:0.79,0.379999995231628)
--(axis cs:0.8,0.519999980926514)
--(axis cs:0.82,0.639999985694885)
--(axis cs:0.84,0.620000004768372)
--(axis cs:0.85,0.740000009536743)
--(axis cs:0.88,0.740000009536743)
--(axis cs:0.9,0.740000009536743)
--(axis cs:0.91,0.85999995470047)
--(axis cs:0.92,0.919999957084656)
--(axis cs:0.94,0.899999976158142)
--(axis cs:0.96,1)
--(axis cs:0.97,0.939999997615814)
--(axis cs:1,1)
--(axis cs:1,1)
--(axis cs:1,1)
--(axis cs:0.97,0.959999978542328)
--(axis cs:0.96,1)
--(axis cs:0.94,0.899999976158142)
--(axis cs:0.92,0.919999957084656)
--(axis cs:0.91,0.899999976158142)
--(axis cs:0.9,0.799999952316284)
--(axis cs:0.88,0.85999995470047)
--(axis cs:0.85,0.740000009536743)
--(axis cs:0.84,0.620000004768372)
--(axis cs:0.82,0.719999969005585)
--(axis cs:0.8,0.599999964237213)
--(axis cs:0.79,0.539999961853027)
--(axis cs:0.76,0.239999994635582)
--(axis cs:0.73,0.179999992251396)
--(axis cs:0.72,0.0999999940395355)
--(axis cs:0.7,0.140000000596046)
--(axis cs:0.68,0.0399999991059303)
--(axis cs:0.64,0.0199999995529652)
--(axis cs:0.6,0.0199999995529652)
--(axis cs:0.5,0.0199999995529652)
--(axis cs:0.4,0)
--(axis cs:0.3,0)
--(axis cs:0.2,0)
--(axis cs:0.1,0)
--cycle;

\addplot [thick, blue, mark=*, mark size=1, mark options={solid}]
table {%
0.1 0
0.2 0
0.3 0
0.4 0
0.5 0.00666666636243463
0.6 0.00666666636243463
0.64 0.0199999995529652
0.68 0.0399999991059303
0.7 0.0999999940395355
0.72 0.0999999940395355
0.73 0.159999996423721
0.76 0.233333334326744
0.79 0.459999978542328
0.8 0.560000002384186
0.82 0.679999947547913
0.84 0.620000004768372
0.85 0.740000009536743
0.88 0.799999952316284
0.9 0.766666650772095
0.91 0.879999995231628
0.92 0.919999957084656
0.94 0.899999976158142
0.96 1
0.97 0.949999988079071
1 1
};
\fill [gray!50!white, opacity=0.3] (axis cs:0.6,-0.05) rectangle (axis cs:1,1.05);
\end{axis}

\end{tikzpicture}
            \input{figs/experiments/train_loss_with_memorization_with_errorbars.tex}
            \input{figs/experiments/test_loss_with_memorization_with_errorbars.tex}
        }
    }

    \caption{\textbf{Our loss calculations and memorization trends are stable under different random seeds.} We observe the same behaviors as \Cref{fig:memorization_phase_transition_loss} when re-attempting the same experiment with three separate random seeds; we provide error bars but note that they are extremely small, indicating a tiny variance.}
    \label{fig:memorization_phase_transition_loss_errorbars}
    \vspace{-1.5em}
\end{figure}

\end{document}